\documentclass{article}
\usepackage{amsmath}
\usepackage{amssymb}
\usepackage{bm}

\usepackage{amsthm}
\newtheorem{theorem}{Theorem}
\newtheorem{lemma}{Lemma}
\newtheorem{proposition}{Proposition}
\newtheorem{assumption}{Assumption}
\newtheorem{corollary}{Corollary}

\renewcommand{\theassumption}{(A\arabic{assumption})}

\renewenvironment{theorem}[1][]{
  \refstepcounter{theorem}% increment the theorem counter
  \noindent\textbf{Theorem \thetheorem\ifx#1\empty\else\ (#1)\fi:} \normalfont
}{\par}

\renewenvironment{proposition}[1][]{
  \refstepcounter{proposition}% increment the proposition counter
  \noindent\textbf{Proposition \theproposition\ifx#1\empty\else\ (#1)\fi:} \normalfont
}{\par}

\renewenvironment{assumption}[1][]{
  \refstepcounter{assumption}% increment the assumption counter
  \noindent\textbf{Assumption \theassumption\ifx#1\empty\else\ (#1)\fi:} \normalfont
}{\par}

\DeclareRobustCommand{\mb}[1]{\boldsymbol{#1}}

\DeclareRobustCommand{\KL}[2]{\ensuremath{\textrm{D}_{\textrm{KL}}\left(#1\;\|\;#2\right)}}

\renewcommand{\mid}{~\vert~}

\newcommand{\mbb}{\mb{b}}
\newcommand{\mbc}{\mb{c}}
\newcommand{\mbd}{\mb{d}}
\newcommand{\mbe}{\mb{e}}
\newcommand{\mbf}{\mb{f}}
\newcommand{\mbg}{\mb{g}}
\newcommand{\mbh}{\mb{h}}

\newcommand{\mbm}{\mb{m}}

\newcommand{\mbr}{\mb{r}}
\newcommand{\mbs}{\mb{s}}

\newcommand{\mbu}{\mb{u}}
\newcommand{\mbv}{\mb{v}}
\newcommand{\mbw}{\mb{w}}
\newcommand{\mbx}{\mb{x}}
\newcommand{\mby}{\mb{y}}
\newcommand{\mbz}{\mb{z}}

\newcommand{\mbA}{\mb{A}}
\newcommand{\mbB}{\mb{B}}
\newcommand{\mbC}{\mb{C}}

\newcommand{\mbJ}{\mb{J}}
\newcommand{\mbK}{\mb{K}}
\newcommand{\mbL}{\mb{L}}
\newcommand{\mbM}{\mb{M}}

\newcommand{\mbQ}{\mb{Q}}
\newcommand{\mbR}{\mb{R}}
\newcommand{\mbS}{\mb{S}}

\newcommand{\mbU}{\mb{U}}
\newcommand{\mbV}{\mb{V}}

\newcommand{\mbalpha}{\mb{\alpha}}
\newcommand{\mbbeta}{\mb{\beta}}

\newcommand{\mbepsilon}{\mb{\epsilon}}

\newcommand{\mbkappa}{\mb{\kappa}}

\newcommand{\mbmu}{\mb{\mu}}
\newcommand{\mbnu}{\mb{\nu}}

\newcommand{\mbphi}{\mb{\phi}}

\newcommand{\mbrho}{\mb{\rho}}
\newcommand{\mbsigma}{\mb{\sigma}}
\newcommand{\mbtau}{\mb{\tau}}
\newcommand{\mbtheta}{\mb{\theta}}

\newcommand{\mbvarepsilon}{\mb{\varepsilon}}

\newcommand{\mbxi}{\mb{\xi}}

\newcommand{\mbDelta}{\mb{\Delta}}

\newcommand{\mbLambda}{\mb{\Lambda}}
\newcommand{\mbOmega}{\mb{\Omega}}
\newcommand{\mbPhi}{\mb{\Phi}}

\newcommand{\mbSigma}{\mb{\Sigma}}

\newcommand{\dif}{\mathop{}\!\mathrm{d}}
\newcommand{\diag}{\textrm{diag}}

\newcommand{\E}{\mathbb{E}}

\newcommand{\bbE}{\mathbb{E}}

\newcommand{\bbI}{\mathbb{I}}

\newcommand{\bbN}{\mathbb{N}}

\newcommand{\bbP}{\mathbb{P}}
\newcommand{\bbQ}{\mathbb{Q}}
\newcommand{\bbR}{\mathbb{R}}

\newcommand{\cA}{\mathcal{A}}
\newcommand{\cB}{\mathcal{B}}
\newcommand{\cC}{\mathcal{C}}
\newcommand{\cD}{\mathcal{D}}
\newcommand{\cE}{\mathcal{E}}

\newcommand{\cG}{\mathcal{G}}

\newcommand{\cI}{\mathcal{I}}

\newcommand{\cL}{\mathcal{L}}

\newcommand{\cN}{\mathcal{N}}

\newcommand{\cP}{\mathcal{P}}
\newcommand{\cQ}{\mathcal{Q}}

\newcommand{\cS}{\mathcal{S}}
\newcommand{\cT}{\mathcal{T}}
\newcommand{\cU}{\mathcal{U}}
\newcommand{\cV}{\mathcal{V}}

\newcommand{\cZ}{\mathcal{Z}}

\PassOptionsToPackage{numbers, compress}{natbib}
 \usepackage[main, final]{neurips_2026}

\usepackage[utf8]{inputenc} % allow utf-8 input
\usepackage[T1]{fontenc}    % use 8-bit T1 fonts
\usepackage{url}            % simple URL typesetting
\usepackage{booktabs}       % professional-quality tables
\usepackage{amsfonts}       % blackboard math symbols
\usepackage{nicefrac}       % compact symbols for 1/2, etc.
\usepackage{microtype}      % microtypography
\usepackage{nccmath}
\usepackage{subcaption}

\usepackage{makecell}

\usepackage[table,usenames,dvipsnames]{xcolor}
\definecolor{tabblue}{HTML}{1F77B4}
\definecolor{tabpurple}{HTML}{9467BD}
\definecolor{tabred}{HTML}{D62728}
\usepackage[colorlinks,linktoc=all]{hyperref}
\hypersetup{citecolor=MidnightBlue}
\hypersetup{linkcolor=MidnightBlue}
\hypersetup{urlcolor=MidnightBlue}

\usepackage{cleveref}
\crefname{appendix}{appendix}{appendices}
\Crefname{appendix}{Appendix}{Appendices}
\crefname{assumption}{assumption}{assumptions}
\Crefname{assumption}{Assumption}{Assumptions}
\crefname{algocf}{Algorithm}{Algorithms}
\Crefname{algocf}{Algorithm}{Algorithms}
\crefname{table}{Table}{Tables}
\Crefname{table}{Table}{Tables}

\usepackage{enumitem}

\usepackage[most]{tcolorbox}
\newtheorem*{theorem*}{Theorem}
\tcolorboxenvironment{theorem}{
  enhanced,
  breakable,
  colback=white,
  colframe=black,
  boxrule=0.4pt,
  arc=4pt,
  left=5pt,
  right=5pt,
  top=4pt,
  bottom=4pt,
  before skip=6pt,
  after skip=6pt,
}
\tcolorboxenvironment{theoremnonum}{
  enhanced,
  breakable,
  colback=white,
  colframe=black,
  boxrule=0.4pt,
  arc=4pt,
  left=5pt,
  right=5pt,
  top=4pt,
  bottom=4pt,
  before skip=6pt,
  after skip=6pt,
}

\tcolorboxenvironment{proposition}{
  enhanced,
  breakable,
  colback=white,
  colframe=black,
  boxrule=0.4pt,
  arc=4pt,
  left=5pt,
  right=5pt,
  top=4pt,
  bottom=4pt,
  before skip=6pt,
  after skip=6pt,
}

\usepackage[ruled]{algorithm2e}

\DontPrintSemicolon
\SetAlgoLined
\SetKwInput{Require}{Require}
\SetKwFor{For}{for}{do}{end for}

\SetCommentSty{mycommfont}

\SetKwComment{tcp}{\#\ }{}
\SetKwComment{tcpp}{\#\#\ }{}
\newcommand{\AlgStep}[1]{\(\triangleright\)\hspace{0.5em}#1\;}

\usepackage{bbm}

\title{Closing the Approximation Gap in Simulation-free Latent SDEs}

\author{
\begin{tabular}{c}
Henry D.~Smith
\qquad Brian L.~Trippe\thanks{Equal advising}
\qquad Scott W.~Linderman\footnotemark[1] \\
{\normalfont Stanford University} \\
\texttt{\{smithhd,btrippe,swl1\}@stanford.edu}
\end{tabular}
}

\begin{document}

\maketitle

\begin{abstract}
Recovering dynamical systems from noisy observations is a recurring challenge across scientific domains, including neuroscience and physics. 
Latent stochastic differential equations (SDEs) address this by modeling the system as an unobserved state that evolves according to a learnable SDE and generates the observations.
Variational inference (VI) provides a tractable objective for fitting latent SDEs. 
Traditional VI algorithms evaluate this objective by numerical simulation over a time discretization, trading fidelity for computational cost.
A recent class of algorithms, simulation-free VI, sidesteps this tradeoff by parameterizing the posterior through its instantaneous marginals rather than its drift.
In this work, we show that the efficiency of existing simulation-free VI algorithms comes at a price: their parameterizations restrict the approximate posterior to a subset of the SDEs available to simulation-based methods, degrading posterior inference and parameter learning. 
We propose Helmholtz-SDE, a simulation-free VI algorithm that closes this gap by optimizing over path laws compatible with a prescribed collection of marginals. 
Helmholtz-SDE recovers dynamics more faithfully than prior simulation-free methods, with the largest gains under high posterior uncertainty.
It further matches the performance of simulation-based VI at a fraction of the runtime.
Code is available at \url{https://github.com/lindermanlab/helmholtz-sde}.
\end{abstract}
\section{Introduction}
Latent stochastic differential equations (SDEs) enable the unsupervised discovery of dynamical systems from noisy, irregularly sampled time series, a setting common across scientific domains. 
They posit an unobserved latent state that evolves continuously in time and generates observations at a finite set of time points \cite{archambeau2007gaussian, li2020scalable}. 
Latent SDEs have been applied to neuroscience \cite{elgazzar2024universal, hu2024modeling}, disease modeling \cite{durso2024probabilistic}, and reduced-order fluid dynamics \cite{course2023state}.
Fitting these models requires both learning the latent dynamics and approximating the posterior over continuous trajectories.

Variational inference (VI) translates the problem of fitting latent SDEs into an optimization problem. 
However, evaluating the variational objective (ELBO) requires computing an integral over the full time horizon (a path-space KL divergence).
Simulation-based VI algorithms estimate this term by sweeping across a time discretization at each gradient step, either by integrating an ODE to compute the posterior marginals or by simulating sample trajectories from the variational SDE.
A recent line of work, simulation-free VI  \cite{course2023state,course2023amortized,bartosh2025sde}, sidesteps sequential simulation using stochastic VI. 
Rather than integrating, they compute an unbiased estimate of the objective by sampling a random time together with a state at that time.
They make this estimate efficient by parameterizing the posterior through its one-time marginal distributions, i.e., the distribution of the latent state at each time, allowing states to be sampled directly without simulating the preceding trajectory.

In this work, we show that the family of approximate posterior distributions over which simulation-free VI optimizes is \emph{strictly smaller} than that available to simulation-based VI.
One-time marginals do not determine a path law.
Existing simulation-free methods choose a single compatible drift from many possibilities, implicitly restricting the variational posterior. 
When observations are sparse or noisy, this arbitrary choice can substantially degrade both posterior inference and learned dynamics.

We propose \emph{Helmholtz-SDE}, a simulation-free VI algorithm that removes this restriction. 
Starting from any drift compatible with the prescribed marginals, we compute a Helmholtz correction that preserves those marginals while optimizing the ELBO. 
The resulting posterior family remains simulation-free but spans a larger class of path laws, closing the gap to simulation-based VI.
\section{Problem Statement}\label{sec:problem-statement}
\begin{figure*}[!t]
  \centering
  \includegraphics[width=\textwidth]{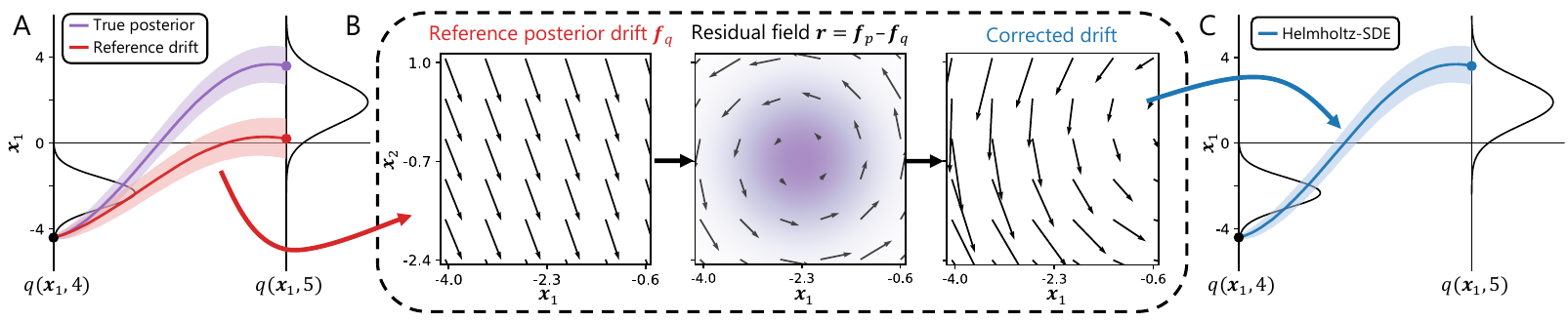}

  \caption{
  Overview of Helmholtz-SDE.
  \textbf{A}:
    Existing simulation-free VI algorithms commit to a single drift (\textcolor{tabred}{red}) for each set of one-time marginals (black). This choice induces transition distributions that can differ from those of the true posterior (\textcolor{tabpurple}{purple}), even when the one-time marginals agree.
    \textbf{B}:
    Helmholtz-SDE takes a drift consistent with a set of one-time marginals, then corrects it by a vector field that preserves the marginals.
    \textbf{C}: The resulting approximation (\textcolor{tabblue}{blue}) better matches the posterior.
    }
  \label{fig:spiral-overview}
  \vspace{-1em}
\end{figure*}

\paragraph{Model.}
We study \emph{latent stochastic differential equation} (SDE) models, in which a $K$-dimensional latent state $\mbx(t)$ evolves over the interval $0 \leq t \leq T$ and gives rise to $D$-dimensional observations $\mby(t_n)$ at times $\{t_n\}_{n=1}^N$. 
Let $\mbtheta$ denote the full collection of model parameters, including the parameters of the prior and the likelihood.
The prior over the latent state is the solution to an SDE
\begin{equation}\label{eqn:prior_sde} 
\bbP_{\mbtheta}\!\!:\; d\mbx(t) = \mbf_p(\mbx(t)| \mbtheta) dt + \mbSigma(\mbtheta)^{1/2} d\mbw(t), \quad  \mbx(0) \sim p(\mbx, 0 | \mbtheta), \quad 0 \leq t \leq T,
\end{equation}
where $\mbf_p(\cdot | \mbtheta)\!: \bbR^K \to \bbR^K$ is the drift function, 
$\mbSigma(\mbtheta)^{1/2}$ is the diffusion coefficient,
$(\mbw(t))_{0 \leq t \leq T}$ is a standard $K$-dimensional Brownian motion, 
and $p(\mbx, 0 | \mbtheta)$ is the initial state distribution. 
We denote the marginal density of $\bbP_{\mbtheta}$ at time $t$ by $p(\mbx, t | \mbtheta)$.
For simplicity of presentation, we restrict notation to a time-independent drift and a state- and time-independent diffusion coefficient.

We observe data $\cD = \{(t_n, \mby(t_n))\}_{n=1}^N$ representing noisy readouts of the latent state. 
These observations are conditionally independent given the latent state:  $p(\mby | (\mbx(t))_{0 \leq t \leq T}, \mbtheta) = \prod_{n=1}^N p(\mby(t_n) | \mbx(t_n), \mbtheta)$, where $\mby = \{\mby(t_n)\}_{n=1}^N$.
We describe our technical assumptions on the generative model in \Cref{app:subsec:assumptions}.

There are two goals in latent SDE models: (i) to learn the parameters $\mbtheta$ that maximize $\log p(\mby | \mbtheta)$; and (ii) to infer an approximate posterior over the latent trajectory $(\mbx(t))_{0 \leq t \leq T}$ given the data $\cD$ and parameters $\mbtheta$.
In practice, one may have multiple trials $\{\mby^{(m)}\}_{m=1}^{M}$, each representing an independent sample from the probabilistic model. 
In this case, the model parameters $\mbtheta$ are learned using all trials, and the posterior factorizes over trials. For simplicity, we suppress the trial index $m$.

\paragraph{Intractability.}
Both goals of learning and inference are generally intractable. 
For learning, evaluating the marginal likelihood $p(\mby | \mbtheta)$
involves marginalizing over the prior on latent trajectories \eqref{eqn:prior_sde}; this integral cannot be evaluated in closed form.
For inference, the posterior law is the solution to an SDE with the same diffusion coefficient $\mbSigma^{1/2}$ as the prior and drift obtained by solving a linear partial differential equation \cite{kushner1964differential, stratonovich1965conditional, pardoux1982equations,archambeau2011approximate}. 
In nonlinear or non-Gaussian settings, the posterior drift is unavailable in closed form, and solving the associated PDE can be computationally prohibitive even for small $K$.

\section{Background}\label{sec:background}
Variational inference (VI) provides a tractable framework for joint learning and inference in latent SDEs. We first describe the variational objective and survey existing VI algorithms. We then identify a class of stochastic VI algorithms that avoid sequential simulation, the setting of our method.

\subsection{Variational inference for latent SDE models}\label{subsec:variational_inference}
Variational inference (VI) approximates the intractable posterior with a distribution belonging to a tractable family, known as a  \emph{variational family} \cite{jordan1999introduction, blei2017variational}. 
For latent SDEs, we consider the family $\cQ$ comprised of diffusion processes
\begin{equation}\label{eqn:variational_family}
    \bbQ_{\mbphi}\!\!:\; d \mbx(t) = \mbf_q(\mbx(t), t | \mbphi) dt + \mbSigma(\mbtheta)^{1/2} d\mbw(t), \quad \mbx(0) \sim q(\mbx, 0 | \mbphi), \quad 0 \leq t \leq T,
\end{equation}
where the variational parameters $\mbphi$ specify the posterior drift $\mbf_q(\mbx, t| \mbphi)$ and the initial state distribution $q(\mbx, 0 | \mbphi)$.
This family is chosen because, when $\mbf_q(\mbx, t | \mbphi)$ and $q(\mbx, 0 | \mbphi)$ are sufficiently expressive,
the exact posterior will belong to $\cQ$ \citep[see, e.g.,][]{hu2025sing}.

We minimize the negative evidence lower bound (ELBO) over distributions in $\cQ$,
\begin{equation}\label{eqn:elbo}
\begin{aligned}
   \cL(\mbtheta,\mbphi) &= \underbrace{\sum_{n=1}^N {-}\bbE_{\bbQ_{\mbphi}} \log p(\mby(t_n) | \mbx(t_n), \mbtheta)}_{\cL_{\mathrm{rec}}} 
   + \underbrace{\KL{q(\mbx, 0 | \mbphi)}{p(\mbx, 0 | \mbtheta)}}_{\cL_{\mathrm{init}}}  + \\
   &\underbrace{\frac{1}{2}\int_0^{T} \bbE_{\bbQ_{\mbphi}} \|\mbSigma(\mbtheta)^{-1/2}\{\mbf_q(\mbx(t), t | \mbphi) - \mbf_p(\mbx(t)| \mbtheta)\}\|_2^2 dt}_{\cL_{\mathrm{path}}},
\end{aligned}
\end{equation}
where $\cL_{\mathrm{rec}}$ is the reconstruction loss, $\cL_{\mathrm{init}}$ is the KL divergence to the prior initial distribution, and $\cL_{\mathrm{path}}$ is the path-space KL divergence to the prior.

The negative ELBO is an upper bound on the negative marginal log likelihood $-\log p(\mby | \mbtheta)$, and the gap between them is the KL divergence between $\bbQ_{\mbphi}$ and the exact posterior.
Therefore, for fixed $\mbtheta$, when the exact posterior is contained in $\cQ$ it is the unique minimizer of \cref{eqn:elbo}.
\Cref{app:subsec:elbo} gives a derivation.
However, unlike the negative log likelihood, the ELBO provides a tractable objective for learning  model parameters, which we seek to estimate by jointly optimizing with respect to $\mbphi$ and $\mbtheta$.

Existing VI algorithms for latent SDEs fall into two broad classes based on how they evaluate the negative ELBO \cref{eqn:elbo}: simulation-based and simulation-free.
We describe each in turn.

\subsection{Simulation-based VI}\label{subsec:sim-based}
\emph{Simulation-based} algorithms evaluate the path KL $\cL_{\mathrm{path}}$ over a discretization of the interval $[0, T]$ at each gradient step. 
When $\bbQ_{\mbphi}$ is restricted to a tractable family, such as a Gaussian process with affine $\mbf_q$ and state-independent $\mbSigma$, the marginals of $\bbQ_{\mbphi}$ can be computed by integrating an ODE \cite{archambeau2007gaussian, hu2025sing,archambeau2007variational, duncker2019learning, verma2024variational, kohs2021variational, wildner2021moment,park2024amortized, parkstochastic2026}.
When $\mbf_q$ is parameterized more flexibly, e.g., by a neural network \cite{li2020scalable, ryder2018black, tzen2019neural, kidger2021efficient, oh2024stable}, samples $\mbx(t) \sim \bbQ_{\mbphi}$ are obtained by simulating the variational SDE \eqref{eqn:variational_family}. 
In both cases, the time complexity per gradient step scales with the discretization of $[0, T]$, introducing a tradeoff between approximation fidelity and computational cost.

\subsection{Simulation-free VI}\label{subsec:sim-free}
\emph{Simulation-free} VI algorithms introduce two ideas to avoid discretizing the variational SDE and accelerate optimization of $\cL(\mbtheta, \mbphi).$
The first is to optimize
a Monte Carlo estimate of \cref{eqn:elbo}
{\small
\begin{equation*}
    \widehat{\cL}(\mbtheta, \mbphi) = -N \log p(\mby(t_n) | \mbx(t_n), \mbtheta) + \log\frac{q(\mbx(0), 0 | \mbphi)}{p(\mbx(0), 0 | \mbtheta)} + \frac{T}{2} \|\mbSigma(\mbtheta)^{-1/2}\{\mbf_q(\mbx(t), t | \mbphi) - \mbf_p(\mbx(t) | \mbtheta)\}\|_2^2.
\end{equation*}
}
Here, $n \sim \cU(\{1, \dots, N\})$ is a random observation index, $t \sim \cU([0, T])$ is a random time, and $\mbx(0), \mbx(t_n), \mbx(t) \sim \bbQ_{\mbphi}$ are posterior samples.
The estimate $\widehat{\cL}(\mbtheta, \mbphi)$ avoids the sum over all observations in $\cL_{\mathrm{rec}}$, the integral over the full trajectory in $\cL_{\mathrm{path}}$, and the expectations over $\bbQ_{\mbphi}$ in all three terms of $\cL(\mbtheta, \mbphi)$. 
Compared to optimizing the negative ELBO, this approach introduces variance that is managed by averaging across Monte Carlo samples.
Because the estimate samples random times rather than full paths, optimization amounts to a stochastic VI algorithm \cite{hoffman2013stochastic}.

The second idea is to parameterize the drift $\mbf_q$ that defines $\bbQ_{\mbphi}$ implicitly through its one-time marginal distributions.
In order for optimizing $\widehat{\cL}(\mbtheta, \mbphi)$ to provide runtime benefits over optimizing the negative ELBO, we must be able to draw samples $\mbx(t) \sim \bbQ_{\mbphi}$ without simulating the variational SDE \cref{eqn:variational_family}.
This is achieved by reversing the usual construction of $\bbQ_{\mbphi}$: we first prescribe the marginal distribution of the posterior at any single time $\{q(\mbx,t|\mbphi)\}_{0 \leq t \leq T}$ and then choose a drift $\mbf_q$ compatible with those marginals.
One can sample $\mbx(t)$ from the posterior with time complexity that does not depend on $t$.

For example, SDE Matching (\citet{bartosh2025sde}) parameterizes the posterior as having Gaussian marginals, $q(\mbx,t|\mbphi)=\cN(\mbm(t|\mbphi),\mbS(t|\mbphi))$ and shows that a time-dependent, affine drift $\mbf_q(\mbx, t | \mbphi) = \mbA_q(t | \mbphi) \mbx + \mbb_q(t | \mbphi)$ defines a variational SDE with precisely these marginals.\footnote{The framework proposed by \citet{bartosh2025sde} can in principle support non-Gaussian one-time marginals, but the method is instantiated and evaluated only with Gaussian marginals. See \Cref{app:subsec:svise-sde-compare} for further discussion.}
A compatible drift is readily computed from $\mbm(t|\mbphi)$ and $\mbS(t|\mbphi)$ as
{\small
\begin{equation}\label{eqn:sym-sqrt-reference}
    \mbf_q(\mbx,t|\mbphi) = \left[\left(\frac{d}{dt}\mbS(t|\mbphi)^{1/2} \right) \mbS(t|\mbphi)^{-1/2} - \frac{1}{2}\mbSigma(\mbtheta)\mbS(t|\mbphi)^{-1} \right][\mbx - \mbm(t|\mbphi)] + \frac{d}{dt}\mbm(t|\mbphi),
\end{equation} 
}
where $\mbS(t|\mbphi)^{1/2}$ is the symmetric square root of $\mbS(t|\mbphi)$.
SVISE (\citet{course2023state,course2023amortized}) also considers Gaussian marginals and a time-dependent affine SDE drift, but they instead choose $ \mbA_q(t | \mbphi)$ to be symmetric.
We provide background on both algorithms in \Cref{app:subsec:svise-sde-compare}.
Once we have chosen a drift $\mbf_q(\mbx, t | \mbphi)$, unbiased gradient estimates of $\cL(\mbtheta, \mbphi)$ with respect to $\mbtheta$ and $\mbphi$ can be obtained via the reparameterization trick, without backpropagating through a simulated trajectory \cite{kingma2013auto}.
\section{Helmholtz-SDE: Augmented simulation-free variational inference}\label{sec:method}

The simulation-free VI framework we described in \Cref{subsec:sim-free} leaves an important ambiguity unresolved: the one-time marginal distributions $\{q(\mbx,t|\mbphi)\}_{0 \leq t \leq T}$ do not uniquely determine the variational path law $\bbQ_{\mbphi}$. 
In general, many drifts $\mbf_q$ realize the same collection of marginals, and existing simulation-free VI algorithms resolve this ambiguity by selecting a particular compatible drift for each set of marginals, independently of the variational objective in \cref{eqn:elbo}. 
This choice is not a benign implementation detail: drifts with identical one-time marginals can induce different path laws, including different multi-time joint distributions $(\mbx(s), \mbx(t))$ for $s \neq t$ (\Cref{fig:spiral-overview}). 
When posterior uncertainty is high, we argue that fixing a single compatible drift can substantially degrade both posterior inference and learned dynamics.

We introduce Helmholtz-SDE, a simulation-free VI algorithm that removes this restriction by optimizing over path laws sharing the same one-time marginals. 
We first characterize the path-law ambiguity left unresolved by marginal-based parameterizations: drifts realizing the same marginals differ by divergence-free vector fields, which can be added without altering those marginals (\Cref{subsec:fpk}). 
We then show that the variational objective selects an optimal drift among this class (\Cref{subsec:helmholtz-decomp}).
Finally, we turn this characterization into a practical algorithm for Gaussian one-time marginals and state-independent diffusion (\Cref{subsec:gaussian-marginals}).

\subsection{The Fokker--Planck equation and divergence-free vector fields}\label{subsec:fpk}
The Fokker--Planck equation characterizes the ambiguity in the posterior path-law that arises from parameterizing the one-time marginals. 
Let $\bbQ_{\mbphi} \in \cQ$ be a variational distribution with drift $\mbf_q$ and one-time marginals $\{q(\mbx,t | \mbphi)\}_{0 \le t \le T}$. The marginals evolve according to the PDE
\begin{equation}\label{eqn:fpk}
    \frac{\partial}{\partial t} q(\mbx, t | \mbphi) = - \nabla_{\mbx} \cdot \bigl[q(\mbx,t | \mbphi)\mbf_q(\mbx,t | \mbphi)\bigr] + \frac{1}{2}\sum_{i,j} \mbSigma_{ij}(\mbtheta) \partial_{\mbx_i \mbx_j} q(\mbx,t | \mbphi), \quad 0 \le t \le T,
\end{equation}
known as the \emph{Fokker--Planck equation}, where $\nabla_{\mbx} \cdot \mbf = \sum_i \partial_{\mbx_i} \mbf_i(\mbx)$ is the divergence operator.

\Cref{eqn:fpk} depends on the variational drift $\mbf_q$ only through the divergence of $q(\mbx, t | \mbphi) \mbf_q(\mbx, t | \mbphi)$. 
Consequently, any vector field $\mbv$ satisfying $\nabla_{\mbx} \cdot [q(\mbx, t | \mbphi) \mbv(\mbx, t)] = 0$ for all  $(\mbx, t)$ leaves the marginals unchanged when added to $\mbf_q$. 
We call such vector fields \emph{$q$-divergence-free}. The following 
proposition shows that divergence-free vector fields are the only source of non-uniqueness.
\begin{proposition}\label{prop:div_free}
Consider two diffusions with drifts $\mbf$ and $\tilde{\mbf}$ that share the same diffusion coefficient and initial distribution $q(\mbx, 0)$.
They induce the same one-time marginals $\{q(\mbx, t)\}_{0 \leq t \leq T}$ iff $\mbf - \tilde{\mbf}$ is $q$-divergence-free.   
\end{proposition}
\noindent\emph{Proof.}
One analyzes the Fokker--Planck equations for the two diffusions. See \Cref{app:subsec:div-free}.\hfill$\square$

\subsection{Optimizing over path laws at fixed marginals}\label{subsec:helmholtz-decomp}
In \Cref{subsec:fpk}, we identified the family of drifts compatible with a given set of one-time marginal distributions. 
Existing simulation-free VI algorithms commit to a single member, independent of the ELBO. 
We show next that optimizing the ELBO over this family produces an optimal drift.

\paragraph{Problem formulation.} 
Fix a collection of one-time marginal distributions $\{q(\mbx, t | \mbphi)\}_{0 \le t \le T}$. 
We call a drift $\mbf_q(\mbx, t | \mbphi)$ a \emph{reference drift} for these marginals if the SDE with drift $\mbf_q$ and diffusion coefficient $\mbSigma^{1/2}$ induces one-time marginals $\{q(\mbx, t | \mbphi)\}_{0 \leq t \leq T}$.
For Gaussian one-time marginals, \cref{eqn:sym-sqrt-reference} provides an example of a reference drift. 
By \Cref{prop:div_free}, every other drift with the same one-time marginals can be written as $\mbf_q + \mbv$, where $\mbv(\mbx,t)$ is $q$-divergence-free. Thus, optimizing over all path laws compatible with these marginals is equivalent to optimizing over all divergence-free perturbations to the reference drift. 

In the negative ELBO, the reconstruction loss and the initial-time KL are unchanged by $\mbv$, so the only dependence on the path law is through the path-space KL $\cL_{\mathrm{path}}$. 
Moreover, since $\mbv$ is $q$-divergence-free, modifying it at time $s$ does not affect the marginal distribution at time $t$.
Therefore, the problem decouples over times $0 \leq t \leq T$ and the optimal compatible drift solves
\begin{equation}\label{eqn:optimize}
    \inf_{\mbv} \; \bbE_{\bbQ_{\mbphi}} \| \mbSigma(\mbtheta)^{-1/2}\{\mbf_p(\mbx(t)| \mbtheta) -  \mbf_q(\mbx(t), t | \mbphi) - \mbv(\mbx(t), t) \}\|_2^2, 
\end{equation}
subject to $\nabla_{\mbx} \cdot [\mbv(\mbx, t) q(\mbx, t | \mbphi)] = 0, \ (\mbx, t) \in \bbR^K \times [0, T]$.

\paragraph{Solution via the Helmholtz decomposition.}
Define the \emph{residual vector field} $\mbr(\mbx,t | \mbtheta, \mbphi) = \mbf_p(\mbx | \mbtheta) - \mbf_q(\mbx, t | \mbphi)$ which measures the gap between the prior and the reference posterior drifts. 
Intuitively, we would like to choose $\mbv = \mbr$ so that \cref{eqn:optimize} becomes zero. 
However, the divergence-free constraint  prevents us from adding the part of $\mbr$ that would alter the marginals. 

To isolate the part of $\mbr$ that alters the marginals, we appeal to the \emph{$q$-weighted Helmholtz decomposition}
\begin{equation}\label{eqn:helmholtz-decomp}
\mbr(\mbx, t | \mbtheta, \mbphi) = \mbSigma(\mbtheta)\nabla_{\mbx} \psi(\mbx,t) + \mbh(\mbx,t),
\end{equation}
where $\nabla_{\mbx} \psi$ is a gradient field and $\mbh$ is a $q$-divergence-free vector field \cite{bhatia2012helmholtz}.
$\mbSigma \nabla_{\mbx} \psi$ points in directions that alter the one-time marginals, while $\mbh$ points in directions that leave them invariant.
The $q$-weighted Helmholtz decomposition thus separates $\mbr$ into exactly the two components we care about: the part we must leave alone to preserve marginals, and the part we are free to add. 

Substituting $\mbr = \mbf_p - \mbf_q$ into \cref{eqn:optimize} and using the $q$-weighted Helmholtz decomposition yields
\begin{equation}\label{eqn:helm-pythagorean}
    \bbE_{\bbQ_{\mbphi}} \| \mbSigma(\mbtheta)^{-1/2}\{\mbh(\mbx(t), t) - \mbv(\mbx(t), t)\}\|_2^2 + \bbE_{\bbQ_{\mbphi}}\|\mbSigma(\mbtheta)^{1/2}  \nabla_{\mbx} \psi(\mbx(t), t)\|_2^2.
\end{equation}
The cross term vanishes by integration by parts, using that $\mbh$ and $\mbv$ are $q$-divergence-free. The second term in \cref{eqn:helm-pythagorean} does not depend on $\mbv$, so the optimum of  \cref{eqn:optimize} is achieved by $\mbv^\star = \mbh$. 

In other words, among all diffusions sharing the prescribed marginals, the negative ELBO is minimized by adding the $q$-divergence-free component of $\mbr(\mbx, t | \mbtheta, \mbphi)$ to the reference drift.
While previous simulation-free VI algorithms commit to a single drift consistent with a set of marginals, we have shown that it is possible, in theory, to optimize over \emph{all} such drifts.

\begin{algorithm}[t]
\caption{Helmholtz-SDE training}
\label{alg:helmholtz_sde_main}
\Require{data $\mathcal D$; prior drift $\mbf_p$; Gaussian marginals $(\mbm,\mbS)$; reference drift $\mbf_q$; degree $\ell$}

\For{gradient iterations}{
    \AlgStep{Sample observation times, states to estimate $\widehat{\cL}_{\mathrm{rec}}$; compute closed-form Gaussian KL $\cL_{\mathrm{init}}$}

    \smallskip
    \AlgStep{Sample KL times $s_j \sim \mathrm{Unif}[0,T]$ and states $\mbx_{j,k} \sim \cN(\mbm(s_j | \mbphi),\mbS(s_j | \mbphi))$}

    \smallskip
    \AlgStep{Compute reference drift 
    $(\mbf_q)_{j,k} \leftarrow \mbf_q(\mbx_{j,k}, s_j|\mbphi)$ and degree $\ell$ Helmholtz correction $\widehat{\mbh}_{j,k}^{(\ell)} \leftarrow \widehat{\mbh}^{(\ell)}(\mbx_{j,k}, s_j)$ via \Cref{app:alg:helm_correction}}

    \smallskip
    \AlgStep{Set $(\mbf_{q,\mathrm{full}})_{j,k} \leftarrow (\mbf_q)_{j,k} + \widehat{\mbh}_{j,k}^{(\ell)}$ and estimate $\widehat{\cL}_{\mathrm{path}}$} 

    \smallskip
    \AlgStep{Update $(\mbtheta,\mbphi)$ by a gradient step}
}
\end{algorithm}
\vspace{-1em}

\subsection{A variational inference algorithm for Gaussian one-time marginals}\label{subsec:gaussian-marginals}

Next, we translate our insights into a simulation-free VI algorithm that optimizes over multiple path laws for each collection of one-time marginals.
We follow prior works \cite{course2023state,course2023amortized,bartosh2025sde} and take the variational family to have Gaussian one-time marginals $q(\mbx, t | \mbphi) = \cN(\mbm(t | \mbphi), \mbS(t | \mbphi))$ with $\mbm(t | \mbphi)$ and $\mbS(t | \mbphi)$ continuously differentiable in $t$ and $\mbS(t | \mbphi)$ positive definite. 
We further restrict to state-independent diffusion; we present the time-independent case $\mbSigma(\mbtheta)^{1/2}$ for simplicity.

To leverage the simulation-free stochastic VI framework from \Cref{subsec:sim-free}, we must compute a variational drift for each collection of marginals. 
By \Cref{subsec:helmholtz-decomp}, the optimal drift is given by the $q$-weighted Helmholtz decomposition of the residual $\mbr = \mbf_p - \mbf_q$, where we choose the reference drift to be that in \cref{eqn:sym-sqrt-reference}. 
However, even for Gaussian marginals, the weighted Helmholtz decomposition \cref{eqn:helmholtz-decomp} is generally not closed-form. 
We describe a tractable approximation to the optimal variational drift and then state the resulting theoretical guarantee. 
\Cref{alg:helmholtz_sde_main} summarizes Helmholtz-SDE.

\paragraph{Polynomial approximation.}
For $K \geq 2$, the vector space of $q$-divergence-free vector fields is infinite-dimensional.
To obtain a tractable approximation to $\mbh$ in the weighted Helmholtz decomposition of $\mbr$, we restrict the optimization problem in \cref{eqn:optimize} to the finite-dimensional subspace of $q$-divergence-free vector fields whose components are polynomials of degree at most $\ell$.

To do so, we first introduce a change of variable $\tilde\mbx$ such that the one-time marginal at time $t$ is standard Gaussian.
We then expand the residual vector field $\mbr(\mbx, t)$, transformed into the new coordinates as $\mbrho(\tilde\mbx,t)$, in the basis of Hermite polynomials:
\begin{equation}\label{eqn:hermite-expansion}
\mbrho_i(\tilde\mbx,t) = \sum_{\mbalpha} \widetilde{\mbb}_{\mbalpha,i}(t) \hat{H}_{\mbalpha}(\tilde\mbx), \quad \widetilde{\mbb}_{\mbalpha,i}(t)= \bbE_{\tilde\mbx\sim\cN(\mb0,\bbI_K)} \left[\mbrho_i(\tilde\mbx,t)\hat{H}_{\mbalpha}(\tilde\mbx)\right], \quad 1\leq i\leq K.
\end{equation}
The functions $\hat{H}_{\mbalpha}(\tilde\mbx)$ are the normalized multivariate probabilists' Hermite polynomials.

We also expand the $q$-divergence-free vector field, written in the new coordinates as $\mbnu(\tilde\mbx,t)$, in the same basis, and we truncate this expansion to basis functions of degree at most $\ell$.
By orthogonality of the Hermite basis, the terms in  \cref{eqn:optimize} that depend on the truncated vector field $\mbnu^{(\ell)}(\tilde\mbx,t)$ involve only the coefficients of $\mbrho(\tilde\mbx,t)$ up to degree $\ell$.
Therefore, we can obtain the optimal $q$-divergence-free vector field $\mbh^{(\ell)}$ within this subspace by computing the coefficients $\widetilde{\mbb}_{\mbalpha,i}(t)$ up to degree $\ell$, solving the objective for $\mbnu^{(\ell)}(\tilde\mbx,t)$, and transforming back to $\mbx$. 
See \Cref{app:subsec:resid-approx} for details.

In general, the Hermite coefficients in \cref{eqn:hermite-expansion} cannot be evaluated in closed-form. 
One can approximate $\widetilde{\mbb}_{\mbalpha,i}(t)$ using Gauss--Hermite quadrature in low latent dimensions or with Monte Carlo.
The resulting vector field $\widehat{\mbh}^{(\ell)}$ is $q$-divergence-free but not generally equal to $\mbh^{(\ell)}$.
In \Cref{app:subsec:resid-approx}, we show that the complexity of computing $\widehat{\mbh}^{(\ell)}$ at a point $\mbx$ is $O(K^{\mathrm{max}(\ell + 1, 3)})$; for linear and quadratic approximations, the complexity is the same as with the fixed reference drift.
Corrections with $\ell \geq 2$ can produce posteriors with Gaussian one-time marginals but non-Gaussian joint distributions.

\paragraph{Optimality of Helmholtz-SDE.}
When the prior is linear, the initial distribution is Gaussian, and the observation model is linear-Gaussian, the true posterior is a Gaussian process given by the Kalman--Bucy smoother \cite{kalman1961new}. 
Existing simulation-free VI algorithms commit to a single posterior drift for each $(\mbm(t | \mbphi), \mbS(t | \mbphi))_{0 \leq t \leq T}$ and, as we show in \Cref{app:subsec:ou-theory}, can incur an \emph{unbounded KL gap to the true posterior} even in this conjugate setting.
Helmholtz-SDE provably closes the gap:
\begin{theorem}[Optimality of Helmholtz-SDE in the conjugate setting]\label{thm:helm-optimal}
   Suppose the prior $\bbP_{\mbtheta}$ is linear with a Gaussian initial law and the observation model is linear-Gaussian.
    Let $\cQ_{(1)}$ consist of laws from Helmholtz-SDE with $\ell{=}1$ applied to continuously differentiable $\mbm(t),\mbS(t)$ with $\mbS(t)$ positive definite.
    Then $\inf_{\bbQ\in\cQ_{(1)}}\KL{\bbQ}{\bbQ^\star}{=}0$, where $\bbQ^\star$ is the posterior.
\end{theorem}
\noindent\emph{Proof.}
In the conjugate setting, the residual $\mbr(\mbx, t)$ is affine. Gauss–Hermite quadrature with two nodes per latent dimension computes the degree-one Hermite coefficients, and hence the optimal correction, exactly; see \Cref{app:subsec:optimality}. \hfill$\square$

Outside the conjugate setting, the true posterior generally does not have Gaussian one-time marginals, so restricting the variational family to diffusion processes with Gaussian one-time marginals incurs a nonzero KL gap.
Nonetheless, when the Hermite coefficients are evaluated exactly, Helmholtz-SDE is guaranteed to perform at least as well as using the reference drift alone.
\Cref{app:subsec:optimality-nonconj} provides an estimable lower bound on the resulting improvement.
\begin{figure*}[!t]
  \centering
\includegraphics[width=\textwidth]{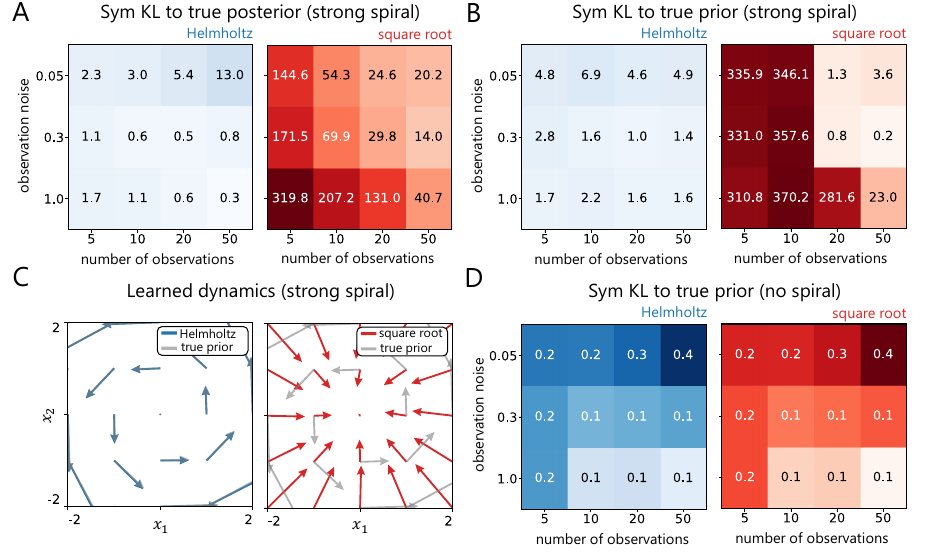}
  \caption{
    Comparison of Helmholtz-SDE and the square root gauge on two OU spiral datasets: strong spiral, $\omega{=}2\pi$ (\textbf{A}-\textbf{C}), and no spiral, $\omega{=}0$ (\textbf{D}). 
    \textbf{A}: The symmetric KL distance (nats) to the posterior across observation noise and number of observations.
    \textbf{B}: The symmetric KL distance to the true prior. 
    \textbf{C}: The dynamics learned by Helmholtz-SDE and the square root gauge against the true prior drift at noise $\sigma = 0.3$ and $N = 10$ observations.
    \textbf{D}: Same as \textbf{B}, but for the no spiral dataset.
  } 
  \label{fig:spiral-results}
  \vspace{-1em}
\end{figure*}

\section{Related Work}

\paragraph{Dynamical systems models.}
Beyond the VI algorithms for latent SDEs discussed in \Cref{sec:background}, several adjacent works model continuous-time, irregularly sampled data. \citet{neklyudov2023action} and \citet{zhang2024trajectory} propose simulation-free methods inspired by score-based generative modeling \cite{hyvarinen2005estimation,songscore2021,lipman2023flow,albergo2025stochastic}, but model the process directly in observation space rather than via a latent SDE. \citet{daemsvariational2024} study latent SDEs driven by fractional Brownian motion. \citet{kiyohara2025neural} develop a VI algorithm permitting simulation-free sampling from both prior and posterior but with a mean-field variational family. A separate line of work replaces the SDE with a latent ODE \cite{heinonen2018learning,chen2018neural,rubanova2019latent,yildiz2019ode2vae}. None of these works address the path-law non-uniqueness arising from marginal-based posterior parameterizations.

\paragraph{Divergence-free constraints.}
Divergence-free vector fields and Helmholtz decompositions appear across machine learning in problems with physical or geometric structure: as parameterizations encoding continuity \cite{richter2022neural} or incompressibility \cite{tompson2017accelerating,li2026project}, and in methods that model or decompose non-gradient vector-field structure \cite{berlinghieri2023gaussian,xu2024hhd,qi2024hdnet,horvatgauge2024,petrovic2025curly}.
Our use is different: rather than imposing a divergence-free condition on the dynamics, we use the Helmholtz decomposition to characterize, at fixed one-time marginals, the family of compatible path laws and the optimal correction to a reference drift.
\section{Experiments}\label{sec:experiments}

\begin{figure*}[!t]
  \centering
  \includegraphics[width=\textwidth]{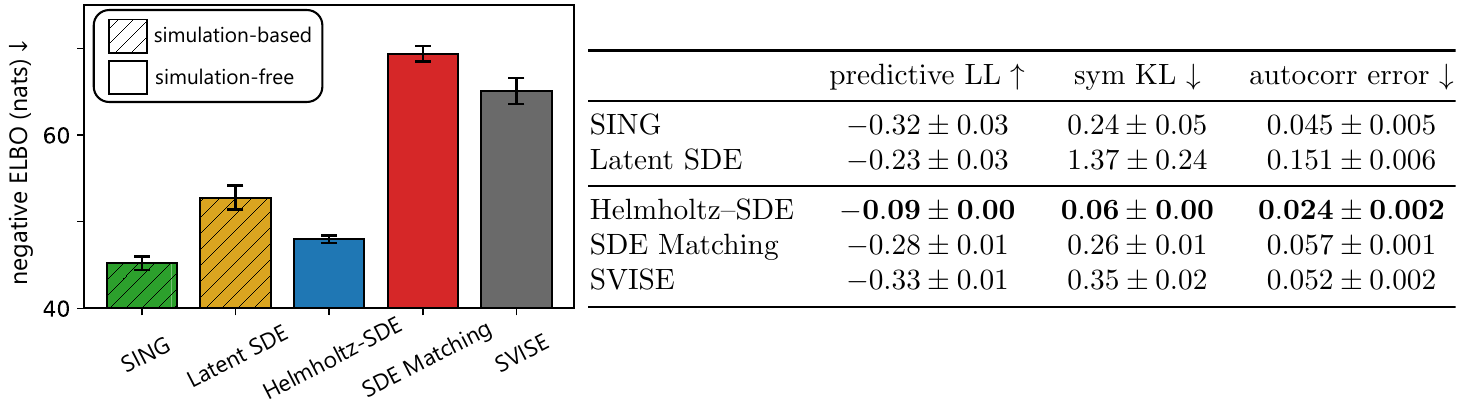}

    \caption{Helmholtz-SDE closes the gap between simulation-based and simulation-free VI algorithms on the noisy Lorenz attractor. nELBO and posterior predictive LL at unobserved times are averaged over trials; sym KL is averaged over one-time marginals. Autocorr error is the entrywise RMSE of the global $3{\times}3$ autocorrelation matrix, over lags $\tau \in [0,2]$. 
    Errors are $\pm2$SE over 10 replicates.}
  \label{fig:lorenz}
  \vspace{-1em}
\end{figure*}

We validate experimentally that Helmholtz-SDE shrinks the expressivity gap between simulation-free and simulation-based VI, with downstream improvements in dynamics learning. 
We begin with a linear, conjugate setting where the posterior is tractable and the gap can be quantified directly (\Cref{subsec:ou}), then move to nonlinear synthetic dynamics (\Cref{subsec:lorenz}) and two real-world systems: predator-prey cycles (\Cref{subsec:predator-prey}) and reduced order fluid dynamics (\Cref{subsec:rom}).

Throughout, we match implementation choices across methods.
In particular, existing VI algorithms differ in how they amortize the posterior approximation: across time, by representing it as a function of $t$, and across trials, by using an encoder to produce trial-specific parameters. 
Both strategies reduce memory at some cost to expressivity. 
We use the same amortization scheme whenever a method supports it.
For Helmholtz-SDE we use the degree $\ell =1$ polynomial approximation in all experiments; we provide ablations of the polynomial degree in \Cref{app:sec:nonlinear-sde}.
Full experimental details are in \Cref{app:sec:ou-spiral,app:sec:nonlinear-sde}.
We provide runtime comparisons against all baselines in \Cref{app:sec:timing}.

\subsection{Ornstein--Uhlenbeck (OU) spiral}\label{subsec:ou}
We first study the impact of augmenting the variational family in the conjugate setting, where the suboptimality of both the variational posterior and the learned prior can be analyzed directly. We consider the two-dimensional OU spiral
\begin{equation}
\label{eqn:ou-spiral}
    \bbP\!: \; d \mbx(t) = (-\alpha \bbI_2 + \omega \mbJ)\mbx(t)dt + d\mbw(t), 
    \quad
    p(\mbx,0) = \cN\left(\mb0, \tfrac{1}{2\alpha}\bbI_2\right),
    \quad
    \mbJ = \bigl(\begin{smallmatrix} 0 & -1 \\ 1 & 0 \end{smallmatrix}\bigr)
\end{equation}
whose drift decomposes into a contracting gradient field $-\alpha \mbx$ and a $p$-divergence-free rotational field $\omega \mbJ \mbx$. Observations follow the linear-Gaussian model $p(\mby(t_i) | \mbx(t_i)) = \cN(\mbC\mbx(t_i)+\mbd, \sigma^2 \mbC\mbC^\top)$.

\paragraph{Suboptimality of inference.}
\Cref{thm:helm-optimal} establishes that Helmholtz-SDE has zero KL gap to the posterior in this setting. In \Cref{app:subsec:ou-theory}, we show that the choices of variational drift made by SDE Matching \cite{bartosh2025sde} and SVISE \cite{course2023state,course2023amortized}, which we refer to as the ``square root'' and  ``symmetric'' gauges, respectively, incur a KL gap that grows at least linearly in $\omega$. We validate this in simulation. Fixing $T_{\max} = 5$ and $\alpha = 0.2$, we vary observation noise $\sigma$, number of observations $N$, and rotational speed $\omega$. 
\Cref{fig:spiral-results}A shows that when prior rotation is strong, Helmholtz-SDE achieves a variational approximation orders of magnitude closer in symmetric KL to the true posterior than the square root gauge, with the largest gains at sparse observations and large noise. 
When $\omega = 0$ (\Cref{fig:spiral-zero}), the posterior drift contains no divergence-free component, in which case the square root gauge can represent it exactly. 
Results for the symmetric gauge and $\omega \in \{\pi/4, \pi/2, \pi\}$ are in \Cref{app:subsec:ou-spiral-supplement}.

\paragraph{Degradation in learning.}
We next test whether posterior misspecification degrades learning. 
Using the same grid of $(\sigma, N, \omega)$ values, we jointly learn the prior (parameterized by a neural network drift) and observation model from $1024$ independent trials. 
\Cref{fig:spiral-results}B and D show that the regimes where the square root gauge yields the worst posterior approximation are also where prior learning degrades most. 
In many cases, the divergence-free correction is the difference between recovering the correct drift and learning an incorrect one (\Cref{fig:spiral-results}C).
However, consistent with previous work \cite{turner2011two}, better variational approximations do not translate monotonically to improvements in learning.

\begin{figure*}[!t]
  \centering
  \includegraphics[width=\textwidth]{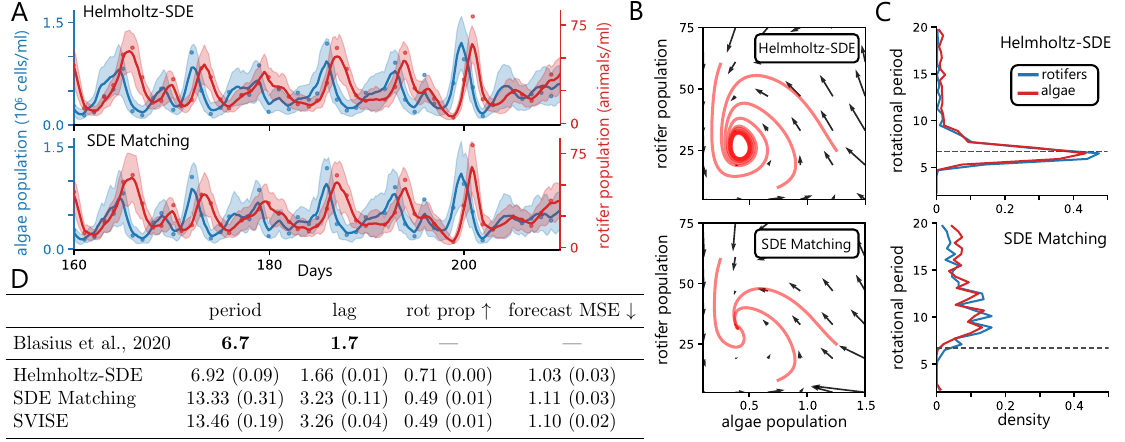}
  \caption{Helmholtz-SDE recovers predator-prey dynamics more faithfully than SDE Matching. \textbf{A}: Posterior mean $\pm 2$ SD over a representative window. 
  \textbf{B}: Phase portraits of learned dynamics: Helmholtz-SDE recovers a limit cycle, SDE Matching decays to a fixed point. 
  \textbf{C}: Peak-period distribution across 100-day forward simulations from the learned prior; dashed line marks the empirical period of 6.7 days.
  \textbf{D}: Rotation period, predator lag, and proportion of time in coherent rotation computed per replicate from forward simulations using the wavelet analysis of \citet{blasius2020long}; forecast MSE, in log coordinates. 
  Values are median (IQR) across 10 replicates.
 }
  \label{fig:rotifer-algae}
  \vspace{-1em}
\end{figure*}

\subsection{Noisy Lorenz attractor}\label{subsec:lorenz}

In the following three subsections, we demonstrate that the advantages from \Cref{subsec:ou} extend to nonlinear dynamics. 
We first consider Lorenz attractor dynamics, a classical chaotic system. 
The standard setup of \citet{li2020scalable} uses small process noise and dense, low-noise observations; in this regime the posterior concentrates tightly around the true latent trajectory (\Cref{fig:lorenz-compare}A). 
To simulate a setting with non-trivial posterior uncertainty, we sample latents on $[0, 5]$ with increased process noise and observe them every $\Delta t = 0.25$ (\Cref{fig:lorenz-compare}B). 
We perform joint inference and learning using $1024$ trials with $K{=}4$ latent dimensions, evaluating both one-time (symmetric KL to the true prior) and multi-time (lagged autocorrelation error) statistics of the prior.

We compare Helmholtz-SDE to simulation-free (SDE Matching \cite{bartosh2025sde}, SVISE \cite{course2023state,course2023amortized}) and simulation-based (SING \cite{hu2025sing}, Latent SDE \cite{li2020scalable}) baselines (\Cref{fig:lorenz}). 
Among simulation-free methods, Helmholtz-SDE achieves a ${\sim}20$\,nat smaller negative ELBO than the alternatives, reflecting the inability of a pre-specified drift to represent the posterior path law. 
As a result, only Helmholtz-SDE accurately learns the outward-spiraling dynamics within each lobe of the attractor (\Cref{fig:lorenz-samples}). 
Helmholtz-SDE also matches or exceeds simulation-based methods in dynamics recovery. 
SING attains a slightly tighter nELBO but at substantially higher cost: on an H100 GPU, Helmholtz-SDE reaches per-trial nELBO below $50$\,nat in $123.7 \pm 3.4$\,s, whereas SING requires $3695.6 \pm 503.8$\,s, over $25\times$ slower.

\subsection{Predator-prey system}\label{subsec:predator-prey}
Predator-prey systems exhibit cyclic dynamics in which the predator population lags the prey, and recovering this lag structure is important to ecological interpretation \cite{lotka1925elements,volterra1926fluctuations,elton1942ten,utida1957cyclic,berryman2002population}.
Because the lag is a multi-time statistic, it directly probes the path law approximation made by simulation-free algorithms.

We study the dataset of \citet{blasius2020long}, who cultivate rotifers (\textit{B.~calyciflorus}, a freshwater zooplankton) feeding on algae in a controlled environment and measure both populations daily for up to 374 days. 
The system exhibits persistent cycles with a mean period of 6.7 days, in which the rotifer population lags the algae by approximately 1.7 days. 
Rather than imposing a biological model on dynamics \cite{mccauley1996structured}, we aim to learn a generative model of the dynamics directly from data.

To this end, we consider the longest trial and restrict to windows of coherent oscillation following \citet{blasius2020long}, yielding $255$ observations across four subtrials.
We hold out the last eight observations of each (${\sim}1$ cycle) for forecasting.
We then fit Helmholtz-SDE, SDE Matching, and SVISE on the log-transformed rotifer and algae populations with $K{=}2$ latent dimensions. 

All methods produce plausible posterior approximations over the observed populations (\Cref{fig:rotifer-algae}A).
However, the learned generative dynamics differ qualitatively: Helmholtz-SDE yields a stable limit cycle, while SDE Matching decays to a fixed point (\Cref{fig:rotifer-algae}B).
Moreover, the rotational dynamics learned by Helmholtz-SDE closely match the true biological system, with median period $6.92$ days and predator-prey lag $1.66$ days, both close to the estimates from \citet{blasius2020long} (\Cref{fig:rotifer-algae}C,D). 
SDE Matching, in contrast, recovers a period twice as long ($13.33$ days) and a correspondingly inflated lag ($3.23$ days). 
We provide details on our wavelet analysis in \Cref{app:subsec:predator-prey}.

\subsection{Reduced order models of fluid dynamics}\label{subsec:rom}

\begin{figure*}[!t]
  \centering
  \includegraphics[width=\textwidth]{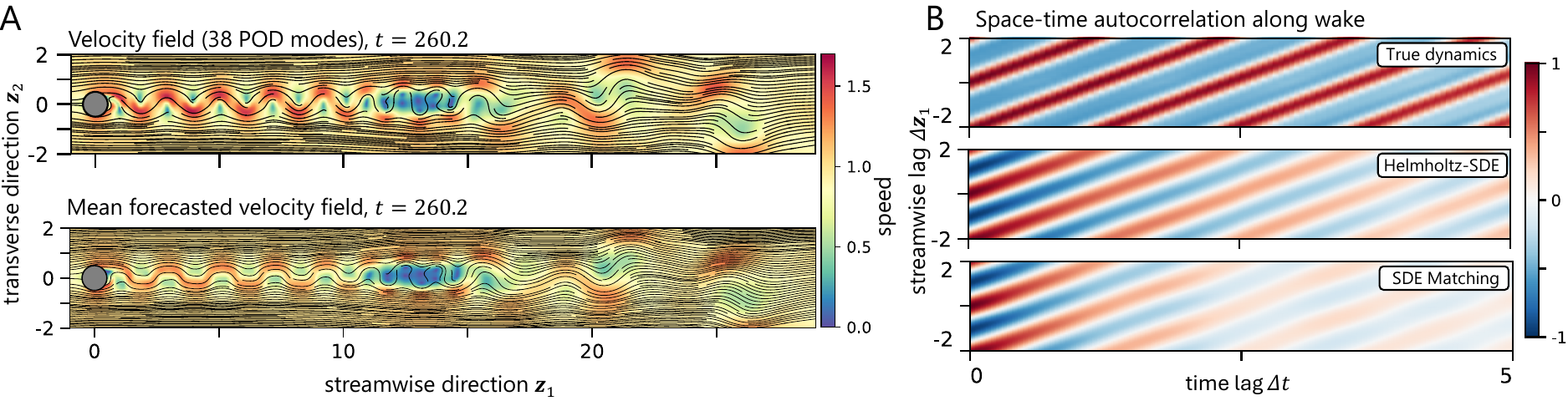}
    \caption{
    Helmholtz-SDE learns coherent wake dynamics in high-dimensional fluid flow past a 2D cylinder.
    \textbf{A}: True and forecasted mean velocity fields at $4$ time units into the forecast window.
    \textbf{B}: Joint space–time autocorrelation along the streamwise direction in the vortex region $[1, 10] \times [-0.1, 0.1]$, from prior samples at the true $t{=}0$ initial condition.
    Diagonal banding indicates vortices being advected downstream at a coherent speed; rapid decay reflects loss of temporal coherence. 
    }
  \label{fig:fluid-dynamics}
  \vspace{-1.5em}
\end{figure*}

As a final case study, we apply Helmholtz-SDE to a problem in computational fluid dynamics: modeling a two-dimensional incompressible flow past a circular cylinder at Reynolds number $2 \cdot 10^3$, a regime with strong, persistent vortices in the wake (\Cref{fig:fluid-dynamics}A). 
We use the dataset of \citet{course2023state}, in which the Navier--Stokes equations are discretized and solved on a grid of nearly $6 \cdot 10^5$ points using the immersed boundary projection method \cite{taira2007immersed}.

Reduced order modeling (ROM) provides a toolkit for representing such high-dimensional systems with low-dimensional surrogates. 
Classical projection-based approaches require explicit access to the discretized PDE \cite{brunton2020machine}.
Data-driven approaches instead learn reduced dynamics directly from snapshots of the time series, sidestepping projection of the PDE.
Following \citet{course2023state}, we apply dimensionality reduction (POD, \cite{sirovich1987turbulence, berkooz1993proper}) to velocity snapshots and retain $K = 38$ modes capturing $90\%$ of the kinetic energy.
We observe each POD coefficient, perturbed by Gaussian noise, every $\Delta t = 0.3$ time units over an interval of length $195.2$. 
We hold out $10$ time units for forecasting.

Helmholtz-SDE successfully learns a ROM of the cylinder dynamics (\Cref{fig:fluid-dynamics}A).
The space--time autocorrelation along the wake reveals that, compared to SDE Matching, Helmholtz-SDE better recovers the traveling-wave structure of vortex advection (\Cref{fig:fluid-dynamics}B). 
This translates to an approximately $10\%$ reduction in cumulative MSE over SDE Matching and SVISE over the forecast window, at modest additional training cost (approx.~3$\times$ the per-iteration cost of SVISE on H100 GPU). 
This example demonstrates that even for moderate latent dimensions $K$, our approximation to the Helmholtz decomposition remains tractable and yields improvements in dynamics learning.
\section{Discussion}\label{sec:discussion}
Simulation-free variational inference promises the accuracy of simulation-based algorithms at a fraction of the cost, but only when the variational family is rich enough to represent the posterior.
We have shown that under high posterior uncertainty, existing simulation-free algorithms fall short of their promise: by committing to an arbitrary drift consistent with their one-time marginals, they restrict the posterior to a subset of the path laws available to simulation-based methods. 
Helmholtz-SDE removes this restriction. 
Starting from any reference drift, it adds a closed-form correction that preserves the prescribed marginals while optimizing over a larger family of path laws. 
The result is a simulation-free algorithm that closes the expressivity gap to simulation-based VI and yields more faithful recovery of the underlying dynamics.

Two limitations point to natural extensions. 
First, the analytical form for the Helmholtz correction depends on the state independence of the diffusion coefficient and Gaussianity of the marginals. 
Relaxing either assumption, for example by learning the divergence-free correction, would broaden the algorithm's applicability.
Second, higher-order polynomial corrections to the residual scale unfavorably as the latent dimension grows; in our experiments we did not observe gains beyond the linear and quadratic approximations, but more efficient parameterizations may be beneficial in settings we have not explored. 
More broadly, we view simulation-free VI as a promising path toward scalable unsupervised discovery of dynamical systems, and Helmholtz-SDE as a step toward making that path competitive with, and ultimately preferable to, its simulation-based counterpart.

\begin{ack}
HDS is supported by the NSF Graduate Research Fellowship (DGE-2146755) and the Knight--Hennessy graduate fellowship. 
HDS and SWL are supported by grants from the NIH BRAIN Initiative (U19NS113201, R01NS131987, R01NS113119, \& RF1MH133778), the NSF/NIH CRCNS Program (R01NS130789), and the Sloan, Simons, and McKnight Foundations.
The authors have no competing interests to declare.

We thank Amber Hu, Cole Citrenbaum, Hannah Lee, Joshua Lunger, and Renato Berlinghieri for their helpful feedback on the manuscript. 
We thank Peter Pao--Huang for discussions at the beginning of the project regarding divergence-free vector fields. 
\end{ack}

\bibliographystyle{unsrtnat}
\bibliography{refs}
\clearpage
\appendix
\crefalias{section}{appendix}
\crefalias{subsection}{appendix}
\crefalias{subsubsection}{appendix}

\section*{Appendix}
\begin{itemize}
    \item[~] \Cref{app:sec:background}: Background on variational inference for latent SDEs
    \begin{itemize}
        \item \ref{app:subsec:assumptions}: Technical assumptions
        \item \ref{app:subsec:elbo}: ELBO derivation
        \item \ref{app:subsec:div-free}: Divergence-free vector fields
        \item \ref{app:subsec:gaussian}: Gaussian one-time marginals
        \item \ref{app:subsec:svise-sde-compare}: SVISE (\citet{course2023state}) and SDE-Matching (\citet{bartosh2025sde})
    \end{itemize}

    \item[~] \Cref{app:sec:helm-sde}: Helmholtz-SDE
    \begin{itemize}
        \item \ref{app:subsec:helmholtz-compute}: Computing the Helmholtz decomposition
        \item \ref{app:subsec:helmholtz-gaussian}: The Helmholtz decomposition for Gaussian marginals
        \item \ref{app:subsec:resid-approx}: Polynomial approximation to the optimal correction
        \item \ref{app:subsec:verify-assumptions}: Verification of assumptions for the Helmholtz-SDE variational family
        \item \ref{app:subsec:optimality}: Optimality of Helmholtz-SDE
        \begin{itemize}
            \item \ref{app:subsec:optimality-conj}: Conjugate setting
            \item \ref{app:subsec:optimality-nonconj}: Nonconjugate setting
        \end{itemize}
        \item \ref{app:subsec:alg-summary}: Algorithm summary
    \end{itemize}
    \item[~] \Cref{app:sec:ou-spiral}: Linear SDE experiments: Ornstein--Uhlenbeck spiral
    \begin{itemize}
        \item \ref{app:subsec:ou-theory}: Theoretical results for OU spiral
        \item \ref{app:subsec:ou-spiral-supplement}: Simulation experiment details
    \end{itemize}
    \item[~] \Cref{app:sec:nonlinear-sde}: Nonlinear SDE experiments
    \begin{itemize}
        \item \ref{app:subsec:lorenz}: Noisy Lorenz attractor
        \item \ref{app:subsec:predator-prey}: Experimental predator-prey system
        \item \ref{app:subsec:rom}: Reduced order models for fluid dynamics 
    \end{itemize}
    \item[~] \Cref{app:sec:timing}: Runtime comparisons
\end{itemize}

\clearpage

\paragraph{Notation.} 
Throughout the Appendix, for ease of notation, we drop the dependence of the prior and likelihood on the model parameters $\mbtheta$, as well as the dependence of the variational posterior on its parameters $\mbphi$.

\section{Background on variational inference for latent SDEs}\label{app:sec:background}

In this section, we provide background on variational inference for latent SDEs. 
First, in \Cref{app:subsec:assumptions} we describe our formal assumptions necessary for the probabilistic model and variational inference objective to be well-posed.   
In \Cref{app:subsec:elbo} we derive the evidence lower bound (ELBO) on the observed data log likelihood.
In \Cref{app:subsec:div-free} we provide further discussion of divergence-free vector fields and prove \Cref{prop:div_free}.
In \Cref{app:subsec:gaussian} we characterize the affine divergence-free vector fields for Gaussian one-time marginals (\Cref{app:prop:gaussian-compatible-drift}).
We conclude in \Cref{app:subsec:svise-sde-compare} by discussing two simulation-free VI algorithms, SVISE (\citet{course2023state, course2023amortized}) and SDE Matching (\citet{bartosh2025sde}). For each algorithm we derive the choice of posterior drift and discuss how this choice relates to Helmholtz-SDE.

For full generality, we allow the prior drift $\mbf_p$ to be time-dependent and for the diffusion coefficient $\mbSigma^{1/2}$ to be both space- and time-dependent.

\subsection{Technical Assumptions}\label{app:subsec:assumptions}
 
In this subsection, we detail our formal assumptions on the prior \eqref{eqn:prior_sde} and posterior \eqref{eqn:variational_family} processes. 
The assumptions are stated abstractly here and are verified for the variational family considered by Helmholtz-SDE in \Cref{app:subsec:verify-assumptions}.
They are required for every admissible choice of model parameters $\mbtheta$ and variational parameters $\mbphi$.  
Here, $\mbf_q$ denotes the \emph{full} posterior drift in \cref{eqn:variational_family}; for Helmholtz-SDE, this is the sum of the reference drift and the Helmholtz correction.

\paragraph{Existence and uniqueness.}
We first require that the prior and posterior SDEs, \cref{eqn:prior_sde} and \cref{eqn:variational_family}, respectively, admit unique strong solutions.
 
\begin{assumption}[SDE existence and uniqueness]\label{app:assum:regularity}
The prior and posterior SDEs admit unique strong solutions on $[0,T]$.
\end{assumption}

For the prior SDE, we follow the classical route of imposing Lipschitz continuity and linear growth on the drift and diffusion coefficient (see, e.g., \citet[Theorem 5.2.1]{oksendal2013stochastic}): there exist constants $C_1, C_2 > 0$ such that for all $\mbx, \mby \in \bbR^K$ and $t \in [0, T]$
\begin{equation*}
\begin{aligned}
    &\|\mbf_p(\mbx,t)- \mbf_p(\mby,t)\|_2 + \| \mbSigma(\mbx, t)^{1/2} -  \mbSigma(\mby, t)^{1/2} \|_F \leq  C_1\|\mbx - \mby\|_2, \quad &\text{(Lipschitz)}\\ 
    &\|\mbf_p(\mbx,t)\|_2 + \|\mbSigma(\mbx, t)^{1/2}\|_F \leq  C_2(1+\|\mbx\|_2). \quad &\text{(linear growth)}
\end{aligned}
\end{equation*}
For example, if $\mbf_p$ is a neural network that takes as input $\mbx$ and $t$, then Lipschitz continuity and linear growth can be ensured by restricting the operator norm of the weight matrices and choosing (globally) Lipschitz activation functions. 

For the posterior SDE, the polynomial approximation to the Helmholtz correction need not satisfy linear growth when the polynomial degree is $\ell > 1$.
We will therefore pursue an alternate route to ensure existence and uniqueness.

\paragraph{Densities.}
We additionally require that, for every $0 \leq t \leq T$, the one-time marginals of $\bbP$ and of any $\bbQ \in \cQ$ are absolutely continuous with respect to the $K$-dimensional Lebesgue measure. The corresponding densities are written as $\{p(\mbx, t)\}_{0 \leq t \leq T}$ and $\{q(\mbx, t)\}_{0 \leq t \leq T}$, respectively. 
 
\begin{assumption}[Marginal densities]\label{app:assum:densities}
For every $t\in[0,T]$, the one-time marginals of $\bbP$ and $\bbQ$ admit Lebesgue densities.
Moreover, $\mbSigma(\mbx,t)$ is positive-definite for all $(\mbx,t)\in\bbR^K\times[0,T]$.
\end{assumption}
 
\paragraph{Absolute continuity.}
In order for the upper bound to the negative marginal log likelihood (negative ELBO) to be well-defined, we need the approximate posterior SDE $\bbQ \in \cQ$ to possess a probability density with respect to the prior $\bbP$ and for the KL divergence between $\bbQ$ and $\bbP$ to be finite. 
These two requirements are the subjects of the following two assumptions.
 
\begin{assumption}[Mutual absolute continuity]\label{app:ass:path-ac}
The prior $\bbP$ and the variational posterior \cref{eqn:variational_family} path laws are mutually absolutely
continuous. 
Denote the density by $d\bbQ/d\bbP(\mbx)$.
\end{assumption}

\begin{assumption}[Finite relative entropy]\label{app:ass:finite-kl}
   The initial distributions satisfy
   \begin{equation*}
       \KL{q(\mbx,0)}{p(\mbx,0)}<\infty
   \end{equation*}
    and the difference between the prior and posterior drifts satisfies
    \begin{equation*}
        \int_0^{T} \bbE_{\bbQ}\| \mbSigma(\mbx(t), t)^{-1/2}\{\mbf_p(\mbx(t),t) -  \mbf_q(\mbx(t), t)\}\|_2^2 dt < \infty.
    \end{equation*}
\end{assumption}

Under these assumptions, we can obtain an expression for the probability density of $\bbQ$ with respect to $\bbP$ via Girsanov's theorem (see e.g., \citet[Corollary 3.5.13]{karatzas2014brownian}).
\begin{theorem}[Girsanov's theorem for SDEs]\label{app:thm:girsanov}
Suppose \Cref{app:assum:regularity,app:assum:densities,app:ass:path-ac,app:ass:finite-kl} hold. Then the probability density of $\bbQ \in \cQ$ with respect to $\bbP$ is
{\small
\begin{equation*}
\begin{aligned}
&\frac{d \bbQ}{d \bbP}(\mbx) = \frac{ q(\mbx(0), 0)}{p(\mbx(0), 0)}\exp\left\{ \int_0^{T} \sum_{i=1}^{K} \left[\mbSigma^{-1/2} \mbb(\mbx(t), t)\right]_i d\mbw_i(t) - \frac{1}{2} \int_0^{T} \sum_{i=1}^{K} \left[\mbSigma^{-1/2} \mbb(\mbx(t), t)\right]_i^2 dt \right\},\\
&\mbb(\mbx, t) = \mbf_q(\mbx, t) - \mbf_p(\mbx, t),
\end{aligned}
\end{equation*}
}
where $(\mbw(t))_{0 \leq t \leq T}$ is a $K$-dimensional Brownian motion with respect to $\bbP$. 

\smallskip
In particular, for any bounded, measurable function $\Phi$ defined on $K$-dimensional continuous functions $\mbx$ having domain $[0, T]$,
\begin{equation*}
    \bbE_{\bbQ}[\Phi(\mbx)] = \bbE_{\bbP}\left[\Phi(\mbx) \frac{d \bbQ}{d \bbP}(\mbx) \right].
\end{equation*}
\end{theorem}

As a direct corollary of \Cref{app:thm:girsanov}, we obtain an expression for the KL divergence between the posterior and prior
{\small
\begin{align}
    &\KL{\bbQ}{\bbP} = \bbE_{\bbQ}\left[\log\left(\frac{d\bbQ}{d\bbP}(\mbx)\right)\right]\nonumber\\
    &= \KL{q(\mbx, 0)}{p(\mbx, 0)} + \frac{1}{2} \int_0^{T} \bbE_{\bbQ}\| \mbSigma(\mbx(t), t)^{-1/2}\{\mbf_p(\mbx(t),t) -  \mbf_q(\mbx(t), t)\}\|_2^2 dt. \label{app:eqn:path-kl}
\end{align}
}
The stochastic integral term vanishes because, when expressed using the $\bbQ$-Brownian motion, it is a mean-zero martingale.

\paragraph{Fokker--Planck equation.}
For a drift $\mbf$ and initial density $\rho(\mbx,0)$, we call $\{\rho(\mbx,t)\}_{0\leq t\leq T}$ a \emph{weak solution} to the Fokker--Planck equation if for every infinitely differentiable function $\varphi: \bbR^K \to \bbR$  having compact support,
\begin{equation*}
    t \longmapsto \int \varphi(\mbx)\rho(d\mbx, t)
\end{equation*}
is absolutely continuous and 
\begin{equation}\label{app:eqn:fpk-weak}
    \frac{d}{dt} \int \varphi(\mbx) \rho(d\mbx, t) = \int\left( \mbf(\mbx, t) \cdot \nabla_{\mbx} \varphi(\mbx) + \frac{1}{2} \mathrm{tr}(\mbSigma(\mbx, t) \nabla_{\mbx}^2 \varphi(\mbx)) \right) \rho(d\mbx, t), \quad \text{a.e.~$t$}.
    \end{equation}

Since $\bbQ \in \cQ$ is a solution to the SDE \cref{eqn:variational_family},  we know that $\{q(\mbx, t)\}_{0 \leq t \leq T}$ is a weak Fokker--Planck solution.
However, in order to argue that two variational SDEs whose drifts differ by a divergence-free vector field induce the same marginals, we need the solution to the Fokker--Planck equation to be unique within a suitable regularity class.
    
\begin{assumption}[Uniqueness of the Fokker--Planck solution]\label{app:ass:fpk-unique}
For every member $\bbQ \in \cQ$ of the variational family having drift $\mbf_q$, the marginal densities belong to the class
\begin{equation}\label{app:eqn:fpk-regularity-class}
        \int_0^T\int \left( \|\mbf_q(\mbx,t)\|_2 + \|\mbSigma(\mbx,t)\|_F \right) \rho(d\mbx,t) dt <\infty.
\end{equation}
The associated weak Fokker--Planck equation with initial law $q(\mbx, 0)$ has at most one solution belonging to the class \cref{app:eqn:fpk-regularity-class}.
\end{assumption}
    
\subsection{ELBO derivation}\label{app:subsec:elbo}
In this subsection, we derive the evidence lower bound (ELBO) in the latent SDE model. 

Adopting the notation from \Cref{app:thm:girsanov}, we have
\begin{equation*}
\begin{aligned}
    &\log p(\mby) = \log \int p(\mby | \mbx) \bbP(d\mbx)\\
    &= \log \int p(\mby | \mbx) \frac{d \bbP}{d \bbQ}(\mbx) \bbQ(d\mbx)\\
    &\overset{(i)}{\geq}  \int \log \left( p(\mby | \mbx)  \frac{d \bbP}{d \bbQ}(\mbx) \right) \bbQ(d\mbx)
    \\
    &= \bbE_{\bbQ}[\log p(\mby | \mbx)] - \KL{\bbQ}{\bbP}\\
    &\overset{(ii)}{=} \bbE_{\bbQ}[\log p(\mby | \mbx)] - \KL{q(\mbx, 0)}{p(\mbx, 0)} - \frac{1}{2} \int_0^{T} \bbE_{\bbQ}\| \mbSigma^{-1/2}\{\mbf_p(\mbx(t),t) -  \mbf_q(\mbx(t),t)\}\|_2^2 dt.
\end{aligned}
\end{equation*}
Inequality (i) is an application of Jensen's inequality and equality (ii) is an application of Girsanov's theorem. Taking the negative of this expression yields the negative ELBO \cref{eqn:elbo}. 

Notice that each term in the negative ELBO depends both on the parameters of the generative model $\mbtheta$ and on the parameters of the variational posterior $\mbphi$.

\subsection{Divergence-free vector fields}\label{app:subsec:div-free}
In this subsection, we discuss the role of $q$-divergence-free vector fields in determining the expressivity of the variational family $\cQ$. We prove \Cref{prop:div_free}, which characterizes the non-uniqueness in path law that arises from marginal-based parameterizations of the variational posterior.

In \cref{eqn:fpk}, the Fokker--Planck equation is stated as a pointwise equality for $(\mbx, t) \in \bbR^K \times [0, T]$; this is only valid under regularity assumptions on the posterior SDE beyond those in \Cref{app:subsec:assumptions}.
Still, we can state the divergence-free condition in its \emph{weak form}, i.e., using the weak form of the Fokker--Planck equation \cref{app:eqn:fpk-weak}
\begin{equation}\label{app:eqn:div-free-weak}
    \int \mbv(\mbx, t) \cdot \nabla_{\mbx} \varphi(\mbx)  q( d\mbx, t)  = 0 \quad \text{a.e.~$t$}
\end{equation}
for all infinitely differentiable test functions $\varphi$ having compact support. Indeed, if $\mbv(\mbx, t)$ and $q(\mbx, t)$ are smooth then integration by parts implies
\begin{equation*}
    0 = \int \mbv(\mbx, t) \cdot \nabla_{\mbx} \varphi(\mbx)  q(d\mbx, t) = - \int \nabla_{\mbx} \cdot[ q( \mbx, t) \mbv(\mbx, t)] \varphi(\mbx)  d \mbx \quad \text{a.e.~$t$}
\end{equation*}
for all test functions $\varphi$. By Fubini's theorem, this implies $\nabla \cdot[ q( \mbx, t) \mbv(\mbx, t)] = 0$ a.e.~$(\mbx, t)$.

Having stated the divergence-free condition in its weak form, we are prepared to prove \Cref{prop:div_free}, which states that at fixed one-time marginals, among diffusion processes having a shared diffusion coefficient $\mbSigma^{1/2}$, the only degree of freedom is the choice of divergence-free vector field.

\begin{proof}[Proof of \Cref{prop:div_free}]
Let $\bbQ,\widetilde{\bbQ}\in\cQ$ have drifts $\mbf$ and
$\widetilde{\mbf}$, respectively, common diffusion coefficient
$\mbSigma^{1/2}$, and common initial marginal density $q(\mbx,0)$.
Let $\{q(\mbx,t)\}_{0\leq t\leq T}$ denote the marginal densities of
$\bbQ$.

Suppose first that $\bbQ$ and $\widetilde{\bbQ}$ have the same one-time marginals. 
Since these common marginals solve the weak Fokker--Planck equation for both drifts, and the diffusion coefficients are equal, subtracting the two equations gives
\begin{equation}\label{app:eqn:div-free-weak-diff}
    \int \left(\mbf(\mbx,t)-\widetilde{\mbf}(\mbx,t)\right) \cdot\nabla\varphi(\mbx) q(d\mbx,t) =0, \qquad\text{a.e.~$t$},
\end{equation}
for every test function $\varphi$. This is precisely the statement that $\mbf-\widetilde{\mbf}$ is $q$-divergence-free.

Conversely, suppose that $\mbf-\widetilde{\mbf}$ is $q$-divergence-free. 
Since $q$ solves the weak Fokker--Planck equation associated with $\mbf$, \cref{app:eqn:div-free-weak-diff} implies that it also solves the weak Fokker--Planck equation associated with $\widetilde{\mbf}$ and the same initial density. 
By \Cref{app:ass:fpk-unique}, this solution is unique. 
It therefore coincides with the marginal densities of $\widetilde{\bbQ}$, so $\bbQ$ and $\widetilde{\bbQ}$ have the same one-time marginals.
\end{proof}

There are also diffusion processes with different diffusion coefficients that share the same one-time marginals. This fact is leveraged by \citet{bartosh2025sde} to learn the shared prior and posterior diffusion coefficient. 
However, for a fixed prior SDE \cref{eqn:prior_sde}, the diffusion coefficient of the posterior \cref{eqn:variational_family} must be equal to that of the prior. 
This is because, by Girsanov's theorem (\Cref{app:thm:girsanov}), diffusion processes whose diffusion coefficients differ are generally not mutually absolutely continuous. 
In other words, the diffusion coefficient must be viewed as a parameter of the generative model, rather than a degree of freedom in the posterior process. 
On the other hand, the prior and variational posterior path laws are mutually absolutely continuous for all admissible choices of variational drift.

\subsection{Gaussian one-time marginals}
\label{app:subsec:gaussian}

In the case of Gaussian one-time marginals, the non-uniqueness of the path law described in \Cref{app:subsec:div-free} can be made explicit. In this subsection, we provide a characterization of the affine vector fields that are $q$-divergence free.

Let $(\mbm(t),\mbS(t))_{0\leq t\leq T}$ be continuously differentiable and $\mbS(t)$ be positive definite. 
For a matrix-valued function $\mbK(t)$, define the affine velocity field
\begin{equation}\label{eq:gaussian-affine-velocity}
    \mbv_{\mbK}(\mbx,t) = \frac{d}{d t}\mbm(t) + \mbK(t)(\mbx- \mbm(t)).
\end{equation}
And consider the SDE having diffusion coefficient $\mbSigma(\mbx,t)^{1/2}$ and drift
\begin{equation}\label{app:eq:gaussian-compatible-drift}
\begin{aligned}
    &\mbf_{q}(\mbx,t) = \mbv_{\mbK}(\mbx,t) + \frac{1}{2}\mbSigma(\mbx,t)\nabla_{\mbx}\log q(\mbx,t) + \frac{1}{2}\sum_j \partial_{x_j}\mbSigma_{\cdot,j}(\mbx,t)\\
    &= \frac{d}{dt}\mbm(t) + \mbK(t)(\mbx-\mbm(t)) - \frac{1}{2}\mbSigma(\mbx,t)\mbS(t)^{-1}(\mbx-\mbm(t)) + \frac{1}{2}\sum_j \partial_{x_j}\mbSigma_{\cdot,j}(\mbx,t).
\end{aligned}
\end{equation}

In \Cref{app:prop:gaussian-compatible-drift}, we characterize the choices of $\mbK(t)$ for which this SDE produces one-time marginals $q(\mbx,t) = \cN(\mbm(t),\mbS(t))$. 
\begin{proposition}
\label{app:prop:gaussian-compatible-drift}
Fix any continuously differentiable square-root factorization $\mbL(t)$ satisfying
$\mbL(t)\mbL(t)^\top=\mbS(t)$. Then the following statements are equivalent:
\begin{enumerate}[label=(\roman*)]
    \item The SDE
    \begin{equation*}
        d\mbx(t) = \mbf_{q}(\mbx(t),t)dt + \mbSigma(\mbx(t),t)^{1/2}d\mbw(t), \quad q(\mbx, 0) =  \cN(\mbm(0),\mbS(0))
    \end{equation*}
    has one-time marginals $q(\mbx,t)=\cN(\mbm(t),\mbS(t))$.

    \item $\mbK(t)$ satisfies the matrix Lyapunov equation
    \begin{equation} \label{app:eq:gaussian-lyapunov}
        \frac{d}{d t}\mbS(t) = \mbK(t)\mbS(t) + \mbS(t)\mbK(t)^\top, \qquad 0\leq t\leq T.
    \end{equation}

    \item There exists a unique skew-symmetric matrix-valued function
    $\mbOmega(t)=-\mbOmega(t)^\top$ such that
    \begin{equation}\label{app:eq:gaussian-square-root}
        \mbK(t) = \left(\frac{d}{d t}\mbL(t)\right)\mbL(t)^{-1} + \mbL(t)\mbOmega(t)\mbL(t)^{-1}, \qquad 0\leq t\leq T.
    \end{equation}
\end{enumerate}
\end{proposition}

\begin{proof}
Before proving the equivalence of (i), (ii), and (iii), we show that the Fokker--Planck equation associated with
the drift \eqref{app:eq:gaussian-compatible-drift} and diffusion coefficient $\mbSigma(\mbx, t)^{1/2}$ reduces to a continuity equation. 
Using $q(\mbx, t)\nabla_{\mbx}\log q(\mbx, t)=\nabla_{\mbx} q(\mbx, t)$,
\begin{equation*}
\begin{aligned}
    q(\mbx,t)\mbf_{q}(\mbx,t) &= q(\mbx,t)\mbv_{\mbK}(\mbx,t) + \frac{1}{2}\mbSigma(\mbx,t)\nabla_{\mbx}q(\mbx,t) + \frac{1}{2}q(\mbx,t) \sum_j\partial_{x_j}\mbSigma_{\cdot,j}(\mbx,t)\\
    &= q(\mbx,t)\mbv_{\mbK}(\mbx,t) + \frac{1}{2}\sum_j \partial_{x_j}\left[ q(\mbx,t)\mbSigma_{\cdot,j}(\mbx,t)\right].
\end{aligned}
\end{equation*}
Consequently, we can rewrite the Fokker--Planck equation as
\begin{align}
    \frac{\partial}{\partial t}q(\mbx,t) &= -\nabla_{\mbx}\cdot \left[ q(\mbx,t)\mbf_{q}(\mbx,t) \right] + \frac{1}{2} \sum_{i,j} \partial_{x_i,x_j} \left[q(\mbx,t)\mbSigma_{ij}(\mbx,t)\right]\nonumber\\
    &= -\nabla_{\mbx}\cdot\left[q(\mbx,t)\mbv_{\mbK}(\mbx,t) \right].\label{app:eq:continuity-affine}
\end{align}
This is the continuity equation corresponding to the affine flow
\begin{equation}\label{app:eq:affine-flow}
    \frac{d}{d t}\mbx(t)
    =
    \mbv_{\mbK}(\mbx(t),t).
\end{equation}
Because the flow map is affine, it maps the Gaussian initial distribution $q(\mbx,0)=\cN(\mbm(0),\mbS(0))$ to a Gaussian distribution at every time.

To prove the equivalence of (i) and (ii), we must show that  $q(\mbx,t)=\cN(\mbm(t),\mbS(t))$ solves \cref{app:eq:continuity-affine} if and only if $\mbK(t)$ satisfies the Lyapunov equation \cref{app:eq:gaussian-lyapunov}.

Let $\widetilde{\mbm}(t) = \bbE[\mbx(t)]$
denote the mean induced by the affine flow \cref{app:eq:affine-flow}. 
By the weak form of the continuity equation \cref{app:eq:continuity-affine}, $\widetilde{\mbm}(t)$ satisfies the ODE
\begin{equation*}
    \frac{d}{d t}\widetilde{\mbm}(t) = \frac{d}{d t}\mbm(t) + \mbK(t) \left(\widetilde{\mbm}(t)-\mbm(t) \right), \quad 
    \widetilde{\mbm}(0) = \mbm(0).
\end{equation*}
It is immediate that $\mbm(t)$ satisfies the ODE for the mean of $\mbx(t)$. 
For the covariance, we define the centered state $\mbz(t) = \mbx(t)-\mbm(t)$, which satisfies the affine flow
\begin{equation*}
    \frac{d}{d t}\mbz(t) = \mbK(t)\mbz(t).
\end{equation*}
Therefore, if $\widetilde{\mbS}(t) = \bbE \left[ \mbz(t)\mbz(t)^\top \right]$ denotes the covariance induced by the affine flow, the weak form of the continuity equation yields the ODE
\begin{equation*}
    \frac{d}{d t}\widetilde{\mbS}(t) =\bbE\left[ \mbK(t)\mbz(t)\mbz(t)^\top + \mbz(t)\mbz(t)^\top\mbK(t)^\top\right] = \mbK(t)\widetilde{\mbS}(t) + \widetilde{\mbS}(t)\mbK(t)^\top, \quad  \widetilde{\mbS}(0)=\mbS(0).
\end{equation*}
If $\mbK(t)$ satisfies the Lyapunov equation
\cref{app:eq:gaussian-lyapunov}, uniqueness of solutions to the covariance ODE implies $\widetilde{\mbS}(t)=\mbS(t)$. 

Conversely, if $q(\mbx,t)=\cN(\mbm(t),\mbS(t))$ solves
\cref{app:eq:continuity-affine}, then its covariance must satisfy the same equation, so
\begin{equation*}
    \frac{d}{d t}\mbS(t) = \mbK(t)\mbS(t) + \mbS(t)\mbK(t)^\top.
\end{equation*}
This proves the equivalence of \textnormal{(i)} and \textnormal{(ii)}.

Suppose next that $\mbK(t)$ has the representation in
\eqref{app:eq:gaussian-square-root}. Since
$\mbS(t)=\mbL(t)\mbL(t)^\top$, we obtain
\begin{equation*}
\begin{aligned}
    \mbK(t)\mbS(t)+\mbS(t)\mbK(t)^\top &= \left(\frac{d}{d t}\mbL(t) \right)\mbL(t)^\top + \mbL(t) \left(\frac{d}{dt}\mbL(t)  \right)^\top + \mbL(t) \left( \mbOmega(t)+\mbOmega(t)^\top\right)\mbL(t)^\top\\
    &= \frac{d}{d t} \left[\mbL(t)\mbL(t)^\top\right] =
    \frac{d}{d t}\mbS(t).
\end{aligned}
\end{equation*}
Thus, (iii) implies (ii).

Conversely, suppose that $\mbK(t)$ satisfies \eqref{app:eq:gaussian-lyapunov} and define
\begin{equation*}
    \mbOmega(t) = \mbL(t)^{-1}\mbK(t)\mbL(t) - \mbL(t)^{-1}\frac{d}{d t}\mbL(t).
\end{equation*}
Then
\begin{equation*}
    \mbOmega(t) + \mbOmega(t)^\top  = \mbL(t)^{-1} \left[    \mbK(t)\mbS(t) + \mbS(t)\mbK(t)^\top - \frac{d}{d t}\mbS(t) \right] \mbL(t)^{-\top} =\mb{0}.
\end{equation*}
Hence, $\mbOmega(t)$ is skew-symmetric. 
Its definition also shows that it is unique for the fixed square root $\mbL(t)$, proving the equivalence of (ii) and (iii).
\end{proof}

For a fixed square root $\mbL(t)$, the skew-symmetric matrix $\mbOmega(t)$ parameterizes the
freedom among compatible affine velocities. In particular, changing $\mbOmega(t)$ adds the vector
field
\begin{equation*}
    \mbh_{\mbOmega}(\mbx,t) =  \mbL(t)\mbOmega(t)\mbL(t)^{-1}(\mbx - \mbm(t)),
\end{equation*}
which satisfies
\begin{equation*}
    \nabla_{\mbx}\cdot \left[ q(\mbx,t)\mbh_{\mbOmega}(\mbx,t) \right] = 0.
\end{equation*}
Thus, $\mbh_{\mbOmega}$ changes the path law without changing the one-time marginals. 
When $K\geq2$, the space of skew-symmetric matrices has dimension $\binom{K}{2}$.

The representation in \eqref{app:eq:gaussian-square-root} depends on the chosen square root $\mbL(t)$.
Replacing $\mbL(t)$ by another factor of the covariance $\mbS(t)$ reparameterizes $\mbOmega(t)$ but does not change the set of compatible $\mbK(t)$.

Finally, Proposition~\ref{app:prop:gaussian-compatible-drift} characterizes only the subfamily of affine $q$-divergence-free vector fields. 
There also exist nonlinear $q$-divergence-free vector fields that preserve Gaussian one-time marginals while inducing non-Gaussian multi-time distributions.

\subsection{SVISE (\citet{course2023state}) and SDE Matching (\citet{bartosh2025sde})}\label{app:subsec:svise-sde-compare}

In this subsection, we explain how existing simulation-free VI methods SVISE \cite{course2023state, course2023amortized} and SDE Matching \cite{bartosh2025sde} select a single drift function that realizes a given collection of one-time Gaussian marginal distributions $q(\mbx, t) = \cN(\mbm(t), \mbS(t))$, where $\mbm(t), \mbS(t)$ are continuously differentiable on $[0, T]$ and $\mbS(t)$ is positive definite. 
We refer to these choices as the symmetric and square root gauges, respectively.

Before describing each choice of posterior drift, we recall \Cref{app:prop:gaussian-compatible-drift} from \Cref{app:subsec:gaussian}, which tells us that
\begin{equation}\label{app:eqn:variational-drift}
\mbf_q(\mbx, t) = \frac{d}{dt}\mbm(t) + \mbK(t)(\mbx-\mbm(t)) - \frac{1}{2}\mbSigma(\mbx,t)\mbS(t)^{-1}(\mbx-\mbm(t)) + \frac{1}{2}\sum_j \partial_{x_j}\mbSigma_{\cdot,j}(\mbx,t)
\end{equation}
attains the marginals $q(\mbx, t)$ whenever $\mbK(t)$ is a solution to the Lyapunov equation \cref{app:eq:gaussian-lyapunov}.
As we will see, SVISE and SDE Matching select the variational drift \cref{app:eqn:variational-drift} using different choices of $\mbK$.

\paragraph{Symmetric gauge.}
Suppose $\mbSigma(t) = \mbSigma(\mbx, t)$ is state-independent and $\mbf_q$ is defined as in \Cref{app:prop:gaussian-compatible-drift}. 
Then  $\mbA_q(t) = \mbK(t) - \frac{1}{2}\mbSigma(t)\mbS(t)^{-1}$ also satisfies a Lyapunov equation
\begin{equation}\label{eqn:course-nair}
    \frac{d}{dt} \mbS(t) - \mbSigma(t)  = \mbA_q(t)\mbS(t) + \mbS(t) \mbA_q(t)^\top, \quad 0 \leq t \leq T.
\end{equation}

In SVISE \cite{course2023state, course2023amortized}, the authors choose the posterior drift $\mbA_q(t)$ to be the unique symmetric solution to \cref{eqn:course-nair}. They write this solution in closed-form using the Kronecker sum.

Unlike the square root gauge, the symmetric gauge does not naturally extend to the case of a state-dependent diffusion coefficient. This is because, in general, the matrix $\mbK$ for which $\mbA_q(\mbx, t) = \mbK - \frac{1}{2}\mbSigma(\mbx, t) \mbS(t)^{-1}$ is symmetric cannot be chosen to be independent of the state $\mbx$.

\paragraph{Square root gauge.}
\citet{bartosh2025sde} instead identify the one-time Gaussian marginals with the drift function $\mbf_q$ defined as in \Cref{app:prop:gaussian-compatible-drift} for $\mbK(t) = \bigl(\frac{d}{dt}\mbS(t)^{1/2}\bigr)\mbS(t)^{-1/2}$. 
Here, $\mbS(t)^{1/2}$ denotes the symmetric square root of $\mbS(t)$.
It is straightforward to verify that this choice of $\mbK$ satisfies the matrix Lyapunov equation \eqref{app:eq:gaussian-lyapunov}. This choice of $\mbK$ is valid even for a state-dependent diffusion coefficient.

The SDE Matching framework, in theory, allows for any square root $\mbsigma(t)$ such that $\mbS(t) = \mbsigma(t) \mbsigma(t)^\top$, to be chosen. 
However, in their implementation, $\mbsigma(t)$ is chosen to be diagonal, which yields the symmetric square root $\mbS(t)^{1/2}$.
At the end of this section, we discuss an extension of SDE Matching that represents all square roots simultaneously.

The authors arrive at the choice of variational drift by converting from a probability flow ODE \cite{songscore2021} to an SDE. In particular, suppose $\mbvarepsilon \sim \cN(\mb{0}, \bbI_K)$ and $F(\mbvarepsilon, t) = \mbm(t) + \mbS(t)^{1/2}\mbvarepsilon$.
$F$ can be viewed as a map that takes as input a sample $\mbvarepsilon$ as well as a time $t$ and produces a sample from $q(\mbx, t)$.
If $\tilde{\mbx}(t) = F(\mbvarepsilon, t)$, then each particle $\tilde{\mbx}(t)$ evolves according to the ODE
\begin{equation}\label{app:eqn:pf-ode}
    \begin{aligned}
    \frac{d}{dt} \tilde{\mbx}(t) &= \frac{d}{dt}\mbm(t) + \left(\frac{d}{dt}\mbS(t)^{1/2}\right)\mbvarepsilon\\
    &= \frac{d}{dt}\mbm(t) + \left(\frac{d}{dt}\mbS(t)^{1/2}\right)\{\mbS(t)^{-1/2}(\tilde{\mbx}(t) - \mbm(t))\}, \quad 0 \leq t \leq T
    \end{aligned}
\end{equation}
with $\tilde{\mbx}(0) \overset{d}{=} F(\mbvarepsilon, 0)$. 

From the ODE \eqref{app:eqn:pf-ode}, the authors construct an SDE with diffusion coefficient $\mbSigma(\mbx, t)^{1/2}$. A straightforward manipulation of the Fokker--Planck equation (see e.g., \citet[][Appendix D.1]{songscore2021}) shows that the solution to the SDE
\begin{equation*}
\begin{aligned}
    d \mbx(t) &= \bigg( \frac{d}{dt}\mbm(t) + \left(\frac{d}{dt}\mbS(t)^{1/2}\right)\{\mbS(t)^{-1/2}(\mbx(t) - \mbm(t))\} + \frac{1}{2} \mbSigma(\mbx(t), t) \nabla_{\mbx} \log q(\mbx(t), t)\\
    &+ \frac{1}{2} \sum_{j} \partial_{\mbx_j} \mbSigma_{\cdot, j}(\mbx(t), t) \bigg)dt + \mbSigma(\mbx(t), t)^{1/2} d\mbw(t), \quad \mbx(0) \sim q(\mbx, 0)
\end{aligned}
\end{equation*}
has the same one-time marginal distributions as $\tilde{\mbx}(t)$. Substituting in the Gaussian score and grouping terms yields the drift
\begin{equation*}
\begin{aligned}
   &\left\{ \left(\frac{d}{dt}\mbS(t)^{1/2}\right)\mbS(t)^{-1/2} - \frac{1}{2}\mbSigma(\mbx(t), t)\mbS(t)^{-1}\right\}\mbx(t) +\\
    &\frac{d}{dt}\mbm(t) - \left\{ \left(\frac{d}{dt}\mbS(t)^{1/2}\right)\mbS(t)^{-1/2} - \frac{1}{2}\mbSigma(\mbx(t), t)\mbS(t)^{-1}\right\} \mbm(t) + \frac{1}{2} \sum_{j} \partial_{\mbx_j} \mbSigma_{\cdot, j}(\mbx(t), t).
\end{aligned}
\end{equation*}
Once we identify $\mbK(t) = \left(\frac{d}{dt}\mbS(t)^{1/2}\right)\mbS(t)^{-1/2}$, the resulting drift is of the form given in \cref{app:eq:gaussian-compatible-drift}.

Although the construction in \citet{bartosh2025sde} allows the marginals $q(\mbx, t)$ to be non-Gaussian (e.g., by parameterizing $F$ as a normalizing flow), the paper itself evaluates the method only in the Gaussian marginal setting. We therefore restrict our discussion to that case.

\paragraph{An extension of SDE Matching.}
The framework of SDE Matching \cite{bartosh2025sde} can in principle be extended to span every solution to the Lyapunov equation \eqref{app:eq:gaussian-lyapunov}.
This is not discussed or implemented by the authors. 
As described in the previous paragraph, their parameterization $\mbS(t) = \mbsigma(t)\mbsigma(t)^\top$ induces the linear coefficient $\mbK(t) = (\frac{d}{dt}\mbsigma(t)) \mbsigma(t)^{-1}$. 
Fixing $\mbsigma$ to a specific square root selects a single path law for each set of one-time marginals. 
By \Cref{app:prop:gaussian-compatible-drift}, the full set of solutions to the Lyapunov equation is parameterized by an additional skew-symmetric matrix $\mbOmega(t)$ via
\begin{equation*}
    \mbK(t) = \left(\frac{d}{dt}\mbsigma(t) \right) \mbsigma(t)^{-1} + \mbsigma(t) \mbOmega(t) \mbsigma(t)^{-1}, \quad 0 \leq t \leq T.
\end{equation*}
Augmenting SDE Matching with a learned $\mbOmega(t)$ would in principle allow the posterior to span at each time the full $\binom{K}{2}$-dimensional affine space of compatible linear drift coefficients.

This extension has two limitations. 
First, it requires introducing and optimizing over $\mbOmega(t)$, sacrificing the closed-form structure of SDE Matching's original parameterization. 
Helmholtz-SDE, by contrast, \emph{computes the optimal correction in closed form} via the Helmholtz decomposition of the residual vector field (\Cref{sec:method}). 
Second, for a state-independent diffusion coefficient, the resulting path law remains a Gaussian process.
\Cref{app:prop:gaussian-compatible-drift} characterizes only diffusions with affine drift that produce the prescribed Gaussian marginals; non-linear drifts can also yield Gaussian one-time marginals while inducing non-Gaussian multi-time joint distributions. 
Helmholtz-SDE can produce posterior approximations that lie outside the Gaussian-process family by using higher-order polynomial approximations to the optimal drift (\Cref{app:subsec:resid-approx}).
\section{Helmholtz-SDE}\label{app:sec:helm-sde}

In this section, we provide further explanation of our algorithm, Helmholtz-SDE.

First, in \Cref{app:subsec:helmholtz-compute} we describe how the $q$-weighted Helmholtz decomposition can be obtained for arbitrary smooth, strictly positive marginal densities $q$ by inverting a Stein operator.
For the remainder of \Cref{app:sec:helm-sde}, we then specialize to the setting of Gaussian one-time marginals.
In \Cref{app:subsec:helmholtz-gaussian}, we provide sufficient conditions under which the Helmholtz decomposition exists uniquely.
In \Cref{app:subsec:resid-approx}, we translate this theoretical result into a practical formula using an order-$\ell$ polynomial approximation to the optimal divergence-free correction.
In \Cref{app:subsec:verify-assumptions}, we verify that the resulting variational drift, obtained by adding the polynomial correction to the reference drift, produces an SDE satisfying our assumptions in \Cref{app:subsec:assumptions}.
In \Cref{app:subsec:optimality}, we discuss the optimality of Helmholtz-SDE in the conjugate and non-conjugate settings.
In \Cref{app:subsec:alg-summary} we provide pseudocode of our algorithm.

\subsection{Computing the Helmholtz decomposition}\label{app:subsec:helmholtz-compute}

We first discuss how the Helmholtz decomposition can be computed for an arbitrary smooth, strictly positive density $q$ as the inverse of a Stein operator. The resulting probabilistic expression can be approximated via Monte Carlo. For ease of notation, we consider a state-independent diffusion coefficient, although the same construction holds in the state-dependent setting.

\paragraph{The Helmholtz decomposition and diffusion Stein operator.}
As described in \Cref{subsec:helmholtz-decomp}, the optimal correction to the reference drift $\mbf_q(\mbx, t)$ is obtained by computing the $q$-weighted Helmholtz decomposition of the residual vector field $\mbr(\mbx, t) = \mbf_p(\mbx, t) - \mbf_q(\mbx, t)$ and taking the $q$-divergence-free vector field.
In particular, we want to find a decomposition of $\mbr$ as
\begin{equation}\label{app:eqn:helm-decomp}
    \mbr(\mbx, t) = \mbSigma(t) \nabla_{\mbx} \psi(\mbx, t) + \mbh(\mbx, t), \quad (\mbx, t) \in \bbR^K \times [0, T],
\end{equation}
where $\nabla_{\mbx} \psi(\mbx, t)$ is a gradient field and $\mbh(\mbx, t)$ is a $q$-divergence-free vector field.

Suppose $q(\mbx, t)$ is a smooth, strictly positive density.
To obtain such a decomposition, observe that if \cref{app:eqn:helm-decomp} holds, then multiplying both sides by $q(\mbx, t)$ and computing the divergence yields
\begin{equation}
     \nabla_{\mbx} \cdot [q(\mbx, t) \mbSigma(t) \nabla_{\mbx} \psi(\mbx, t)]
     =
     \nabla_{\mbx} \cdot [q(\mbx, t) \mbr(\mbx, t)],
     \quad
     (\mbx, t) \in \bbR^K \times [0, T],
\end{equation}
since $\mbh$ is $q$-divergence-free. 
Dividing both sides by $q$ and expanding the left-hand side yields
\begin{equation}\label{app:eqn:stein-operator-eqn}
    \mathrm{tr}(\mbSigma(t) \nabla_{\mbx}^2 \psi(\mbx, t))
    +
    \nabla_{\mbx} \log q(\mbx, t) \cdot
    \mbSigma(t) \nabla_{\mbx} \psi(\mbx, t)
    =
    \frac{1}{q(\mbx, t)}
    \nabla_{\mbx} \cdot [q(\mbx, t) \mbr(\mbx, t)] .
\end{equation}

From \cref{app:eqn:stein-operator-eqn}, we make two observations. First, the left-hand side can be written as $[\cL_t^{\mbSigma}\nabla_{\mbx}\psi(\cdot,t)](\mbx)$, where $\cL_t^{\mbSigma}$ is the differential operator
\begin{equation}\label{app:eqn:diffusion-stein}
    [\cL_t^{\mbSigma}\mbf](\mbx) = \mathrm{tr}(\mbSigma(t) D_{\mbx} \mbf (\mbx, t))
    +
    \mbSigma(t) \nabla_{\mbx} \log q(\mbx, t) \cdot
    \mbf(\mbx, t)
\end{equation}
defined for differentiable vector fields $\mbf: \bbR^K \to \bbR^K$, where $D_{\mbx} \mbf$ represents the Jacobian of $\mbf$.
This is the \emph{diffusion Stein operator} corresponding to the density $q(\mbx,t)$ and constant covariance coefficient $\mbSigma(t)$ \citep{barbour1990stein,gorham2019measuring}. 
When $\mbSigma(t)=\bbI_K$, it reduces to the \emph{Langevin Stein} operator
\begin{equation*}
    [\cL_t \mbf](\mbx) = \nabla_{\mbx} \cdot \mbf(\mbx)  + \nabla_{\mbx} \log q(\mbx, t) \cdot \mbf(\mbx).
\end{equation*}
We refer the reader to \citet{gorham2019measuring} for background on diffusion Stein operators.

Second, whenever $\mbr(\mbx, t)$ and
$\cL_t[\mbr(\cdot, t)](\mbx)$ are integrable, the right-hand side of \cref{app:eqn:stein-operator-eqn} has expectation
\begin{equation*}
    \bbE_{\bbQ}\left[
    \frac{1}{q(\mbx, t)}
    \nabla_{\mbx} \cdot [q(\mbx, t) \mbr(\mbx, t)]
    \right]
    =
    \int
    \nabla_{\mbx} \cdot [q(\mbx, t) \mbr(\mbx, t)]
    d \mbx = 0
\end{equation*}
The final equality holds by the divergence theorem \citep[][Theorem 3.1]{liu2026probabilistic}. 

Overall, we have reduced computing the Helmholtz decomposition to solving
\begin{equation}\label{app:eqn:helmholtz-stein-eqn}
\begin{aligned}
    &[\cL_t^{\mbSigma}\nabla_{\mbx} \psi(\cdot, t)](\mbx)
    =
    \eta(\mbx, t),
    \quad
    (\mbx, t) \in \bbR^K \times [0, T],\\
    &\eta(\mbx, t)
    =
    \frac{1}{q(\mbx, t)}
    \nabla_{\mbx} \cdot [q(\mbx, t) \mbr(\mbx, t)]
\end{aligned}
\end{equation}
where $\eta(\mbx, t)$ has mean zero under $q(\mbx,t)$.
Intuitively, in order to solve \cref{app:eqn:helmholtz-stein-eqn}, we need to find an inverse of the diffusion Stein operator $\cL_t^{\mbSigma}$. 
Doing so has been the subject of prior literature, namely \citet{barbour1990stein,pardoux2005poisson,gorham2019measuring}.

\paragraph{Inverting the diffusion Stein operator.}

Fix $t \in [0,T]$. Consider the time-homogeneous diffusion in an auxiliary time variable $s$ with diffusion coefficient $\mbSigma(t)^{1/2}$:
\begin{equation}\label{app:eqn:langevin-diffusion}
    d\mbz^{(t)}(s)
    =
    \frac{1}{2} \mbSigma(t)\nabla_{\mbx}\log q(\mbz^{(t)}(s), t)ds
    +
    \mbSigma(t)^{1/2}d\mbw(s). 
\end{equation}
When $\mbSigma(t) = \bbI_K$, \cref{app:eqn:langevin-diffusion} is recognizable as the overdamped Langevin diffusion.
In particular, \cref{app:eqn:langevin-diffusion} has invariant density $q(\mbx, t)$.
Its \emph{infinitesimal generator} is
\begin{equation}
\begin{aligned}
    [\cA_t u](\mbz) &= \frac{1}{2q(\mbz, t)} \nabla_{\mbz} \cdot \left[q(\mbz, t)\mbSigma(t)\nabla_{\mbz}u(\mbz)\right]\\
    &= \frac{1}{2} \mathrm{tr}(\mbSigma(t)\nabla_{\mbz}^2 u(\mbz)) + \frac{1}{2} \mbSigma(t) \nabla_{\mbz}\log q(\mbz, t) \cdot \nabla_{\mbz} u(\mbz)
\end{aligned}
\end{equation}
for $u: \bbR^K \to \bbR$ twice continuously differentiable and satisfying integrability conditions \cite{karatzas2014brownian}.
This implies, by \cref{app:eqn:diffusion-stein}, $\cA_t u = \cL_t^{\mbSigma}(\frac{1}{2}\nabla_{\mbx}u)$. 
Hence, solving \cref{app:eqn:helmholtz-stein-eqn} is equivalent to solving
\begin{equation}\label{app:eqn:poisson}
    [\cA_t u(\cdot, t)](\mbx) = \eta(\mbx, t), \quad  (\mbx, t) \in \bbR^K \times [0, T]
\end{equation}
and then setting $\nabla_{\mbx}\psi(\mbx,t) = \frac{1}{2} \nabla_{\mbx}u(\mbx, t)$.
\Cref{app:eqn:poisson} is known as the \emph{Poisson equation} \cite{pardoux2005poisson,gilbarg1998elliptic}.

Assume that, for each fixed $t$, $\eta(\mbx, t)$ and the coefficients of the above diffusion satisfy the regularity conditions of \citet[][Theorem 5]{gorham2019measuring}. Then \citet[][Theorem 5]{gorham2019measuring} implies that 
\begin{equation}\label{app:eqn:poisson-soln}
    u(\mbx, t) = -\int_0^\infty \bbE[\eta(\mbz^{(t)}(s), t) | \mbz^{(t)}(0) = \mbx]ds
\end{equation}
solves the Poisson equation. The integrand in \cref{app:eqn:poisson-soln} is computed by starting the diffusion \cref{app:eqn:langevin-diffusion} from $\mbx$ and then running it forward until time $s$.

Therefore,
\begin{equation}\label{app:eqn:helmholtz-soln}
    \nabla_{\mbx}\psi(\mbx,t) =  -\frac{1}{2}\nabla_{\mbx} \int_0^\infty \bbE[\eta(\mbz^{(t)}(s), t) | \mbz^{(t)}(0) = \mbx]ds
\end{equation}
solves \cref{app:eqn:helmholtz-stein-eqn}. Consequently, defining $\mbh(\mbx,t) = \mbr(\mbx,t) - \mbSigma(t)\nabla_{\mbx}\psi(\mbx,t)$
gives $\nabla_{\mbx} \cdot \left[ q(\mbx,t)\mbh(\mbx,t) \right] = 0$.
Thus $\mbh(\cdot,t)$ is the $q$-divergence-free component of the residual vector field.

\subsection{The Helmholtz decomposition for Gaussian marginals}\label{app:subsec:helmholtz-gaussian}

For the remainder of \Cref{app:sec:helm-sde}, we specialize to the setting of Gaussian one-time marginals and a state-independent diffusion coefficient.
In this subsection, we establish the existence,  uniqueness, and regularity of the weighted Helmholtz decomposition in this setting.

Since we are considering only the existence and uniqueness of the weighted Helmholtz decomposition, and not the implied variational SDE, we do not require the assumptions from \Cref{app:subsec:assumptions}.

First, we show that under minimal assumptions, the $q$-divergence-free vector field $\mbh$ that minimizes \cref{eqn:optimize} exists. It is the orthogonal projection onto the subspace of $q$-divergence-free vector fields.
\begin{theorem}[Existence and uniqueness of weak Helmholtz decomposition]\label{app:thm:hilbert-projection}
     Suppose that $\mbr(\mbx, t)$ has polynomial growth and that $\mbSigma(t)$ is positive definite. 
     Let $\cL^2(q(\mbx, t), \bbR^K, \mbSigma(t))^K$ be the Hilbert space of vector fields $\mbf: \bbR^K \to \bbR^K$ satisfying $\bbE_{q(\mbx, t)}\|\mbf(\mbx)\|^2 < \infty$ equipped with the inner product 
     \begin{equation}\label{app:eqn:l2-inner-prod}
         \langle \mbf, \mbg \rangle_{q, t} = \bbE_{q(\mbx, t)} [\mbf(\mbx)^\top \mbSigma(t)^{-1} \mbg(\mbx)], \qquad \| \mbf \|_{q, t} = \langle \mbf, \mbf \rangle_{q, t}^{1/2}
    \end{equation}

     \smallskip
     Then $\mbr(\mbx, t) \in \cL^2(q(\mbx, t), \bbR^K, \mbSigma(t))^K$ admits the decomposition 
     \begin{equation}\label{app:eqn:helmholtz-weak}
         \mbr(\mbx, t) = \mbg(\mbx, t) + \mbh(\mbx, t)
     \end{equation}
     where $\mbh, \mbg \in  \cL^2(q(\mbx, t), \bbR^K, \mbSigma(t))^K$, $\mbh$ is $q$-divergence-free in the weak sense \cref{app:eqn:div-free-weak}, and $\mbg$ is orthogonal to all $q$-divergence-free vector fields.
     Moreover, this decomposition is unique in $\cL^2(q(\mbx, t), \bbR^K, \mbSigma(t))^K$. 
\end{theorem}
\begin{proof}
    Consider the set $\cV_t$ of all vector fields that are divergence-free in the weak sense \cref{app:eqn:div-free-weak}:
    \begin{equation}\label{app:eqn:weak-subspace}
    \cV_t = \left\{ \mbv \in  \cL^2(q(\mbx, t), \bbR^K, \mbSigma(t))^K \ \bigg| \  \int \mbv(\mbx, t)^\top \nabla_{\mbx} \varphi(\mbx) q(d\mbx, t) = 0, \quad \forall \varphi \in \cC \right\},
    \end{equation}
    where $\cC$ is the collection of test functions that are infinitely differentiable and have compact support.

    Since $\mbr(\mbx, t)$ has polynomial growth, it belongs to $\cL^2(q(\mbx, t), \bbR^K, \mbSigma(t))^K$.
    By the Hilbert projection theorem, the result follows once we establish that $\cV_t$ is a closed subspace of $\cL^2(q(\mbx, t), \bbR^K, \mbSigma(t))^K$. 
   
    The fact that $\cV_t$ is a subspace is immediate from the definition. To show that it is closed, fix $\varphi \in \cC$ and define
    \begin{equation*}
    \cL_\varphi(\mbv) =  \int \mbv(\mbx)^\top\nabla_{\mbx}\varphi(\mbx) q(d\mbx,t).
    \end{equation*}
Since $\cL_\varphi(\mbv) = \left\langle \mbv, \mbSigma(t)\nabla_{\mbx}\varphi \right\rangle_{q, t}$, the Cauchy--Schwarz inequality gives
\begin{equation*}
|\cL_\varphi(\mbv)| \leq \|\mbv \|_{q, t} \left\| \mbSigma(t)\nabla_{\mbx}\varphi \right\|_{q, t}.
\end{equation*}
Consequently, $\cL_{\varphi}$ is a continuous linear functional.
Therefore, $\cV_t = \bigcap_{\varphi\in \cC}
\mathrm{ker} \cL_{\varphi}$ is closed.
\end{proof}

The preceding theorem tells us that, among the vector fields that are $q$-divergence-free in the weak sense \cref{app:eqn:div-free-weak}, $\mbh(\mbx,t)$ uniquely minimizes \cref{eqn:optimize}.
We call \cref{app:eqn:helmholtz-weak} the \emph{weak Helmholtz decomposition}.
\Cref{app:thm:hilbert-projection} does \emph{not}, however, provide a Helmholtz decomposition as in \cref{eqn:helmholtz-decomp}.
Specifically, it neither establishes that $\mbh$ satisfies the $q$-divergence-free condition pointwise, nor identifies $\mbg$ with a gradient field of the form $\mbSigma(t)\nabla_{\mbx}\psi$, nor verifies the regularity of $\mbg$, $\mbh$, or $\psi$.

The following result supplies this representation and regularity for the Gaussian marginals used in Helmholtz-SDE.
 
\begin{assumption}\label{app:assum:resid}
    Assume that the following hold:
    \begin{enumerate}[label=(\roman*)]
        \item\label{item:ass:cont} $(\mbm(t), \mbS(t))_{0 \leq t \leq T}$ is continuously differentiable in $t$, and $\mbS(t)$ is positive definite.
        \item\label{item:ass:diff} $\mbr$ is twice differentiable in $\mbx$, and its derivatives are continuous in $(\mbx, t)$.
        \item\label{item:ass:poly} $\mbr$, $D_{\mbx} \mbr$, and $D_{\mbx}^{(2)}\mbr$ have polynomial growth in $\mbx$, uniformly in $t$.
        \item\label{item:ass:cont-sigma} $\mbSigma(t)$ is continuous and positive definite.
    \end{enumerate}  
    $D_{\mbx} \mbr \in \bbR^{K \times K}$ and $D_{\mbx}^{(2)} \mbr \in \bbR^{K \times K \times K}$ represent the derivative tensors of the residual vector field $\mbr: \bbR^K \to \bbR^K$. 
\end{assumption}

\begin{theorem}[Regularity of the Helmholtz decomposition]\label{app:thm:helm-gaussian}
    Suppose \Cref{app:assum:resid} holds.
    For every $t\in[0,T]$, the weak Helmholtz decomposition of $\mbr(\mbx, t)$ in \cref{app:eqn:helmholtz-weak} has the pointwise representation
    \begin{equation*}
    \mbr(\mbx,t)=\mbSigma(t)\nabla_{\mbx}\psi(\mbx,t)+\mbh(\mbx,t),
    \end{equation*}
    where $\mbh$ is $q$-divergence-free in the strong sense
    \begin{equation*}
    \nabla_{\mbx}\cdot[q(\mbx,t)\mbh(\mbx,t)]=0, \quad (\mbx, t) \in \bbR^K \times [0, T].
\end{equation*}

    $\psi$ is twice differentiable in $\mbx$, $\mbh$ is once differentiable in $\mbx$, and their derivatives are continuous in $(\mbx, t)$.
    Within the stated regularity class, $\mbh$ is unique pointwise, and $\psi(\cdot,t)$ is unique up to an additive constant, which may depend on $t$.
\end{theorem}
\begin{proof}
Fix $t \in [0, T]$.
Since $\mbS(t)$ is positive definite, $q(\mbx, t)$ is a smooth, strictly positive density on $\bbR^K$. 
Recall from \Cref{app:subsec:helmholtz-compute} that solving the Stein operator equation \cref{app:eqn:helmholtz-stein-eqn} for $\nabla_{\mbx}\psi$ reduces, via the identity $\cA_t u = \cL_t^{\mbSigma}(\tfrac{1}{2}\nabla_{\mbx} u)$, to solving the Poisson equation \cref{app:eqn:poisson} for $u$:
\begin{equation}\label{app:eqn:weighted-poisson-theorem}
    [\cA_t u(\cdot,t)](\mbx) \;=\; \frac{1}{2 q(\mbx, t)} \nabla_{\mbx} \cdot \left[q(\mbx, t)\mbSigma(t)\nabla_{\mbx} u(\mbx,t)\right] \;=\; \eta(\mbx, t),
\end{equation}
where $\cA_t$ is the infinitesimal generator of the Langevin diffusion \cref{app:eqn:langevin-diffusion} and $\eta(\mbx, t)$ is the source term in \cref{app:eqn:helmholtz-stein-eqn}.
The gradient component of the Helmholtz decomposition is then recovered as $\nabla_{\mbx}\psi(\mbx, t) = \tfrac{1}{2}\nabla_{\mbx} u(\mbx, t)$.

Our goal is to show that the candidate $u(\mbx, t)$ in \cref{app:eqn:poisson-soln} satisfies \cref{app:eqn:weighted-poisson-theorem} pointwise, and that $u$ is twice continuously differentiable in $\mbx$ with derivatives continuous in $(\mbx, t)$.
Once \cref{app:eqn:weighted-poisson-theorem} is established, defining
\begin{equation}\label{app:eqn:h-definition-theorem}
    \mbh(\mbx,t) = \mbr(\mbx,t) - \mbSigma(t)\nabla_{\mbx}\psi(\mbx,t)
\end{equation}
gives $\nabla_{\mbx} \cdot \left[q(\mbx, t)\mbh(\mbx,t)\right] = 0$, yielding the desired Helmholtz decomposition.

Our proof proceeds in four main steps:
\begin{enumerate}[label=(\roman*)]
    \item\label{item:whiten} Reexpress the candidate $u$ in whitened and rotated coordinates, where the associated Langevin diffusion becomes a product of one-dimensional Ornstein--Uhlenbeck (OU) processes.
    \item\label{item:reg} Prove that the candidate $u$ is well-defined and satisfies the desired regularity conditions.
    \item\label{item:poisson} Verify that $u$ satisfies the Poisson equation \cref{app:eqn:weighted-poisson-theorem} pointwise.
    \item\label{item:unique} Identify the constructed vector fields with the unique weak Helmholtz decomposition from \Cref{app:thm:hilbert-projection}.
\end{enumerate}

\medskip
\textbf{Step \ref{item:whiten}}:
The candidate $u$ proposed in \Cref{app:subsec:helmholtz-compute} (\cref{app:eqn:poisson-soln}) is expressed as an integral of expectations along the Langevin diffusion \cref{app:eqn:langevin-diffusion} with invariant law $q(\mbx, t)$. We use the same whitening and rotation as in \Cref{app:subsec:resid-approx}. Define
\begin{equation*}
\begin{aligned}
    \mbC(t) &= \mbS(t)^{-1/2}\mbSigma(t)\mbS(t)^{-1/2} = \mbQ(t)\mbLambda(t)\mbQ(t)^\top,\\
    \mbLambda(t) &= \diag(\lambda_1(t),\ldots,\lambda_K(t)), \qquad \mbQ(t)^\top\mbQ(t)=\bbI_K,
\end{aligned}
\end{equation*}
and introduce the coordinates
\begin{equation*}
    \tilde\mbx = \mbQ(t)^\top\mbS(t)^{-1/2}(\mbx-\mbm(t)), \qquad \chi(\tilde\mbx,t) = u(\mbm(t)+\mbS(t)^{1/2}\mbQ(t)\tilde\mbx,t).
\end{equation*}
Transform the residual in the same way:
\begin{equation*}
    \widetilde{\mbrho}(\tilde\mbx,t) = \mbQ(t)^\top\mbS(t)^{-1/2} \mbr(\mbm(t)+\mbS(t)^{1/2}\mbQ(t)\tilde\mbx,t).
\end{equation*}
Then $q(\tilde\mbx,t)=\cN(\mb0,\bbI_K)$. Since $\mbS(t)$ and $\mbSigma(t)$ are positive definite, every $\lambda_i(t)$ is strictly positive.

After an orthogonal transformation of the driving Brownian motion, the Langevin diffusion \cref{app:eqn:langevin-diffusion} becomes
\begin{equation}\label{app:eqn:ou-process-theorem}
    d\tilde\mbz^{(t)}(s) = -\tfrac{1}{2}\mbLambda(t)\tilde\mbz^{(t)}(s)ds + \mbLambda(t)^{1/2}d\mbw(s), \qquad \tilde\mbz^{(t)}(0)=\tilde\mbx.
\end{equation}
Thus its coordinates are independent one-dimensional OU processes with invariant law $\cN(0,1)$. Equivalently, each coordinate is a standard OU process after rescaling the artificial time $s$ by $\lambda_i(t)$. The joint invariant law is $\cN(\mb0,\bbI_K)$, and the infinitesimal generator is
\begin{equation*}
    (\tilde{\cA}_t f)(\tilde\mbx) = \tfrac{1}{2}\mathrm{tr}\left( \mbLambda(t)\nabla_{\tilde\mbx}^2 f(\tilde\mbx) \right) - \tfrac{1}{2}\tilde\mbx^\top \mbLambda(t)\nabla_{\tilde\mbx}f(\tilde\mbx)
\end{equation*}
for $f:\bbR^K\to\bbR$ twice continuously differentiable with derivatives having polynomial growth. Define the transformed source
\begin{equation}\label{app:eqn:g-definition}
    g(\tilde\mbx,t) = \nabla_{\tilde\mbx}\cdot\widetilde{\mbrho}(\tilde\mbx,t) - \tilde\mbx^\top\widetilde{\mbrho}(\tilde\mbx,t).
\end{equation}
A direct computation shows that
\begin{equation*}
    g(\tilde\mbx,t) = \eta(\mbm(t)+\mbS(t)^{1/2}\mbQ(t)\tilde\mbx,t),
\end{equation*}
where $\eta$ is the source term in \cref{app:eqn:helmholtz-stein-eqn}. Overall, in the $\tilde\mbx$ coordinates, the Poisson equation becomes
\begin{equation}\label{app:eqn:whitened-poisson-theorem}
    (\tilde{\cA}_t \chi(\cdot, t))(\tilde\mbx) = g(\tilde\mbx, t).
\end{equation}

By \Cref{app:assum:resid}, $\mbr$ is differentiable in $\mbx$, and both $\mbr$ and $D_{\mbx}\mbr$ have polynomial growth; the same therefore holds of $\widetilde{\mbrho}(\cdot,t)$ and $D_{\tilde\mbx}\widetilde{\mbrho}(\cdot,t)$. Stein's identity gives
\begin{equation*}
    \E_{\tilde\mbx \sim \cN(\mb0,\bbI_K)}\left[\nabla_{\tilde\mbx}\cdot \widetilde{\mbrho}(\tilde\mbx,t)\right] = \E_{\tilde\mbx \sim \cN(\mb0,\bbI_K)}\left[\tilde\mbx^\top \widetilde{\mbrho}(\tilde\mbx,t)\right],
\end{equation*}
so $g(\cdot, t)$ has mean zero under the OU invariant law.

In these whitened coordinates, the candidate from \cref{app:eqn:poisson-soln} takes the form
\begin{equation}\label{app:eqn:chi-representation-theorem}
    \chi(\tilde\mbx,t) = -\int_0^\infty (P_s^t g)(\tilde\mbx)\,ds, \qquad (P_s^t g)(\tilde\mbx) = \bbE\left[g\left(\tilde\mbz^{(t)}(s),t\right)\,\big|\,\tilde\mbz^{(t)}(0)=\tilde\mbx\right],
\end{equation}
where $(P_s^t)_{s \geq 0}$ is the OU semigroup associated with \cref{app:eqn:ou-process-theorem}.
Equivalently,
\begin{equation*}
    u(\mbx,t) = \chi\left( \mbQ(t)^\top\mbS(t)^{-1/2}(\mbx-\mbm(t)),t \right).
\end{equation*}

\medskip
\textbf{Step \ref{item:reg}}:
We first justify that \cref{app:eqn:chi-representation-theorem} is well-defined.
For all times $0 \le s < \infty$, the marginal distribution of the OU process
\eqref{app:eqn:ou-process-theorem} satisfies
\begin{equation*}
    \tilde\mbz^{(t)}(s) \overset{d}{=} \mbA_t(s)\tilde\mbx + \mbB_t(s)\mbxi, \qquad \mbxi \sim \cN(\mb0,\bbI_K),
\end{equation*}
where
\begin{equation*}
    \mbA_t(s) = e^{-s\mbLambda(t)/2}, \qquad \mbB_t(s) = \left(\bbI_K - e^{-s\mbLambda(t)}\right)^{1/2}.
\end{equation*}
Thus
\begin{equation*}
    (P_s^t g)(\tilde\mbx) = \bbE_{\mbxi \sim \cN(\mb0,\bbI_K)} \left[g\left(\mbA_t(s)\tilde\mbx+\mbB_t(s)\mbxi, t\right) \right].
\end{equation*}
Using $\E_{\mbxi \sim \cN(\mb0,\bbI_K)}[g(\mbxi,t)] = 0$, we may write
\begin{equation*}
    (P_s^t g)(\tilde\mbx) = \bbE_{\mbxi \sim \cN(\mb0,\bbI_K)} \left[ g\left(\mbA_t(s)\tilde\mbx+\mbB_t(s)\mbxi,t\right) - g(\mbxi,t) \right].
\end{equation*}
Hence,
\begin{equation*}
    \left|(P_s^t g)(\tilde\mbx)\right| \leq \bbE_{\mbxi \sim \cN(\mb0,\bbI_K)} \left| g\left(\mbA_t(s)\tilde\mbx+\mbB_t(s)\mbxi,t\right)-g(\mbxi,t) \right|.
\end{equation*}
Now decompose
\begin{equation*}
\begin{aligned}
    g\left(\mbA_t(s)\tilde\mbx+\mbB_t(s)\mbxi,t\right)-g(\mbxi,t) &= \Big( g\left(\mbA_t(s)\tilde\mbx+\mbB_t(s)\mbxi,t\right) - g\left(\mbB_t(s)\mbxi,t\right) \Big) \\
    &\quad + \Big( g\left(\mbB_t(s)\mbxi,t\right)-g(\mbxi,t) \Big).
\end{aligned}
\end{equation*}
Applying the mean value theorem to each term and using that $\nabla_{\tilde\mbx} g$ has at most polynomial growth,
there exist constants $C, m > 0$ such that
\begin{equation*}
    \|\nabla_{\tilde\mbx} g(\tilde\mbx,t)\|_2 \le C(1+\|\tilde\mbx\|_2^m).
\end{equation*}
Since $\mbxi$ is Gaussian, all of its moments are finite. Moreover,
\begin{equation*}
    \|\mbA_t(s)\|_{\mathrm{op}} \leq C_t e^{-c_t s}, \qquad \|\mbB_t(s)-\bbI_K\|_{\mathrm{op}} \leq C_t e^{-c_t s},
\end{equation*}
for some constants $C_t, c_t > 0$. It follows that
\begin{equation*}
    \left| \bbE\left[g\left(\tilde\mbz^{(t)}(s),t\right)\,\big|\,\tilde\mbz^{(t)}(0)=\tilde\mbx\right] \right| \leq C_t'(1+\|\tilde\mbx\|_2^{m+1})e^{-c_t s},
\end{equation*}
for some constant $C_t' > 0$. In particular, \cref{app:eqn:chi-representation-theorem} is finite for every $(\tilde\mbx, t)$.

We next show that, for every fixed $t$, $\chi(\cdot,t)$ is twice continuously differentiable in $\tilde\mbx$. We then establish joint continuity in $(\mbx,t)$ after transforming these derivatives back to the original coordinates. The estimates below also justify passage of the spatial derivatives, and hence of $\tilde{\cA}_t$, under the integral sign in step \ref{item:poisson}.

Using the notation introduced above, we can rewrite
\cref{app:eqn:chi-representation-theorem} as
\begin{equation}\label{app:eqn:chi-gaussian-representation-theorem}
    \chi(\tilde\mbx,t) = -\int_0^\infty \E_{\mbxi \sim \cN(\mb0,\bbI_K)} \left[ g\left(\mbA_t(s)\tilde\mbx+\mbB_t(s)\mbxi,t\right) \right] ds.
\end{equation}

We first study the gradient of $\chi$.
To do so, we must justify the exchange of the gradient and integral in \cref{app:eqn:chi-gaussian-representation-theorem}. Applying Stein's identity to the integrand of \cref{app:eqn:chi-gaussian-representation-theorem} yields 
\begin{equation*}
    F(\tilde\mbx,t,s) = \nabla_{\tilde\mbx}(P_s^t g)(\tilde\mbx) = \mbA_t(s)\mbB_t(s)^{-1} \bbE_{\mbxi \sim \cN(0,\bbI_K)} \left[ \mbxi \, g\left(\mbA_t(s)\tilde\mbx+\mbB_t(s)\mbxi,t\right) \right].
\end{equation*}
The polynomial growth of $g$ implies
\begin{equation*}
\begin{aligned}
    \|F(\tilde\mbx,t,s)\|_2 &\leq \|\mbA_t(s)\mbB_t(s)^{-1}\|_{\mathrm{op}} \bbE_{\mbxi \sim \cN(0,\bbI_K)} \left[ \|\mbxi\|_2\,|g\left(\mbA_t(s)\tilde\mbx+\mbB_t(s)\mbxi,t\right)| \right] \\
    &\leq C\|\mbA_t(s)\mbB_t(s)^{-1}\|_{\mathrm{op}} \left(1+\|\mbA_t(s)\tilde\mbx\|_2^m\right),
\end{aligned}
\end{equation*}
where we used that $\mbxi$ is Gaussian and $\mbB_t(s)$ is bounded.
Now, as $s\downarrow 0$, $\|\mbA_t(s)\mbB_t(s)^{-1}\|_{\mathrm{op}} = O(s^{-1/2})$,
while as $s\to\infty$,
\begin{equation*}
    \|\mbA_t(s)\|_{\mathrm{op}} \le C_t e^{-c_t s}, \quad \|\mbB_t(s)^{-1}\|_{\mathrm{op}} \le C_t.
\end{equation*}
Therefore
\begin{equation*}
    \|F(\tilde\mbx,t,s)\|_2 \leq \begin{cases} C(1+\|\tilde\mbx\|_2^m)s^{-1/2}, & 0 < s \leq 1,\\ C_t(1+\|\tilde\mbx\|_2^m)e^{-c_t s}, & s\geq 1, \end{cases}
\end{equation*}
and the right-hand side is integrable on $(0,\infty)$. Moreover, for every compact set $K \subset \bbR^K \times [0,T]$, this integrable bound can be chosen uniformly for $(\tilde\mbx,t)\in K$. By dominated convergence we deduce 
\begin{equation*}
    \nabla_{\tilde\mbx} \chi(\tilde\mbx,t) = -\int_0^\infty F(\tilde\mbx, t, s) ds.
\end{equation*}

We next study the Hessian of $\chi$.
Since $g$ is continuously differentiable in $\tilde\mbx$, for each fixed $s>0$ the map
$(\tilde\mbx,t)\mapsto F(\tilde\mbx,t,s)$ is continuously differentiable in $\tilde\mbx$, and
\begin{equation*}
    \nabla_{\tilde\mbx} F(\tilde\mbx,t,s) = \nabla_{\tilde\mbx}^2(P_s^t g)(\tilde\mbx) = \mbA_t(s)\mbB_t(s)^{-1} \bbE_{\mbxi \sim \cN(0,\bbI_K)} \left[ \mbxi \, \left(\nabla_{\tilde\mbx} g\left(\mbA_t(s)\tilde\mbx+\mbB_t(s)\mbxi,t\right)\right)^\top \right] \mbA_t(s).
\end{equation*}
Using that $\nabla_{\tilde\mbx}g$ has at most polynomial growth, we obtain
\begin{equation*}
\begin{aligned}
    \|\nabla_{\tilde\mbx} F(\tilde\mbx,t,s)\|_2&\leq\|\mbA_t(s)\mbB_t(s)^{-1}\|_{\mathrm{op}} \,\|\mbA_t(s)\|_{\mathrm{op}} \, \bbE_{\mbxi \sim \cN(0,\bbI_K)} \left[\|\mbxi\|_2\,\|\nabla_{\tilde\mbx} g\left(\mbA_t(s)\tilde\mbx+\mbB_t(s)\mbxi,t\right)\|_2 \right] \\
    &\leq C\|\mbA_t(s)\mbB_t(s)^{-1}\|_{\mathrm{op}}\,\|\mbA_t(s)\|_{\mathrm{op}} \left(1+\|\mbA_t(s)\tilde\mbx\|_2^m\right).
\end{aligned}
\end{equation*}
Hence,
\begin{equation*}
    \|\nabla_{\tilde\mbx} F(\tilde\mbx,t,s)\|_2 \leq \begin{cases} C(1+\|\tilde\mbx\|_2^m)s^{-1/2}, & 0 < s \leq 1,\\ C_t(1+\|\tilde\mbx\|_2^m)e^{-c_t s}, & s\ge 1, \end{cases}
\end{equation*}
which is again integrable on $(0,\infty)$. Again, for every compact set $K \subset \bbR^K \times [0,T]$, this integrable bound may be chosen uniformly for $(\tilde\mbx,t)\in K$, and dominated convergence yields
\begin{equation*}
    \nabla_{\tilde\mbx}^2 \chi(\tilde\mbx,t) = -\int_0^\infty \nabla_{\tilde\mbx}F(\tilde\mbx,t,s)\,ds.
\end{equation*}
Hence, $\chi(\cdot,t)$ is twice continuously differentiable in $\tilde\mbx$ for every $t$.

It remains to establish continuity after transforming back to the original coordinates. A continuous choice of the eigenvector matrix $\mbQ(t)$ is not required. Indeed,
\begin{equation*}
    \mbQ(t)\mbA_t(s)\mbQ(t)^\top = e^{-s\mbC(t)/2}, \quad  \mbQ(t)\mbB_t(s)\mbQ(t)^\top = \left(\bbI_K-e^{-s\mbC(t)}\right)^{1/2}.
\end{equation*}
Consequently, after undoing the rotation, the expressions for the derivatives of $u$ depend on $t$ through $\mbm(t)$, $\mbS(t)$, $\mbr$, and the continuous matrix functions of $\mbC(t)$ displayed above, rather than through the particular choice of $\mbQ(t)$.

Continuity and positive definiteness of $\mbS(t)$ and $\mbSigma(t)$ imply uniform upper and lower spectral bounds for $\mbC(t)$ on $[0,T]$. Therefore the bounds above may be chosen locally uniformly in $(\mbx,t)$. Dominated convergence in the Gaussian expectation and in $s$ then gives joint continuity of the first and second spatial derivatives in $\mbx$ coordinates.

\medskip
\textbf{Step \ref{item:poisson}}:
We now verify that $u$ satisfies the Poisson equation \cref{app:eqn:weighted-poisson-theorem} pointwise.
By step \ref{item:whiten}, it suffices to show that $\chi$ satisfies the Poisson equation in the transformed coordinates \eqref{app:eqn:whitened-poisson-theorem}.

The estimates established in step \ref{item:reg} show that $(P_s^t g)(\tilde\mbx)$ is twice continuously differentiable in $\tilde\mbx$.
For each $s>0$, $(P_s^t g)$ is also continuously differentiable in $s$ and satisfies the backward Kolmogorov equation
\begin{equation*}
    \frac{\partial}{\partial s}(P_s^t g)(\tilde\mbx) = (\tilde{\cA}_t P_s^t g)(\tilde\mbx), \qquad s>0.
\end{equation*}

Moreover, the spatial derivatives of $(P_s^t g)$ admit integrable bounds in $s$, locally uniformly in $\tilde\mbx$ for fixed $t$.
Therefore the operator $\tilde{\cA}_t$ may be passed under the integral sign in \cref{app:eqn:chi-representation-theorem}, yielding
\begin{equation*}
\begin{aligned}
    (\tilde{\cA}_t \chi(\cdot,t))(\tilde\mbx) &= -\int_0^\infty (\tilde{\cA}_t P_s^t g)(\tilde\mbx)ds\\
    &= -\int_0^\infty \frac{\partial}{\partial s}(P_s^t g)(\tilde\mbx)ds\\
    &= g(\tilde\mbx,t)-\lim_{R\to\infty}(P_R^t g)(\tilde\mbx).
\end{aligned}
\end{equation*}
By the tail estimate established in step \ref{item:reg},
\begin{equation*}
    |(P_R^t g)(\tilde\mbx)| \leq C_t'(1+\|\tilde\mbx\|_2^{m+1})e^{-c_tR} \to 0 \qquad\text{as }R\to\infty,
\end{equation*}
for each $(\tilde\mbx,t)$.
Hence $(\tilde{\cA}_t\chi(\cdot,t))(\tilde\mbx)=g(\tilde\mbx,t)$ for every $(\tilde\mbx,t)$, and consequently
\begin{equation*}
    u(\mbx,t) = \chi\left( \mbQ(t)^\top\mbS(t)^{-1/2}(\mbx-\mbm(t)),t \right)
\end{equation*}
satisfies \cref{app:eqn:weighted-poisson-theorem} pointwise. The gradient component of the Helmholtz decomposition is then $\nabla_{\mbx}\psi(\mbx,t)=\tfrac{1}{2}\nabla_{\mbx}u(\mbx,t)$.

\medskip
\textbf{Step \ref{item:unique}}:
It remains to identify the constructed vector fields with the components of the weak Helmholtz decomposition in
\Cref{app:thm:hilbert-projection}. 
The estimates in step \ref{item:reg} show that $\psi$, $\nabla_{\mbx}\psi$, and $\mbh$ have polynomial growth. 
Since $q(\mbx,t)$ is Gaussian, $\mbSigma(t)\nabla_{\mbx}\psi(\mbx,t)$ and $\mbh(\mbx,t)$ belong to $\cL^2(q(\mbx, t), \bbR^K, \mbSigma(t))^K$.

First, by appealing to integration by parts, the pointwise identity $\nabla_{\mbx}\cdot \left[q(\mbx,t)\mbh(\mbx,t)\right] = 0$ implies that $\mbh(\mbx,t)$ is $q$-divergence-free in the weak sense \cref{app:eqn:div-free-weak}. 
We next show that $\mbSigma(t)\nabla_{\mbx}\psi(\cdot,t)$ is orthogonal to all weakly $q$-divergence-free vector fields.

Fix $\mbv \in \cV_t$, and define an infinitely differentiable function $\zeta_R$ satisfying $\zeta_R=1$ on the ball of radius $R$, $\zeta_R=0$ outside the ball of radius $3R/2$, and $\|\nabla_{\mbx}\zeta_R\|_2 \leq \frac{C}{R}$.
$\zeta_R(\mbx)\psi(\mbx,t)$ is twice continuously differentiable and compactly supported on the ball of radius $3R/2$. 
We may therefore choose $\{\varphi_{R,k}\}_{k=1}^{\infty}$ infinitely differentiable and compactly supported on the ball of radius $2R$ such that
\begin{equation*}
    \nabla_{\mbx}\varphi_{R,k} \to \nabla_{\mbx}\left[    \zeta_R\psi(\cdot,t) \right]
\end{equation*}
uniformly. Since $\mbv\in\cV_t$,
\begin{equation*}
    \int
    \mbv(\mbx)^\top\nabla_{\mbx}\varphi_{R,k}(\mbx)
    q(\mbx,t)d\mbx
    =
    0
\end{equation*}
for every $k$. 
Cauchy--Schwarz allows us to pass to the limit as $k\to\infty$, producing
\begin{equation*}
    \int \mbv(\mbx)^\top \nabla_{\mbx}\left[   \zeta_R(\mbx)\psi(\mbx,t) \right] q(\mbx,t)d\mbx = 0.
\end{equation*}
The polynomial growth of $\psi$ and $\nabla_{\mbx}\psi$, together with the Gaussian tails of $q(\mbx,t)$, gives
\begin{equation*}
\left\| \mbSigma(t) \nabla_{\mbx}\left[ \zeta_R(\mbx)\psi(\mbx,t) \right] - \mbSigma(t)\nabla_{\mbx}\psi(\mbx,t) \right\|_{q, t} \to 0
\end{equation*}
as $R\to\infty$.
Taking the limit in the preceding identity therefore gives
$\left\langle \mbv(\mbx), \mbSigma(t)\nabla_{\mbx}\psi(\mbx,t) \right\rangle_{q, t} = 0$. 
Since this holds for every $\mbv \in \cV_t$, we have the desired result.

Thus,
\begin{equation*}
\mbr(\mbx,t) = \mbSigma(t)\nabla_{\mbx}\psi(\mbx,t) + \mbh(\mbx,t)
\end{equation*}
is precisely the unique orthogonal decomposition supplied by \Cref{app:thm:hilbert-projection}. 
In particular, the two vector fields are unique up to $q(\mbx,t)$-almost-everywhere equality among all square-integrable weak Helmholtz decompositions. 
Because $q(\mbx,t) > 0$ for every $\mbx$ and the constructed vector fields are continuous, almost-everywhere uniqueness implies pointwise uniqueness.
\end{proof}

\Cref{app:thm:helm-gaussian} implies that for Gaussian marginals and a state-independent diffusion coefficient, the Helmholtz decomposition exists and admits the probabilistic solution
\begin{equation}\label{app:eq:prob-rep}
\begin{aligned}
    \nabla_{\mbx}\psi(\mbx,t)
    &=
    \frac{1}{2}\mbS(t)^{-1/2}\mbQ(t)
    \nabla_{\tilde\mbx}\chi\left(
        \mbQ(t)^\top\mbS(t)^{-1/2}(\mbx-\mbm(t)),t
    \right),\\
    \nabla_{\tilde\mbx}\chi(\tilde\mbx,t)
    &=
    -\int_0^\infty
    \mbA_t(s)\mbB_t(s)^{-1}
    \bbE_{\mbxi\sim\cN(\mb0,\bbI_K)}
    \left[
        \mbxi\,
        g\left(
            \mbA_t(s)\tilde\mbx+
            \mbB_t(s)\mbxi,t
        \right)
    \right]ds,\\
    \mbA_t(s)
    &=
    e^{-s\mbLambda(t)/2},
    \qquad
    \mbB_t(s)
    =
    \left(\bbI_K-e^{-s\mbLambda(t)}\right)^{1/2},\\
    \mbC(t)
    &=
    \mbS(t)^{-1/2}\mbSigma(t)\mbS(t)^{-1/2}
    =
    \mbQ(t)\mbLambda(t)\mbQ(t)^\top,\\
    g(\tilde\mbx,t)
    &=
    \nabla_{\tilde\mbx}\cdot
    \widetilde{\mbrho}(\tilde\mbx,t)
    -
    \tilde\mbx^\top\widetilde{\mbrho}(\tilde\mbx,t),\\
    \widetilde{\mbrho}(\tilde\mbx,t)
    &=
    \mbQ(t)^\top\mbS(t)^{-1/2}
    \mbr\left(
        \mbm(t)+\mbS(t)^{1/2}\mbQ(t)\tilde\mbx,t
    \right).
\end{aligned}
\end{equation}

\subsection{Polynomial approximation of the optimal correction}
\label{app:subsec:resid-approx}

In \Cref{subsec:helmholtz-decomp} we showed that at fixed marginals $\{q(\mbx, t)\}_{0 \leq t \leq T}$, the $q$-divergence-free component $\mbh$ of the weighted Helmholtz decomposition is the vector field $\mbv$ that minimizes
\begin{equation}\label{app:eq:fixed-marginal-objective}
\int_0^T \bbE_{\bbQ} \left\|\mbSigma(t)^{-1/2} \{\mbr(\mbx(t), t)-\mbv(\mbx(t), t)\}\right\|^2dt.
\end{equation}
However, even when $q(\mbx, t)$ is Gaussian, the correction generally admits only the probabilistic solution \cref{app:eqn:helmholtz-soln}.
In our preliminary experiments, we found that estimating \cref{app:eqn:helmholtz-soln} directly has prohibitively high variance.
We make the problem tractable by restricting the divergence-free correction $\mbv$ to a finite-dimensional subspace comprised of polynomials of degree at most $\ell$.
Within this subspace, the minimizer of \cref{app:eq:fixed-marginal-objective} can be computed coefficient by coefficient in a Hermite polynomial basis.

Throughout this subsection, we assume only that $(\mbm(t),\mbS(t))_{0 \leq t \leq T}$, $\mbr(\mbx, t)$, $\mbSigma(t)$ satisfy the regularity conditions in \Cref{app:assum:resid}.

\paragraph{Whiten and rotate.}
First, we define the whitened coordinates $\tilde\mbx$ just as we did in the proof of \Cref{app:thm:helm-gaussian}:
\begin{equation}\label{app:eq:hermite-coordinates}
\tilde\mbx = \mbQ(t)^\top\mbS(t)^{-1/2}(\mbx-\mbm(t)), \quad \widetilde{\mbrho}(\tilde\mbx, t) =\mbQ(t)^\top\mbS(t)^{-1/2}\mbr(\mbm(t)+\mbS(t)^{1/2}\mbQ(t)\tilde\mbx, t),
\end{equation}
where $\mbQ(t)\mbLambda(t)\mbQ(t)^\top$ is the eigendecomposition of $\mbC(t)=\mbS(t)^{-1/2}\mbSigma(t)\mbS(t)^{-1/2}$.
In the whitened coordinates, $\tilde\mbx\sim\cN(\mb0,\bbI_K)$.

Suppose $\mbnu(\tilde\mbx,t)$ is divergence-free with respect to the standard Gaussian density.
In the original coordinates, this vector field is
\begin{equation}\label{app:eq:hermite-unwhiten}
\mbv(\mbx, t) =\mbS(t)^{1/2}\mbQ(t)\mbnu\left(\mbQ(t)^\top\mbS(t)^{-1/2}(\mbx-\mbm(t)), t\right).
\end{equation}
The objective and constraint become
\begin{align}
&\bbE_{\tilde\mbx\sim\cN(\mb0,\bbI_K)}\left[(\widetilde{\mbrho}(\tilde\mbx,t)-\mbnu(\tilde\mbx,t))^\top\mbLambda(t)^{-1}(\widetilde{\mbrho}(\tilde\mbx,t)-\mbnu(\tilde\mbx,t))\right]\label{app:eq:hermite-objective}\\
&\nabla_{\tilde\mbx}\cdot\mbnu(\tilde\mbx,t)-\tilde\mbx^\top\mbnu(\tilde\mbx,t)=0.\label{app:eq:hermite-div-free}
\end{align}
Indeed, \cref{app:eq:hermite-div-free} is exactly $\nabla_{\mbx}\cdot[q(\mbx,t)\mbv(\mbx,t)]=0$ after the change of variables.

\paragraph{Closed-form Hermite projection.}
Let $\mathrm{He}_d(u)$, $d\in\bbN$, be the $d$th unnormalized probabilists' Hermite polynomial.
The collection $\{\mathrm{He}_d(u)/\sqrt{d!}\}_{d=0}^{\infty}$ constitutes an orthonormal basis of the Hilbert space $\cL^2(\cN(0,1);\bbR)$, which is comprised of functions $f:\bbR\to\bbR$ satisfying $\bbE_{u\sim\cN(0,1)}[f(u)^2]<\infty$ with the inner product $\langle f,g\rangle=\bbE_{u\sim\cN(0,1)}[f(u)g(u)]$.

We can extend the univariate Hermite polynomial basis to an orthonormal basis of $\cL^2(\cN(\mb0,\bbI_K);\bbR)$, comprised of functions $f:\bbR^K\to\bbR$ satisfying $\bbE_{\tilde\mbx\sim\cN(\mb0,\bbI_K)}[f(\tilde\mbx)^2]<\infty$ with the analogous inner product.
Define the collection of multi-indices
\begin{equation*}
\cN_{K, \ell} = \left\{ \mbalpha \in \{0, 1, \ldots, \ell\}^K \bigg| \sum_{j=1}^K \mbalpha_j \leq \ell\right\}, \quad \cN_{K, +\infty} = \{0, 1, \ldots\}^K, \quad \ell \geq 0.
\end{equation*}
For a multi-index $\mbalpha \in \cN_{K, +\infty}$, we define the corresponding normalized multivariate probabilists' Hermite polynomial
\begin{equation}\label{app:eq:hermite-basis}
\hat{H}_{\mbalpha}(\tilde\mbx) =\prod_{i=1}^K \frac{\mathrm{He}_{\mbalpha_i}(\tilde\mbx_i)}{\sqrt{\mbalpha_i!}}.
\end{equation}
$\{\hat{H}_{\mbalpha}(\tilde\mbx)\}_{\mbalpha\in\cN_{K,+\infty}}$ forms an orthonormal basis of $\cL^2(\cN(\mb0,\bbI_K);\bbR)$.

By \Cref{app:assum:resid}, $\widetilde{\mbrho}_i(\tilde\mbx,t)$, $1\leq i\leq K$, has polynomial growth and hence is in $\cL^2(\cN(\mb0,\bbI_K);\bbR)$.
Therefore, we can expand the rotated residual in the multivariate Hermite basis as
\begin{equation}\label{app:eq:hermite-coefficients}
\widetilde{\mbrho}_i(\tilde\mbx,t)=\sum_{\mbalpha\in\cN_{K,+\infty}}\widetilde{\mbb}_{\mbalpha,i}(t)\hat{H}_{\mbalpha}(\tilde\mbx), \qquad \widetilde{\mbb}_{\mbalpha,i}(t)=\bbE_{\tilde\mbx\sim\cN(\mb0,\bbI_K)}\left[\widetilde{\mbrho}_i(\tilde\mbx,t)\hat{H}_{\mbalpha}(\tilde\mbx)\right].
\end{equation}
We choose this basis expansion because of the identity
\begin{equation}\label{app:eq:hermite-raising}
(\partial_{\tilde\mbx_i}-\tilde\mbx_i)\hat{H}_{\mbalpha}(\tilde\mbx) =-\sqrt{\mbalpha_i+1}\,\hat{H}_{\mbalpha+\mbe_i}(\tilde\mbx).
\end{equation}
The Gaussian weighted divergence raises the Hermite degree by one and creates no lower-order terms.
Rather than keeping the infinitely many basis vectors in the expansion of $\widetilde{\mbrho}_i(\tilde\mbx, t)$, we truncate the expansion \cref{app:eq:hermite-coefficients} to $\{\hat{H}_{\mbalpha}(\tilde\mbx)\}_{\mbalpha \in \cN_{K, \ell}}$, which span the $\binom{K + \ell}{\ell}$-dimensional subspace of polynomials having degree at most $\ell$.
In \Cref{app:prop:hermite-projection}, we show that by doing so, we can obtain the optimal correction to the reference posterior drift restricted to this subspace.
\begin{proposition}[Hermite projection]\label{app:prop:hermite-projection}
Among all polynomial fields $\mbnu(\tilde\mbx, t)$ of degree at most $\ell$ satisfying \cref{app:eq:hermite-div-free}, the unique minimizer of \cref{app:eq:hermite-objective} is
\begin{equation}\label{app:eq:hermite-solution}
\begin{aligned}
\mu_{\mbbeta}(t)&=\frac{\sum_{i=1}^K\sqrt{\mbbeta_i}\, \widetilde{\mbb}_{\mbbeta-\mbe_i,i}(t)} {\sum_{i=1}^K\lambda_i(t)\mbbeta_i}, &&\mbbeta \in \cN_{K, \ell+1} \setminus \cN_{K, 0},\\
\widetilde{\mbc}_{\mbalpha,i}(t) &=\widetilde{\mbb}_{\mbalpha,i}(t) -\lambda_i(t)\sqrt{\mbalpha_i+1}\, \mu_{\mbalpha+\mbe_i}(t), &&\mbalpha \in \cN_{K, \ell},\\
\mbnu^{(\ell)}(\tilde\mbx, t) &=\sum_{\mbalpha \in \cN_{K, \ell}} \widetilde{\mbc}_{\mbalpha}(t)\hat{H}_{\mbalpha}(\tilde\mbx),
\end{aligned}
\end{equation}
where a coefficient with a negative multi-index is zero.
The optimal correction $\mbh^{(\ell)}(\mbx, t)$ in the original coordinates is obtained from $\mbnu^{(\ell)}(\tilde\mbx, t)$ using \cref{app:eq:hermite-unwhiten}.
\end{proposition}

\begin{proof}
Note that since $\mbnu$ is a polynomial field, its components are in $\cL^2(\cN(\mb0,\bbI_K);\bbR)$.
Expanding $\mbnu$ in the multivariate Hermite basis $\mbnu_i(\tilde\mbx, t)=\sum_{|\mbalpha|\leq\ell}\widetilde{\mbc}_{\mbalpha,i}(t)\hat{H}_{\mbalpha}(\tilde\mbx)$ and applying \cref{app:eq:hermite-raising} to $\mbnu(\tilde\mbx, t)$ gives
\begin{equation}\label{app:eq:hermite-stein-action}
\nabla_{\tilde\mbx}\cdot\mbnu(\tilde\mbx, t) -\tilde\mbx^\top\mbnu(\tilde\mbx, t) =-\sum_{1 \leq |\mbbeta| \leq \ell+1} \left( \sum_{i=1}^K\sqrt{\mbbeta_i}\, \widetilde{\mbc}_{\mbbeta-\mbe_i,i}(t) \right)\hat{H}_{\mbbeta}(\tilde\mbx).
\end{equation}
Thus divergence freedom gives one scalar constraint for each $\mbbeta$.
Moreover, every coefficient $\widetilde{\mbc}_{\mbalpha,i}(t)$ occurs in exactly one constraint, namely $\mbbeta=\mbalpha+\mbe_i$.

By Parseval's identity, the part of the objective that depends on the correction is
\begin{equation*}
\sum_{|\mbalpha|\leq\ell} (\widetilde{\mbb}_{\mbalpha}(t)- \widetilde{\mbc}_{\mbalpha}(t))^\top \mbLambda(t)^{-1} (\widetilde{\mbb}_{\mbalpha}(t)- \widetilde{\mbc}_{\mbalpha}(t)).
\end{equation*}
For a fixed $\mbbeta$, the relevant coefficients solve
\begin{equation*}
\min_{\{c_i\}}\frac{1}{2}\sum_{i:\mbbeta_i>0}\frac{1}{\lambda_i(t)}\left(\widetilde{\mbb}_{\mbbeta-\mbe_i,i}(t)-c_i\right)^2 \quad\text{subject to}\quad \sum_{i:\mbbeta_i>0}\sqrt{\mbbeta_i}\,c_i=0.
\end{equation*}
Introducing the Lagrange multiplier $\mu_{\mbbeta}(t)$ and differentiating the Lagrangian yields the first-order conditions
\begin{equation*}
c_i = \widetilde{\mbb}_{\mbbeta-\mbe_i,i}(t) -\lambda_i(t)\sqrt{\mbbeta_i}\,\mu_{\mbbeta}(t).
\end{equation*}
Substituting this expression into the constraint yields
\begin{equation*}
\begin{aligned}
&\sum_i\sqrt{\mbbeta_i} c_i = \sum_i\sqrt{\mbbeta_i} \widetilde{\mbb}_{\mbbeta-\mbe_i,i}(t) - \mu_{\mbbeta}(t)\sum_i\lambda_i(t)\mbbeta_i=0,\\
\implies& \mu_{\mbbeta}(t) =\frac{\sum_{i=1}^K\sqrt{\mbbeta_i}\widetilde{\mbb}_{\mbbeta-\mbe_i,i}(t)} {\sum_{i=1}^K\lambda_i(t)\mbbeta_i}.
\end{aligned}
\end{equation*}
Because every $\lambda_i(t)>0$, the objective is strictly convex.
Consequently, this feasible stationary point is the unique minimizer.
\end{proof}

Recall from \Cref{app:subsec:helmholtz-compute} that the weighted Helmholtz decomposition can be obtained by solving a Poisson equation \cref{app:eqn:poisson} defined by the infinitesimal generator of the Langevin diffusion \cref{app:eqn:langevin-diffusion}.
In the whitened coordinates, the Langevin diffusion is a time-rescaled Ornstein--Uhlenbeck process \cref{app:eqn:ou-process-theorem}, and the eigenfunctions of the  infinitesimal generator are the Hermite polynomials.
Indeed, evaluating \cref{app:eq:hermite-raising} at $\mbalpha - \mbe_i$ and combining it with $\partial_{\tilde\mbx_i} \hat{H}_{\mbalpha}(\tilde\mbx) = \sqrt{\mbalpha_i} \hat{H}_{\mbalpha - \mbe_i}(\tilde\mbx)$ produces
\begin{equation*}
    (\partial_{\tilde\mbx_i}-\tilde\mbx_i)\partial_{\tilde\mbx_i} \hat{H}_{\mbalpha}(\tilde\mbx) = \sqrt{\mbalpha_i} (\partial_{\tilde\mbx_i}-\tilde\mbx_i) \hat{H}_{\mbalpha - \mbe_i}(\tilde\mbx) = - \mbalpha_i \hat{H}_{\mbalpha}(\tilde\mbx).
\end{equation*}
Hence, \Cref{app:prop:hermite-projection} can be understood as expanding the left and right-hand sides of the Poisson equation in the orthonormal basis of eigenfunctions and then solving for the coefficients belonging to $\cN_{K, \ell+1}$ (indexed by $\mbbeta$).
We then differentiate the result to obtain the gradient field $\nabla_{\tilde\mbx} \chi^{(\ell)}(\tilde\mbx,t)$.

\paragraph{Two useful special cases.}
When $\ell = 0$, $\mbbeta=\mbe_i$, \cref{app:eq:hermite-solution} gives $\widetilde{\mbc}_{\mb0,i}(t)=0$, so a compatible correction has no constant term and cannot change the prescribed mean path.
For $\ell=1$, write $\mbnu^{(1)}(\tilde\mbx, t)=\widetilde{\mbA}(t)\tilde\mbx$ and set
\begin{equation*}
\widetilde{\mbB}_{ij}(t) =\widetilde{\mbb}_{\mbe_j,i}(t) =\bbE_{\tilde\mbx\sim\cN(\mb0,\bbI_K)}\left[\widetilde{\mbrho}_i(\tilde\mbx,t)\tilde\mbx_j\right].
\end{equation*}
By Stein's identity, $\widetilde{\mbB}(t)=\bbE_{\tilde\mbx\sim\cN(\mb0,\bbI_K)}[D_{\tilde\mbx}\widetilde{\mbrho}(\tilde\mbx,t)]$.
The solution reduces to
\begin{equation}\label{app:eq:hermite-degree1}
\widetilde{\mbA}_{ij}(t) = \frac{\lambda_j(t)\widetilde{\mbB}_{ij}(t) -\lambda_i(t)\widetilde{\mbB}_{ji}(t)}{\lambda_i(t)+\lambda_j(t)}, \qquad \widetilde{\mbA}_{ii}(t)=0.
\end{equation}
Hence the linear correction is skew-symmetric in whitened coordinates.
If $\mbLambda(t)$ is proportional to the identity, it is simply the antisymmetric part of the averaged Jacobian.

\paragraph{Monte Carlo.}
Computing the polynomial approximation to the optimal divergence-free vector field requires computing the coefficients of $\widetilde{\mbrho}_i(\tilde\mbx, t)$ in the Hermite basis.
Generally, these coefficients cannot be computed in closed-form.
However, they do have the unbiased estimator
\begin{equation}\label{app:eq:hermite-mc}
\widehat{\widetilde{\mbb}}_{\mbalpha,i}(t)=\frac{1}{N_{\mathrm{mc}}}\sum_{n=1}^{N_{\mathrm{mc}}}\widetilde{\mbrho}_i(\tilde\mbx^{(n)},t)\hat{H}_{\mbalpha}(\tilde\mbx^{(n)}), \qquad \tilde\mbx^{(n)}\overset{\mathrm{i.i.d.}}{\sim}\cN(\mb0,\bbI_K).
\end{equation}
Substituting any estimated coefficients into \cref{app:eq:hermite-solution} still makes \cref{app:eq:hermite-stein-action} vanish identically.
Thus, the estimated field is $q$-divergence-free, although it will generally not be the optimal degree $\ell$ polynomial $q$-divergence-free vector field.
In the stochastic VI algorithm, the batch of points $\mbx_{j, k} \overset{\mathrm{i.i.d.}}{\sim} q(\mbx, t_j)$ used to evaluate the path KL $\cL_{\mathrm{path}}$ and the batch used to compute the correction to the reference drift at these points should be independent.

\paragraph{Gauss--Hermite quadrature.}
Alternatively, for low dimensional latent spaces, we can use Gauss--Hermite quadrature to estimate the Hermite coefficients of the residual vector field. If $N_{\mathrm{quad}}$ is the number of quadrature points along each latent dimension, the estimator is
\begin{equation}\label{app:eq:hermite-quad}
\widehat{\widetilde{\mbb}}_{\mbalpha,i}(t)=\sum_{n=1}^{N_{\mathrm{quad}}^K} w_n \widetilde{\mbrho}_i(\tilde\mbx_n, t)\hat{H}_{\mbalpha}(\tilde\mbx_n), \quad \sum_{n=1}^{N_{\mathrm{quad}}^K} w_n = 1,
\end{equation}
where the sum is taken over the Cartesian product of one-dimensional quadrature points $\tilde\mbx_n$, and $w_n$ is the product of the corresponding one-dimensional quadrature weights.
When $\widetilde{\mbrho}_i(\tilde\mbx, t) \hat{H}_{\mbalpha}(\tilde\mbx)$ is a polynomial of degree at most $2 N_{\mathrm{quad}} - 1$, Gauss--Hermite quadrature computes the Hermite coefficient $\widetilde{\mbb}_{\mbalpha,i}(t)$ exactly.

\paragraph{Stein estimand.}
If derivatives of the residual are available, Stein's identity provides an alternate estimator for $\widetilde{\mbb}_{\mbalpha,i}(t)$.
For any multi-index $\mbkappa$ such that $\mbkappa_j \leq \mbalpha_j, 1 \leq j \leq K$, repeated integration by parts gives
\begin{equation}\label{app:eq:stein-estimator}
\widetilde{\mbb}_{\mbalpha,i}(t)=\sqrt{\frac{(\mbalpha-\mbkappa)!}{\mbalpha!}}\,\bbE_{\tilde\mbx\sim\cN(\mb0,\bbI_K)}\left[\partial_{\tilde\mbx}^{\mbkappa}\widetilde{\mbrho}_i(\tilde\mbx,t)\hat{H}_{\mbalpha-\mbkappa}(\tilde\mbx)\right].
\end{equation}
The expectation in \cref{app:eq:stein-estimator} can be approximated using either Monte Carlo or Gauss--Hermite quadrature.
The choice $\mbkappa=\mb0$ recovers \cref{app:eq:hermite-mc}.
If derivatives through order $|\mbalpha|$ are available, choosing $\mbkappa=\mbalpha$ removes the Hermite polynomial entirely:
\begin{equation*}
\widetilde{\mbb}_{\mbalpha,i}(t) =\frac{1}{\sqrt{\mbalpha!}} \bbE_{\tilde\mbx\sim\cN(\mb0,\bbI_K)}\left[\partial_{\tilde\mbx}^{\mbalpha} \widetilde{\mbrho}_i(\tilde\mbx,t) \right].
\end{equation*}
Under the twice-differentiability condition in \Cref{app:assum:resid}, one may use any $\mbkappa \in \cN_{K, 2}$ with $\mbkappa \leq \mbalpha$ componentwise.
The derivatives with respect to $\tilde\mbx$ can be obtained from spatial derivatives of $\mbr$ through the affine change of coordinates in \cref{app:eq:hermite-coordinates} and evaluated using automatic differentiation.

Although the normalized Hermite polynomials have unit second moments, their higher moments grow rapidly with their degree.
Integration by parts can therefore reduce the variance of the coefficient estimates by replacing a high-degree Hermite factor with derivatives of the residual.
This reduction is not guaranteed, since the derivatives may themselves have high variance.
The choice of $\mbkappa$ trades the degree of the remaining Hermite factor against the cost and variability of higher-order derivatives.
Regardless of this choice, substituting the resulting coefficient estimates into \cref{app:eq:hermite-solution} preserves the divergence-free constraint exactly.

\paragraph{Computational complexity.}
There are $K\binom{K+\ell}{\ell}$ scalar coefficients.
Once they have been estimated, \cref{app:eq:hermite-solution} touches each coefficient a constant number of times and costs $O(K\binom{K+\ell}{\ell})$, rather than $O(\binom{K+\ell+1}{\ell+1}^3)$ for a generic dense solve of the polynomial Poisson equation.
Including residual evaluations, Hermite features, and eigendecomposition, the cost at one time point is
\begin{equation*}
O\left(N_{\mathrm{est}}C_{\mbr} + N_{\mathrm{est}} K\binom{K+\ell}{\ell}+K^3\right),
\end{equation*}
where $C_{\mbr}$ is the cost of one residual evaluation and $N_{\mathrm{est}}$ is the number of Monte Carlo samples or quadrature points used to approximate the Hermite coefficients.
When $\ell$ is fixed, $\binom{K+\ell}{\ell} = \Theta(K^{\ell})$, so the coefficient-dependent term scales as $K^{\ell+1}$.

The Stein estimand \cref{app:eq:stein-estimator} additionally requires derivatives of the transformed residual.
For the full-order choice $\mbkappa=\mbalpha$, the highest-order derivative tensor has $K^\ell$ entries for each of the $K$ output components, and hence $K^{\ell+1}$ entries in total.
Using automatic differentiation, computing these derivatives costs $O(K^{\ell+1} C_{\mbr})$.

\paragraph{Relationship to Taylor approximation.}
In an earlier version of this work, we constructed the polynomial correction by approximating the residual $\mbr$ with a multivariate Taylor expansion about the posterior mean $\mbm(t)$ and then computing the weighted Helmholtz decomposition of the resulting polynomial vector field.
The present version also computes the weighted Helmholtz decomposition of a polynomial vector field, but its coefficients are obtained by expanding the residual in the truncated multivariate Hermite polynomial basis.

The distinction between these two approaches is their approximation criterion: a Taylor expansion matches local derivatives at the posterior mean, whereas a truncated Hermite expansion with exact coefficients gives the orthogonal projection of the residual onto the space of polynomial vector fields having degree at most $\ell$, with respect to the Gaussian marginal.
As shown above, the weighted Helmholtz decomposition of this projection yields the optimal degree $\ell$ polynomial $q$-divergence-free correction at fixed marginals.
In practice, the Hermite coefficients are generally estimated numerically, and the resulting correction need not attain this optimum.
Nonetheless, both the Taylor and Hermite constructions preserve the $q$-divergence-free constraint exactly.

Moreover, the two approximations are closely related through the Stein representation. 
For $\ell=1$, the Taylor approximation uses the Jacobian of the residual evaluated at the posterior mean, while the Hermite construction uses the expectation of the Jacobian under the Gaussian marginal. 
Consequently, approximating this expectation with one-node Gauss--Hermite quadrature recovers the degree-one Taylor correction.

\subsection{Verification of assumptions for the Helmholtz-SDE variational family}
\label{app:subsec:verify-assumptions}

In \Cref{app:subsec:helmholtz-gaussian,app:subsec:resid-approx} we discussed how to obtain a correction $\mbh^{(\ell)}$ to the reference posterior drift from the weighted Helmholtz decomposition of the residual vector field $\mbr = \mbf_p - \mbf_q$.
It remains to establish that this correction produces a valid variational SDE.
In this subsection, we verify that the Gaussian variational family used by Helmholtz-SDE satisfies the assumptions in \Cref{app:subsec:assumptions}.

Let $q(\mbx,t)=\cN(\mbm(t),\mbS(t))$.
We assume that the prior SDE admits a unique strong solution on $[0,T]$ and that $(\mbm(t),\mbS(t))_{0\leq t\leq T}$, $\mbr(\mbx,t)$, and $\mbSigma(t)$ satisfy the regularity conditions in \Cref{app:assum:resid}. \Cref{app:assum:regularity,app:assum:densities,app:ass:path-ac,app:ass:finite-kl,app:ass:fpk-unique} are verified under these conditions for $0 \leq \ell \leq 3$.

Helmholtz-SDE uses the reference drift
\begin{equation*}
    \mbf_q(\mbx,t)  = \frac{d}{dt}\mbm(t) + \left\{ \left(\frac{d}{dt}\mbL(t)\right)\mbL(t)^{-1} -\frac{1}{2}\mbSigma(t)\mbS(t)^{-1} \right\} (\mbx-\mbm(t)),
\end{equation*}
where $\mbL(t)$ is the symmetric square root of $\mbS(t)$.
The full posterior drift is $\mbf_q(\mbx,t)+\mbh^{(\ell)}(\mbx,t)$, where $\mbh^{(\ell)}$ is the degree $\ell$ correction constructed in \Cref{app:subsec:resid-approx}.  
The posterior initial density is $q(\mbx,0)$, and the prior initial distribution is $\cN(\mbmu_0,\mbV_0)$, with $\mbV_0$ positive definite.

Because $[0,T]$ is compact, the eigenvalues of $\mbS(t)$ and $\mbSigma(t)$ are bounded above and bounded away from zero.  
In unrotated whitened coordinates, the polynomial coefficients defining $\mbh^{(\ell)}$ are continuous, and hence bounded, in $t$.  
Indeed, continuity follows from \Cref{app:assum:resid}, dominated convergence under the Gaussian marginals, and the finite-dimensional projection in a fixed basis in \Cref{app:prop:hermite-projection}.

\paragraph{Existence and uniqueness.}
Introduce the whitened variable $\mbz$
\begin{equation*}
    \mbz=\mbL(t)^{-1}(\mbx-\mbm(t)), \qquad \mbnu^{(\ell)}(\mbz,t) =\mbL(t)^{-1} \mbh^{(\ell)}(\mbm(t)+\mbL(t)\mbz,t).
\end{equation*}
For the symmetric-square-root reference drift, the terms involving $d\mbL(t)/dt$ cancel under this change of variables. 
Consequently, the posterior SDE becomes
\begin{equation}
\label{app:eq:whitened-posterior-sde}
\begin{aligned}
    d\mbz(t) &= \left\{ -\frac{1}{2}\mbC(t)\mbz(t) +\mbnu^{(\ell)}(\mbz(t),t) \right\}dt+\mbL(t)^{-1}\mbSigma(t)^{1/2}d\mbw(t),\\\mbC(t) &=\mbS(t)^{-1/2}\mbSigma(t)\mbS(t)^{-1/2}.
\end{aligned}
\end{equation}
The drift in \cref{app:eq:whitened-posterior-sde} is polynomial in $\mbz$, and is therefore locally Lipschitz.  
This gives a pathwise-unique strong solution up to its explosion time.

It remains to rule out explosion.
The correction is Gaussian divergence-free, so \cref{app:eq:hermite-div-free} gives
\begin{equation*}
    \mbz^\top\mbnu^{(\ell)}(\mbz,t)=\nabla_{\mbz}\cdot\mbnu^{(\ell)}(\mbz,t).
\end{equation*}

For $\ell=0$, the correction is zero.  For $1\leq\ell\leq3$, the right-hand side is a polynomial of degree at most $\ell-1$, and hence for some constant $C > 0$
\begin{equation*}
    \left|\mbz^\top\mbnu^{(\ell)}(\mbz,t)\right| \leq CV(\mbz), \qquad (\mbz,t) \in \bbR^K \times [0,T],
\end{equation*}
where $V$ is the Lyapunov function $V(\mbz)=1+\|\mbz\|^2$. 
The infinitesimal generator of \cref{app:eq:whitened-posterior-sde} satisfies
\begin{equation*}
    \cL_t V(\mbz)=-\mbz^\top\mbC(t)\mbz +2\mbz^\top\mbnu^{(\ell)}(\mbz,t) +\mathrm{tr}(\mbC(t)),
\end{equation*}
and is therefore bounded above by $C V(\mbz)$.  Moreover,
\begin{equation*}
    \lim_{R\to\infty} \inf_{\|\mbz\|>R}V(\mbz)=\infty.
\end{equation*}

The non-explosion criterion from \citet[Theorem~3.5]{khasminskii2011stochastic} therefore implies that a pathwise-unique, strong solution is defined on all of $[0, T]$. 
This proves \Cref{app:assum:regularity} for the posterior SDE.
Notice that our argument does not require the posterior drift to have linear
growth.

\paragraph{Fokker--Planck equation and Gaussian marginals.}
By construction, $\{q(\mbx,t)\}_{0 \leq t \leq T}$ is a weak solution to the Fokker--Planck equation for the reference drift $\mbf_q$.  
Moreover,
\begin{equation*}
    \nabla_{\mbx}\cdot  \left[q(\mbx,t)\mbh^{(\ell)}(\mbx,t)\right] =0, \quad (\mbx, t) \in \bbR^K \times [0, T]
\end{equation*}
so the same $q(\mbx,t)$ also solves the weak Fokker--Planck equation for the corrected drift $\mbf_q+\mbh^{(\ell)}$.  
The corrected drift is polynomial, $\mbSigma(t)$ is bounded, and Gaussian distributions have moments of every finite order.  
Hence $q$ belongs to the regularity class \cref{app:eqn:fpk-regularity-class}.

It remains to show that, within the regularity class \cref{app:eqn:fpk-regularity-class}, the Fokker--Planck equation has a unique solution. 
The superposition principle (\citet[Theorem~2.5]{trevisan2016well}) states that any weak Fokker--Planck solution $\{\rho(\mbx, t)\}_{0 \leq t \leq T}$ can be lifted to a probability distribution over continuous paths whose one-time marginals are $\{\rho(\mbx, t)\}_{0 \leq t \leq T}$ and which is a weak solution to the SDE with drift $\mbf_q+\mbh^{(\ell)}$ and diffusion coefficient $\mbSigma^{1/2}$.

Suppose $\{q(\mbx, t)\}_{0 \leq t \leq T}$ and $\{\tilde{q}(\mbx, t)\}_{0 \leq t \leq T}$ are two weak solutions to the Fokker--Planck equation with drift $\mbf_q + \mbh^{(\ell)}$ and diffusion coefficient $\mbSigma^{1/2}$. 
By the superposition principle, we obtain two probability distributions having one-time marginals $\{q(\mbx, t)\}_{0 \leq t \leq T}$ and $\{\tilde{q}(\mbx, t)\}_{0 \leq t \leq T}$ that are weak solutions to the variational SDE  with drift $\mbf_q + \mbh^{(\ell)}$ and diffusion coefficient $\mbSigma^{1/2}$.
The strong existence and pathwise uniqueness established above imply that this SDE is unique in law.
Therefore, the two weak solutions have the same path law and hence the same one-time marginals. 
Thus $q(\mbx,t)=\tilde q(\mbx,t)$ almost everywhere in $\mbx$ for every $t\in[0,T]$.
This verifies \Cref{app:ass:fpk-unique} for the posterior drift.

\paragraph{Densities.}
The posterior marginals are non-degenerate Gaussians and hence have smooth, strictly positive Lebesgue densities. 
This verifies the posterior part of
\Cref{app:assum:densities}.

\paragraph{Absolute continuity.}
Define the first exit time from the ball of radius $n$
\begin{equation*}
    \tau_n  = \inf\left\{t\in[0,T] \ | \  \|\mbx(t)\|_2\geq n \right\}, \qquad \inf\varnothing=\infty.
\end{equation*}
Since the prior and posterior drifts are locally bounded and $\mbSigma(t)^{-1/2}$ is bounded on $[0,T]$, there exists $K_n<\infty$ such that
\begin{equation*}
    \sup_{\substack{\|\mby\|_2\leq n\\0\leq t\leq T}} \left\| \mbSigma(t)^{-1/2} \left[ \mbf_q(\mby,t)+\mbh^{(\ell)}(\mby,t) -\mbf_p(\mby,t) \right] \right\|_2 \leq K_n.
\end{equation*}
Consequently,
{\small
\begin{equation*}
\begin{aligned}
    \bbE_{\bbP}\left[ \exp\left\{ \frac{1}{2} \int_0^{T\wedge\tau_n} \left\| \mbSigma(t)^{-1/2} \left[ \mbf_q(\mbx(t),t) +\mbh^{(\ell)}(\mbx(t),t) -\mbf_p(\mbx(t),t) \right] \right\|_2^2dt \right\} \right]  \leq  \exp\left\{\frac{1}{2}K_n^2T\right\} <\infty.
\end{aligned}
\end{equation*}
}
Thus, Novikov's criterion (\citet[Corollary~3.5.13]{karatzas2014brownian}) holds for the stopped process.

Both initial densities are non-degenerate Gaussian densities and are therefore strictly positive on $\bbR^K$. Define
{\small
\begin{equation*}
\begin{aligned}
    \cZ_T^{(n)}
    &= \frac{q(\mbx(0),0)}{p(\mbx(0),0)} \exp\Bigg\{ \sum_{i=1}^K \int_0^{T\wedge\tau_n} \left[ \mbSigma(t)^{-1/2} \left\{ \mbf_q(\mbx(t),t) +\mbh^{(\ell)}(\mbx(t),t) -\mbf_p(\mbx(t),t) \right\} \right]_i d\mbw_i(t)\\
    &-\frac{1}{2} \int_0^{T\wedge\tau_n} \left\| \mbSigma(t)^{-1/2} \left[ \mbf_q(\mbx(t),t)  +\mbh^{(\ell)}(\mbx(t),t)-\mbf_p(\mbx(t),t) \right] \right\|_2^2dt\Bigg\},
\end{aligned}
\end{equation*}
}
where $(\mbw(t))_{0\leq t\leq T}$ is a $K$-dimensional Brownian motion with respect to $\bbP$.

Novikov's criterion implies that the stochastic exponential appearing in $\cZ_T^{(n)}$ is a martingale. Conditionally on $\mbx(0)$, it has expectation one, and therefore
\begin{equation*}
    \bbE_{\bbP}\left[\cZ_T^{(n)}\right]
    = \bbE_{\bbP}\left[ \frac{q(\mbx(0),0)}{p(\mbx(0),0)}   \right] = \int \frac{q(\mbx,0)}{p(\mbx,0)} p(\mbx,0)\,d\mbx =1.
\end{equation*}
We may therefore define a probability measure
$\widehat{\bbQ}^{(n)}$ on $\mathcal F_T$ by
\begin{equation*}
    \frac{d\widehat{\bbQ}^{(n)}}{d\bbP} = \cZ_T^{(n)}.
\end{equation*}
Under $\widehat{\bbQ}^{(n)}$, Girsanov's theorem shows that
\begin{equation*}
\begin{aligned}
    \widehat{\mbw}^{(n)}(t) &= \mbw(t) - \int_0^{t\wedge\tau_n} \mbSigma(s)^{-1/2} \left\{ \mbf_q(\mbx(s),s) +\mbh^{(\ell)}(\mbx(s),s) -\mbf_p(\mbx(s),s) \right\}ds
\end{aligned}
\end{equation*}
is a Brownian motion.
Hence, up to $\tau_n$, the coordinate process satisfies
\begin{equation*}
    d\mbx(t) = \left[ \mbf_q(\mbx(t),t) +\mbh^{(\ell)}(\mbx(t),t) \right]dt + \mbSigma(t)^{1/2} d\widehat{\mbw}^{(n)}(t),
\end{equation*}
and its initial density is $q(\mbx,0)$.

Let $\cG_n = \sigma\left\{ \mbx(t\wedge\tau_n):0\leq t\leq T \right\}$ be the sigma-algebra generated by the stopped coordinate process.
Under both $\widehat{\bbQ}^{(n)}$ and $\bbQ$, the stopped coordinate process has initial density $q(\mbx,0)$ and satisfies the same stopped posterior SDE. 
Uniqueness in law therefore gives
\begin{equation*}
    \left.
    \widehat{\bbQ}^{(n)}
    \right|_{\mathcal G_n}
    =
    \left.
    \bbQ
    \right|_{\mathcal G_n}.
\end{equation*}
Because $\mbSigma(t)$ is invertible, the prior Brownian motion up to
$\tau_n$ is determined by the stopped coordinate process. Hence
$\cZ_T^{(n)}$ is $\mathcal G_n$-measurable. It follows that, for every
$A\in\mathcal G_n$,
\begin{equation*}
    \bbQ(A)
    =
    \widehat{\bbQ}^{(n)}(A)
    =
    \bbE_{\bbP}\left[
        \mathbbm{1}_A\cZ_T^{(n)}
    \right].
\end{equation*}

Define the corresponding unstopped stochastic exponential by
{\small
\begin{equation*}
\begin{aligned}
    \cZ_T &= \frac{q(\mbx(0),0)}{p(\mbx(0),0)} \exp\Bigg\{ \sum_{i=1}^K \int_0^T \left[ \mbSigma(t)^{-1/2} \left\{ \mbf_q(\mbx(t),t) +\mbh^{(\ell)}(\mbx(t),t) -\mbf_p(\mbx(t),t) \right\} \right]_i d\mbw_i(t) \\
    &- \frac{1}{2} \int_0^T \left\|    \mbSigma(t)^{-1/2} \left[\mbf_q(\mbx(t),t) +\mbh^{(\ell)}(\mbx(t),t)-\mbf_p(\mbx(t),t) \right] \right\|_2^2dt\Bigg\}.
\end{aligned}
\end{equation*}
}
This random variable is well defined $\bbP$-almost surely by local
boundedness and non-explosion.

To remove the stopping time, let $E_n=\{\tau_n>T\}$. 
Since the prior and posterior processes have strong solutions defined on all of $[0,T]$, their paths are continuous and bounded on $[0,T]$. 
Therefore, $E_n\uparrow\Omega$ both $\bbP$-almost surely and $\bbQ$-almost surely. 
Moreover, by definition, $\cZ_T^{(n)}=\cZ_T$ on $E_n$. 
For every $A$ measurable with respect to $\sigma(\mbx(t) \ | \ 0 \leq t \leq T)$, the event $A\cap E_n$ belongs to $\cG_n$, and hence
\begin{equation*}
\begin{aligned}
    \bbQ(A\cap E_n) &= \bbE_{\bbP}\left[ \mathbbm{1}_{A\cap E_n}\cZ_T^{(n)}  \right] = \bbE_{\bbP}\left[\mathbbm{1}_{A\cap E_n}\cZ_T \right].
\end{aligned}
\end{equation*}
Continuity from below under $\bbQ$ and monotone convergence under
$\bbP$ now give
\begin{equation*}
    \bbQ(A) = \lim_{n\to\infty}\bbQ(A\cap E_n) = \lim_{n\to\infty} \bbE_{\bbP}\left[\mathbbm{1}_{A\cap E_n}\cZ_T \right] = \bbE_{\bbP}\left[ \mathbbm{1}_A\cZ_T \right].
\end{equation*}
Therefore,
\begin{equation*}
    \frac{d\bbQ}{d\bbP}  = \cZ_T.
\end{equation*}
In particular, $\bbQ$ is absolutely continuous with respect to $\bbP$. 
Since $\cZ_T>0$ $\bbP$-almost surely, the reverse absolute continuity also holds. 
Therefore, the prior and posterior path laws are mutually absolutely continuous.

\paragraph{Finite KL divergence.}
The initial-distribution term is finite because both initial distributions are non-degenerate Gaussians.

For the path-space term, observe that
\begin{equation*}
    \mbf_p(\mbx,t) -\left\{\mbf_q(\mbx,t)+\mbh^{(\ell)}(\mbx,t)\right\}  =\mbr(\mbx,t)-\mbh^{(\ell)}(\mbx,t).
\end{equation*}

Both terms on the right have polynomial growth in $\mbx$, uniformly for $t\in[0,T]$, while $\mbSigma(t)^{-1}$ is uniformly bounded.
Since the posterior marginals are Gaussian and their means and covariances are bounded on $[0,T]$, they have uniformly bounded moments of every finite order.
Consequently,
\begin{equation*}
    \int_0^T \bbE_{\bbQ} \left\|    \mbSigma(t)^{-1/2} \left[ \mbf_p(\mbx(t),t) -\left\{\mbf_q(\mbx(t),t) +\mbh^{(\ell)}(\mbx(t),t)\right\}  \right] \right\|^2dt <\infty.
\end{equation*}
Thus \Cref{app:ass:finite-kl} holds, and the path-space KL identity \cref{app:eqn:path-kl} is finite.

We have therefore verified the assumptions of \Cref{app:subsec:assumptions} for the Helmholtz-SDE variational family with $0 \leq \ell \leq 3$, conditional on the stated assumptions on the prior.  
For $\ell \geq 4$, the divergence-free identity still preserves the prescribed Fokker--Planck equation, but the Lyapunov estimate above does not establish non-explosion. 
Global well-posedness must be verified separately.

\subsection{Optimality of Helmholtz-SDE}
\label{app:subsec:optimality}

Now that we have established that Helmholtz-SDE produces a valid variational SDE, we study its optimality.
First, we prove \Cref{thm:helm-optimal}, which shows Helmholtz-SDE approximates the posterior arbitrarily closely in the conjugate setting.
We then study the nonconjugate setting. 
We quantify the suboptimality resulting from restricting the $q$-divergence-free correction to a polynomial vector field of degree $\ell$.
As a result, we get a lower bound on the improvement over the variational family defined by the reference drift.

Throughout this subsection, we let $\mbh^{(\ell)}$ denote the degree $\ell$ correction obtained from the weighted Helmholtz decomposition applied to the polynomial vector field computed  from the \emph{exact} Hermite coefficients of the residual. 
In the conjugate setting of \Cref{app:subsec:optimality-conj}, the residual vector field is affine, so the exact Hermite coefficients can be computed using Gauss--Hermite quadrature with $N_{\mathrm{quad}} = 2$ nodes per latent dimension.
In the nonconjugate setting of \Cref{app:subsec:optimality-nonconj}, this is generally not the case; our analysis does not include coefficient estimation error from Monte Carlo or numerical quadrature.

\subsubsection{Conjugate setting}
\label{app:subsec:optimality-conj}

\begin{proof}[Proof of \Cref{thm:helm-optimal}]
Suppose that the prior SDE is linear with a nondegenerate Gaussian initial law and that the observation model is nondegenerate linear-Gaussian.
The posterior path law $\bbQ^\star$ is then Gauss--Markov, its drift $\mbf^\star(\mbx,t)$ is affine in $\mbx$, and its one-time marginals have the form $q^\star(\mbx,t)=\cN(\mbm^\star(t),\mbS^\star(t))$ \citep{kalman1961new,sarkka2023bayesian}.
The covariance $\mbS^\star(t)$ is positive definite.
The functions $\mbm^\star(t)$ and $\mbS^\star(t)$ are continuous and piecewise continuously differentiable, with possible derivative jumps at the observation times \citep[Appendix J]{hu2025sing}.

For any continuously differentiable mean and covariance paths $(\mbm(t),\mbS(t))_{0\leq t\leq T}$, Helmholtz-SDE uses the affine reference drift from \Cref{app:prop:gaussian-compatible-drift} and adds the degree-$1$ correction obtained from \Cref{app:prop:hermite-projection}.
Because the prior and reference drifts are affine, the residual vector field is affine.
Moreover, the identity \cref{app:eq:hermite-raising} shows that the divergence-free projection preserves the Hermite degree.
Consequently, the projection onto the subspace $\cV_{t,1}$ of $q$-divergence-free vector fields that are polynomials of degree at most one is equal to the projection onto the subspace $\cV_t$ of all $q$-divergence-free vector fields.
Thus, at any prescribed Gaussian marginals, Helmholtz-SDE selects the drift that minimizes the path-space KL divergence to the prior among all admissible drifts having those marginals.

The paths $\mbm^\star$ and $\mbS^\star$ need not be continuously differentiable at the observation times, so $\bbQ^\star$ need not itself belong to the variational family $\cQ_{(1)}$.
We instead argue that $\bbQ^\star$ can be approximated arbitrarily closely in KL by $\cQ_{(1)}$.
Fix $\delta>0$, and let $U_\delta\subset[0,T]$ be a union of neighborhoods of the observation times with Lebesgue measure at most $\delta$.
The mean and covariance paths can be smoothed inside $U_\delta$ to obtain continuously differentiable functions $\mbm^\delta$ and $\mbS^\delta$ that agree with $\mbm^\star$ and $\mbS^\star$ outside $U_\delta$ and converge uniformly to them as $\delta \to 0$.
The smoothing may be chosen so that the first derivatives remain uniformly bounded in $\delta$.
For sufficiently small $\delta$, $\mbS^\delta(t)$ remains positive definite because $\lambda_{\min}(\mbS^\star(t))$ is bounded away from zero on $[0,T]$.

Let $\bbQ^\delta$ be the path law produced by Helmholtz-SDE with the degree-$1$ correction applied to $\mbm^\delta$ and $\mbS^\delta$.
Outside $U_\delta$, these marginal paths and their derivatives agree with those of $\bbQ^\star$.
The corresponding Helmholtz-SDE drift therefore minimizes the fixed-marginal path-space KL objective at $q^\star$.
The posterior drift $\mbf^\star$ minimizes the same objective, since replacing it by a drift with the same marginals would leave the expected log likelihood and the initial KL term unchanged while changing only the path-space KL term.
Strict convexity of the fixed-marginal objective gives
\begin{equation*}
\mbf_q^\delta(\mbx,t)+\mbh^{(1),\delta}(\mbx,t)=\mbf^\star(\mbx,t), \qquad t\notin U_\delta,
\end{equation*}
up to $q^\star(\mbx,t)$-almost-everywhere equality.

The drifts $\mbf_q^\delta + \mbh^{(1),\delta}$ and $\mbf^\star$ are affine in $\mbx$.
The uniform bounds on the smoothed marginal paths and their derivatives imply that their affine coefficients are uniformly bounded on $U_\delta$.
The Gaussian marginals of $\bbQ^\delta$ also have uniformly bounded second moments.
Since $\mbSigma(t)^{-1}$ is bounded on $[0,T]$, there is a constant $C<\infty$, independent of $\delta$, such that
\begin{equation*}
\bbE_{\bbQ^\delta}\left\|\mbSigma(t)^{-1/2}\left\{\mbf_q^\delta(\mbx(t),t)+\mbh^{(1),\delta}(\mbx(t),t)-\mbf^\star(\mbx(t),t)\right\}\right\|_2^2\leq C, \qquad t\in U_\delta.
\end{equation*}

Applying the path-space KL identity to $\bbQ^\delta$ and $\bbQ^\star$ gives
{\small
\begin{equation*}
\begin{aligned}
&\KL{\bbQ^\delta}{\bbQ^\star}\\
&=\KL{q^\delta(\mbx,0)}{q^\star(\mbx,0)}+\frac{1}{2}\int_{U_\delta}\bbE_{\bbQ^\delta}\left\|\mbSigma(t)^{-1/2}\left\{\mbf_q^\delta(\mbx(t),t)+\mbh^{(1),\delta}(\mbx(t),t)-\mbf^\star(\mbx(t),t)\right\}\right\|_2^2dt\\
&\leq\KL{q^\delta(\mbx,0)}{q^\star(\mbx,0)}+\frac{C\delta}{2}.
\end{aligned}
\end{equation*}
}
The Gaussian initial KL converges to zero by uniform convergence of the mean and covariance paths.
Hence $\KL{\bbQ^\delta}{\bbQ^\star}\to0$ as $\delta\to0$, and therefore $\inf_{\bbQ\in\cQ_{(1)}}\KL{\bbQ}{\bbQ^\star}=0$.
\end{proof}

\subsubsection{Nonconjugate setting}
\label{app:subsec:optimality-nonconj}

In the nonconjugate setting, the posterior generally does not have Gaussian one-time marginals.
Simulation-free variational inference algorithms that impose Gaussian one-time marginals therefore share a common, irreducible approximation error. 
Helmholtz-SDE further restricts the variational family by considering a subset of divergence-free vector fields.

As in \Cref{app:thm:hilbert-projection}, we consider the Hilbert space $\cL^2(q(\mbx, t), \bbR^K, \mbSigma(t))^K$ of vector fields satisfying $\bbE_{q(\mbx, t)}\|\mbf(\mbx)\|^2 < \infty$ equipped with the inner product defined in \cref{app:eqn:l2-inner-prod}.
Let $\cV_t$ be the closed subspace of weakly $q$-divergence-free vector fields, let $\cP_\ell$ be the subspace of polynomial vector fields of degree at most $\ell$, and let $\cV_{t,\ell}=\cV_t\cap\cP_\ell$.

Let $\Pi_\cS$ be the orthogonal projection onto the subspace $\cS$.
Then the degree $\ell$ correction is $\mbh^{(\ell)}(\mbx,t)= [\Pi_{\cV_{t,\ell}}\mbr(\cdot,t)](\mbx)$.
First, we quantify the reduction in negative ELBO at \emph{fixed marginals} from increasing the degree of the polynomial correction.
\begin{proposition}[Improvement at fixed marginals]
\label{app:prop:polynomial-suboptimality}
Suppose \Cref{app:assum:resid} holds, and fix one-time marginals $\{q(\mbx, t)\}_{0 \leq t \leq T}$.
Let $\ell \geq 0$ and $k \geq 1$ be integers for which the degree $\ell$ and $\ell+k$ corrections induce admissible path laws.

The reduction in negative ELBO from using the degree $\ell+k$ correction in place of the degree $\ell$ correction is
\begin{equation}\label{app:eq:poly-suboptimality}
\begin{aligned}
\Delta_{\ell,\ell+k}(q)&=\frac{1}{2}\int_0^T \left\{\|\mbr(\mbx,t)-\mbh^{(\ell)}(\mbx,t)\|_{q,t}^2-\|\mbr(\mbx,t)-\mbh^{(\ell+k)}(\mbx,t)\|_{q,t}^2\right\}dt\\
&=\frac{1}{2}\int_0^T\|\mbh^{(\ell+k)}(\mbx,t)-\mbh^{(\ell)}(\mbx,t)\|_{q,t}^2dt.
\end{aligned}
\end{equation}
Equivalently,
\begin{equation}\label{app:eq:degree-est}
\Delta_{\ell,\ell+k}(q)=\frac{1}{2}\sum_{\ell<|\mbalpha|\leq\ell+k}\int_0^T\widetilde{\mbc}_{\mbalpha}(t)^{\top}\mbLambda(t)^{-1}\widetilde{\mbc}_{\mbalpha}(t)dt.
\end{equation}
$\widetilde{\mbc}_{\mbalpha}(t)$ are the coefficients of the $q$-divergence-free vector field in the Hermite basis obtained from \Cref{app:prop:hermite-projection}.
\end{proposition}

\begin{proof}
Since the path laws obtained from $\mbh^{(\ell)}$ and $\mbh^{(\ell+k)}$ possess one-time marginals $\{q(\mbx, t)\}_{0 \leq t \leq T}$, their reconstruction losses and initial KL divergences are equal.
Therefore, the reduction in negative ELBO is the difference between the path-space KL divergences to the prior.
The second equality in \cref{app:eq:poly-suboptimality} follows from the Pythagorean identity for projections onto the nested subspaces $\cV_{t,\ell}\subseteq\cV_{t,\ell+k}$.

The identity \cref{app:eq:hermite-raising} maps each Hermite degree to the next degree and does not couple different degrees.
Consequently, $\cV_t$ is the orthogonal direct sum of its Hermite-degree components, and the degree $\ell$ and $\ell+k$ corrections have identical Hermite coefficients through degree $\ell$.
Orthonormality of the Hermite basis gives \cref{app:eq:degree-est}.
\end{proof}

\Cref{app:prop:polynomial-suboptimality} compares polynomial degrees at fixed marginals. 
We now consider optimizing within each variational family.
Define $\cQ_{(\ell)}$ to be the variational family obtained by applying Helmholtz-SDE with the degree $\ell$ correction to the Gaussian one-time marginals whose mean and covariance satisfy \Cref{app:assum:resid}.
From \cref{app:eq:degree-est}, we obtain a lower bound on the improvement from using $\cQ_{(\ell + k)}$ over $\cQ_{(\ell)}$.
\begin{proposition}[Suboptimality of variational families]
\label{app:prop:error-nonlinear}
Let $\ell \geq 0$ and $k \geq 1$ be integers for which the degree $\ell$ and $\ell+k$ corrections induce admissible path laws.
For any $\epsilon>0$, suppose that $\bbQ_\epsilon\in\cQ_{(\ell)}$ has marginals $q_\epsilon(\mbx,t)$ and satisfies
\begin{equation*}
\KL{\bbQ_\epsilon}{\bbQ^\star}\leq\inf_{\bbQ\in\cQ_{(\ell)}}\KL{\bbQ}{\bbQ^\star}+\epsilon.
\end{equation*}
Let $\mbh_\epsilon^{(\ell)}$ and $\mbh_\epsilon^{(\ell+k)}$ be the degree $\ell$ and $\ell + k$  corrections computed at $q_\epsilon(\mbx,t)$.
Then
\begin{equation}\label{app:eq:optimality-bound}
\begin{aligned}
&\inf_{\bbQ\in\cQ_{(\ell)}}\KL{\bbQ}{\bbQ^\star}-\inf_{\bbQ\in\cQ_{(\ell + k)}}\KL{\bbQ}{\bbQ^\star}\\
&\geq\frac{1}{2}\int_0^T\|\mbh_\epsilon^{(\ell + k)}(\mbx,t) - \mbh_\epsilon^{(\ell)}(\mbx,t)\|_{q_\epsilon,t}^2dt-\epsilon.
\end{aligned}
\end{equation}
\end{proposition}

\begin{proof}
Let $\bbQ_\epsilon^{(\ell + k)}$ be the path law obtained by adding $\mbh_\epsilon^{(\ell + k)}$ to the reference drift of $\bbQ_\epsilon$.
Then $\bbQ_\epsilon^{(\ell + k)}$ and $\bbQ_\epsilon$ have the same one-time marginals because $\mbh_\epsilon^{(\ell + k)}$ is $q_\epsilon$-divergence-free.
Their reconstruction losses and initial KL divergences are equal.
The difference between their KL divergences to $\bbQ^\star$ is consequently the difference between their path-space KL divergences to the prior.
By \Cref{app:prop:polynomial-suboptimality} and the $\epsilon$-optimality of $\bbQ_\epsilon$,
\begin{equation*}
\begin{aligned}
&\inf_{\bbQ\in\cQ_{(\ell)}}\KL{\bbQ}{\bbQ^\star}-\inf_{\bbQ\in\cQ_{(\ell + k)}}\KL{\bbQ}{\bbQ^\star}\\
&\geq\KL{\bbQ_\epsilon}{\bbQ^\star}-\KL{\bbQ_\epsilon^{(\ell + k)}}{\bbQ^\star}-\epsilon\\
&= \Delta_{\ell,\ell+k}(q_\epsilon) -\epsilon =\frac{1}{2}\int_0^T\|\mbh_\epsilon^{(\ell + k)}(\mbx,t) - \mbh_\epsilon^{(\ell)}(\mbx,t)\|_{q_\epsilon,t}^2dt-\epsilon.
\end{aligned}
\end{equation*}
\end{proof}
In particular, by choosing $\ell = 0$, \Cref{app:prop:error-nonlinear} provides a lower bound on the improvement over the variational family defined by the reference drift (e.g., SVISE \cite{course2023state,course2023amortized} and SDE Matching \cite{bartosh2025sde}).

\Cref{app:prop:error-nonlinear} provides a heuristic for estimating the improvement of $\cQ_{(\ell + k)}$ over $\cQ_{(\ell)}$.
After fitting Helmholtz-SDE with a degree $\ell$ correction, we hold the fitted marginals and generative model parameters fixed, compute the Hermite coefficients through degree $\ell+k$, and evaluate \cref{app:eq:optimality-bound}.
We use the resulting
quantity as an estimate of the improvement, taking the unknown
optimization error $\epsilon$ to be zero.
With an optimal degree $\ell$ fit and exact coefficient evaluation, this quantity is a lower bound on the improvement between the  variational families.
We validate this heuristic for the stochastic Lorenz attractor in \Cref{app:subsec:lorenz}.

\subsection{Algorithm summary}
\label{app:subsec:alg-summary}

We summarize the Hermite implementation of the polynomial correction in \Cref{app:alg:helm_correction} and the resulting Helmholtz-SDE training procedure in \Cref{app:alg:helmholtz_sde}.

The coefficient samples used by \texttt{HelmCorrection} are independent of the state samples used to estimate the path-space KL.
The coefficient-estimation configuration $\cE$ specifies the integration rule (Monte Carlo or Gauss--Hermite quadrature), the sample count or number of quadrature nodes per dimension, and the Stein derivative order (zero for the direct estimator). 
The corresponding coefficient estimators are described in \Cref{app:subsec:resid-approx}.

\clearpage
\begin{algorithm}[H]
\caption{\texttt{HelmCorrection}: Hermite approximation to the divergence-free correction}
\label{app:alg:helm_correction}
\Require{polynomial order $\ell$; prior drift $\mbf_p$; reference drift $\mbf_q$; diffusion covariance $\mbSigma$; posterior mean $\mbm$ and covariance $\mbS$ at time $t$; coefficient-estimation configuration $\cE$}

\medskip
\tcp{Whiten and rotate}
\AlgStep{Compute the symmetric square root $\mbS^{1/2}$ and its inverse $\mbS^{-1/2}$}

\smallskip
\AlgStep{Form $\mbC\leftarrow\mbS^{-1/2}\mbSigma(t)\mbS^{-1/2}$ and compute $\mbC=\mbQ\mbLambda\mbQ^\top$ with $\mbLambda=\diag(\lambda_1,\ldots,\lambda_K)$}

\medskip
\tcp{Estimate the Hermite coefficients of the rotated residual}
\AlgStep{Define $\widetilde{\mbrho}(\mbv,t)\leftarrow\mbQ^\top\mbS^{-1/2}\left\{\mbf_p(\mbm+\mbS^{1/2}\mbQ\mbv,t)-\mbf_q(\mbm+\mbS^{1/2}\mbQ\mbv,t)\right\}$}

\smallskip
\For{$\mbalpha\in\cN_{K,\ell}$ and $i=1,\ldots,K$}{
    \AlgStep{Estimate $\widehat{\widetilde{\mbb}}_{\mbalpha,i}$ using the coefficient estimator specified by $\cE$ (\Cref{app:subsec:resid-approx})}
}

\medskip
\tcp{Project coefficient by coefficient onto the divergence-free subspace}
\For{$\mbbeta\in\cN_{K,\ell+1}\setminus\cN_{K,0}$}{
    \AlgStep{$\widehat\mu_{\mbbeta}\leftarrow\left\{\sum_{i=1}^K\sqrt{\mbbeta_i}\widehat{\widetilde{\mbb}}_{\mbbeta-\mbe_i,i}\right\}\big/\left\{\sum_{i=1}^K\lambda_i\mbbeta_i\right\}$}
}

\smallskip
\For{$\mbalpha\in\cN_{K,\ell}$ and $i=1,\ldots,K$}{
    \AlgStep{$\widehat{\widetilde{\mbc}}_{\mbalpha,i}\leftarrow\widehat{\widetilde{\mbb}}_{\mbalpha,i}-\lambda_i\sqrt{\mbalpha_i+1}\widehat\mu_{\mbalpha+\mbe_i}$}
}

\medskip
\tcp{Return the correction in the original coordinates}
\AlgStep{Define $\widehat{\mbnu}^{(\ell)}(\mbv,t)\leftarrow\sum_{\mbalpha\in\cN_{K,\ell}}\widehat{\widetilde{\mbc}}_{\mbalpha}\hat{H}_{\mbalpha}(\mbv)$}

\smallskip
\AlgStep{Define $\widehat{\mbh}^{(\ell)}(\mbx,t)\leftarrow\mbS^{1/2}\mbQ\widehat{\mbnu}^{(\ell)}\left(\mbQ^\top\mbS^{-1/2}(\mbx-\mbm),t\right)$}

\smallskip
\Return $\widehat{\mbh}^{(\ell)}(\mbx,t)$
\end{algorithm}

\begin{algorithm}[H]
\caption{Helmholtz-SDE training}
\label{app:alg:helmholtz_sde}
\Require{dataset $\cD=\{\mby^{(m)}(t_n)\}_{n=1,m=1}^{N,M}$; prior drift $\mbf_p(\mbx,t\mid\mbtheta)$; diffusion covariance $\mbSigma(t\mid\mbtheta)$; observation model $p(\mby(t_n)\mid\mbx(t_n),\mbtheta)$; model parameters $\mbtheta$; parameters $\mbphi$ of the posterior marginal paths $(\mbm^{(m)}(t\mid\mbphi),\mbS^{(m)}(t\mid\mbphi))_{0\leq t\leq T}$; reference drift $\mbf_q$; polynomial order $\ell$; sample sizes $N_{\mathrm{time}}$ and $N_{\mathrm{space}}$; coefficient-estimation configuration $\mathcal{E}$}

\For{learning iterations}{
    \tcp{Compute the initial KL term}
    \AlgStep{$\cL_{\mathrm{init}}\leftarrow\sum_{m=1}^M\KL{\cN(\mbm^{(m)}(0\mid\mbphi),\mbS^{(m)}(0\mid\mbphi))}{\cN(\mbmu_0,\mbV_0)}$}

    \medskip
    \tcp{Estimate the reconstruction term}
    \AlgStep{Sample $s_j\overset{\mathrm{i.i.d.}}{\sim}\mathrm{Unif}\{t_1,\ldots,t_N\}$ for $j=1,\ldots,N_{\mathrm{time}}$}

    \smallskip
    \AlgStep{Sample $\mbx_{m,j,k}\overset{\mathrm{i.i.d.}}{\sim}\cN(\mbm^{(m)}(s_j\mid\mbphi),\mbS^{(m)}(s_j\mid\mbphi))$ for $k=1,\ldots,N_{\mathrm{space}}$}

    \smallskip
    \AlgStep{$\widehat{\cL}_{\mathrm{rec}}\leftarrow-\frac{N}{N_{\mathrm{time}}N_{\mathrm{space}}}\sum_{m=1}^M\sum_{j=1}^{N_{\mathrm{time}}}\sum_{k=1}^{N_{\mathrm{space}}}\log p(\mby^{(m)}(s_j)\mid\mbx_{m,j,k},\mbtheta)$}

    \medskip
    \tcp{Estimate the path-space KL term}
    \AlgStep{Sample $s_j\overset{\mathrm{i.i.d.}}{\sim}\mathrm{Unif}[0,T]$ for $j=1,\ldots,N_{\mathrm{time}}$}

    \smallskip
    \For{$m=1,\ldots,M$ and $j=1,\ldots,N_{\mathrm{time}}$}{
        \AlgStep{$\widehat{\mbh}_{m,j}^{(\ell)}(\cdot)\leftarrow\text{\texttt{HelmCorrection}}(\ell,\mbf_p,\mbf_q,\mbSigma,\mbm^{(m)}(s_j\mid\mbphi),\mbS^{(m)}(s_j\mid\mbphi),s_j,\mathcal{E})$}
    }

    \smallskip
    \AlgStep{Independently sample $\mbx_{m,j,k}\overset{\mathrm{i.i.d.}}{\sim}\cN(\mbm^{(m)}(s_j\mid\mbphi),\mbS^{(m)}(s_j\mid\mbphi))$ for $k=1,\ldots,N_{\mathrm{space}}$}

    \smallskip
    \AlgStep{$\widehat{\cL}_{\mathrm{path}}\leftarrow\frac{T}{2N_{\mathrm{time}}N_{\mathrm{space}}}\sum_{m=1}^M\sum_{j=1}^{N_{\mathrm{time}}}\sum_{k=1}^{N_{\mathrm{space}}}$ $\left\|\mbSigma(s_j\mid\mbtheta)^{-1/2}\left\{\mbf_p(\mbx_{m,j,k},s_j\mid\mbtheta)-\mbf_q(\mbx_{m,j,k},s_j)-\widehat{\mbh}_{m,j}^{(\ell)}(\mbx_{m,j,k})\right\}\right\|_2^2$}

    \medskip
    \tcp{Update the model and variational parameters}
    \AlgStep{$\cL\leftarrow\widehat{\cL}_{\mathrm{rec}}+\cL_{\mathrm{init}}+\widehat{\cL}_{\mathrm{path}}$}

    \smallskip
    \AlgStep{$\mbtheta,\mbphi\leftarrow\texttt{gradient-step}(\cL,\mbtheta,\mbphi)$}
}
\end{algorithm}
\clearpage
\section{Linear SDE experiments: the Ornstein--Uhlenbeck spiral}\label{app:sec:ou-spiral}
In this section, we study the linear SDE prior
\begin{equation}\label{app:eqn:ou-spiral}
    \bbP: d\mbx(t) = (-\alpha \bbI_2 + \omega \mbJ) \mbx(t) dt + d\mbw(t), \quad 0 \leq t \leq T, \quad p(\mbx, 0) = \cN(\mbmu(0), \mbV(0)),
\end{equation}
and the linear, Gaussian observation model $p(\mby(t_i)| \mbx(t_i)) = \cN(\mbC \mbx(t_i) + \mbd, \mbR)$. We refer to this model as the \emph{Ornstein--Uhlenbeck (OU) spiral}. 

The prior SDE $\bbP$ has stationary distribution $\cN(\mb{0}, 1/(2\alpha)\bbI_2)$.
Moreover, the drift function of $\bbP$ is the sum of a gradient field $-\alpha \bbI_2\mbx$ as well as a vector field $\omega\mbJ\mbx$ that is divergence-free with respect to the stationary distribution.

In the observation model, $\mbC \in \bbR^{2 \times 2}$ is a linear transformation having determinant $1$, $\mbd \in \bbR^2$ is an offset vector, and $\mbR = \sigma^2 \mbC \mbC^\top$ is the observation noise, where $\sigma^2 > 0$.

First, in \Cref{app:subsec:ou-theory} we demonstrate that neither SVISE \cite{course2023state, course2023amortized} nor SDE-Matching \cite{bartosh2025sde} can be expected to recover the true posterior in this model, even though it is tractable. 
Then, in \Cref{app:subsec:ou-spiral-supplement} we provide additional simulation results that validate our theoretical claim.  

\subsection{Theoretical results for OU spiral}\label{app:subsec:ou-theory}

In this subsection, we theoretically characterize the suboptimality of the symmetric and square root gauges when the prior and likelihood models are known and fixed:
\begin{equation}\label{app:eqn:kl-gap}
    d_{\cQ}^\star = \inf_{\bbQ \in \cQ} \KL{\bbQ}{\bbQ^\star},
\end{equation}
where $\bbQ^\star$ is the true posterior for the latent SDE model (\Cref{subsec:variational_inference}).
In the OU spiral model, since the prior SDE is linear and the likelihood model is linear and Gaussian, $\bbQ^\star$ is a Gaussian process.

Even though exact inference is possible, we show that \emph{existing simulation-free VI algorithms cannot be expected to recover the posterior}. 
Specifically, consider  prior initial distribution $p(\mbx, 0)$ equal to the stationary distribution $\cN(\mb{0}, 1/(2\alpha)\bbI_2)$ and all continuously differentiable marginal parameters  $(\mbm(t), \mbS(t))_{0 \leq t \leq T}$ with $\mbS(t)$ positive definite.
We show that $d_{\cQ}^\star$ grows at least linearly with $\omega$ for the corresponding variational families defined by the symmetric and square root gauges.
We denote them by $\cQ_{\mathrm{sym}}$ and $\cQ_{\mathrm{sqrt}}$, respectively.

\begin{theorem}[Suboptimality of simulation-free VI algorithms in the OU spiral]
\label{thm:suboptimality-formal}
Consider the OU spiral model \cref{app:eqn:ou-spiral}, and let $\bbQ^\star$ denote the (true) posterior SDE.  

Fix any $c_0 > 1$ and $\omega \geq c_0 \alpha$. Then every distribution $\bbQ \in \cQ_{\mathrm{sym}} \cup \cQ_{\mathrm{sqrt}}$ satisfies
\begin{equation*}
    \KL{\bbQ}{\bbQ^\star}
    \geq
     \min\left\{
        \frac{1}{8}, \  
        \frac{1-\frac{1}{2c_0}}{2\sqrt{2}}
    \right\}
    T \omega - \frac{1}{2\sqrt{2}} + C_{\mathrm{OU}}
\end{equation*}
where 
\begin{equation}\label{eqn:constant}
    C_{\mathrm{OU}}(\omega, \alpha, \mbC, \mbd, \sigma^2, T, \mby) = \log p(\mby) + N \log(2 \pi \sigma^2).
\end{equation}

\smallskip
Moreover, if $p(\mbx, 0) = \cN(\mb{0}, \frac{1}{2\alpha} \bbI_2)$, then $C_{\mathrm{OU}}$ is lower bounded by a constant $C_{\mathrm{OU-stationary}}$ that does not depend on $\omega$.
\end{theorem}

\begin{proof}
Our proof consists of five key steps:
\begin{enumerate}[label=(\roman*)]
    \item\label{item:decomp} Express the KL divergence to $\bbQ^\star$ at any pair $(\mbm(t), \mbS(t))_{0 \leq t \leq T}$ as a sum of two functions $f_{\mbm}$ and $f_{\mbS}$, where $f_{\mbm}$ is a function of the marginal means $(\mbm(t))_{0 \leq t \leq T}$ and $f_{\mbS}$ is a function of the marginal covariances $(\mbS(t))_{0 \leq t \leq T}$.
    
    \item\label{item:mean-path} Justify that $f_{\mbm}$ is minimized at the posterior mean $(\mbm^\star(t))_{0 \leq t \leq T}$.

    \item\label{item:eigen} Rewrite $f_{\mbS}$ in terms of the eigendecomposition of $\mbS(t)$ at each time $0 \leq t \leq T$.

    \item\label{item:occupation-time} Use an occupation-time argument to show that either $\lambda_1(t)$ is often bounded away from zero, in which case the rotational mismatch term is large, or else $\lambda_1(t)$ is often close to zero, in which case keeping it small requires a large negative control and hence large control cost.

    \item \label{item:indpendence} Show that when $p(\mbx, 0)$ is the stationary distribution of the prior \eqref{app:eqn:ou-spiral}, $\log p(\mby)$ is lower bounded by a constant not depending on $\omega$. 
    This is done using uniform (in $\omega$) spectral bounds on the observation covariance matrix.
\end{enumerate}

\textbf{Step \ref{item:decomp}}:
Let $\bbQ$ be an SDE with diffusion coefficient $\bbI_2$ and Gaussian marginals $q(\mbx, t) = \cN(\mbm(t), \mbS(t))$, where $(\mbm(t))_{0 \leq t \leq T}$ is a differentiable mean function and $(\mbS(t))_{0 \leq t \leq T}$ is a differentiable and symmetric, positive definite covariance function. 
Moreover, suppose that $\bbQ$ has drift function of the form 
\begin{equation*}
    \begin{aligned}
    \mbf_q(\mbx, t) = \mbA_q(t) \mbx + \frac{d}{dt} \mbm(t) - \mbA_q(t) \mbm(t), \quad 
    \mbA_q(t) = \mbK(t) - \frac{1}{2}\mbS(t)^{-1},
    \end{aligned}
\end{equation*}
where $\mbK(t)$ is a solution to the Lyapunov equation \cref{app:eq:gaussian-lyapunov}. In \Cref{app:prop:gaussian-compatible-drift}, we showed that this drift function induces the marginals $q(\mbx, t)$.

Using Girsanov's Theorem (\Cref{app:thm:girsanov}), the KL divergence to the posterior decomposes as
\begin{equation*}
    \KL{\bbQ}{\bbQ^\star} = f_{\mbm}(\mbm) + f_{\mbS}(\mbS) + C_{\mathrm{OU}}(\omega, \alpha, \mbC, \mbd, \sigma^2, T, \mby),
\end{equation*}
where $f_{\mbm}$ depends only on the mean path, $f_{\mbS}$ depends only on the covariance path, and $C_{\mathrm{OU}}$ is a constant defined in \cref{eqn:constant} not depending on $\bbQ$.

These functions are given by
\begin{align}\label{app:eqn:mean-cov-functions}
    &f_{\mbm}(\mbm) = \frac{1}{2}(\mbm(0) - \mbmu(0))^\top\mbV(0)^{-1}(\mbm(0) - \mbmu(0))  + \frac{1}{2} \int_0^{T} \left\| \frac{d}{dt}\mbm(t) - (-\alpha \bbI_2 + \omega \mbJ)\mbm(t) \right\|_2^2 dt \nonumber\\
    &+ \frac{1}{2} \sum_{i=1}^N\big(\mbC \mbm(t_i) + \mbd - \mby(t_i)\big)^\top \mbR^{-1} \big(\mbC \mbm(t_i) + \mbd - \mby(t_i)\big), \nonumber\\
    &f_{\mbS}(\mbS) = \frac{1}{2} \left[ \mathrm{tr}\big(\mbV(0)^{-1} \mbS(0)\big) - \log \frac{\det\big(\mbS(0)\big)}{\det\big(\mbV(0)\big)} - 2\right] + \nonumber\\
    &\frac{1}{2} \int_0^{T} \nonumber \mathrm{tr}\left(\big((\mbK(t) - \frac{1}{2} \mbS(t)^{-1} )-(-\alpha \bbI_2 + \omega \mbJ)\big) \mbS(t)\big((\mbK(t) - \frac{1}{2} \mbS(t)^{-1} ) -(-\alpha \bbI_2 + \omega \mbJ)\big)^\top\right) dt \nonumber\\
    &\ + \frac{1}{2} \sum_{n=1}^N \mathrm{tr}\left(\mbR^{-1}\mbC \mbS(t_n)\mbC^\top\right).
\end{align}
Notice that $f_{\mbS}$ does not depend on the observations $\{\mby(t_n)\}_{n=1}^N$.

Since $\mbR=\sigma^2 \mbC \mbC^\top$, we have $\mbC^\top \mbR^{-1} \mbC = \sigma^{-2}\bbI_2$,
and therefore
\begin{equation*}
    \frac{1}{2} \sum_{i=1}^N \mathrm{tr}\left(\mbR^{-1}\mbC \mbS(t_i)\mbC^\top\right) = \frac{1}{2\sigma^2}\sum_{i=1}^N \mathrm{tr}\big(\mbS(t_i)\big) \geq 0.
\end{equation*}
The first term of $f_{\mbS}$ is also nonnegative since the matrix $\mbA = \mbV(0)^{-1/2}\mbS(0) \mbV(0)^{-1/2}$ is positive definite and
\begin{equation*}
    \mathrm{tr}\big(\mbV(0)^{-1} \mbS(0)\big) - \log \frac{\det\big(\mbS(0)\big)}{\det\big(\mbV(0)\big)} - 2 = \sum_{i=1}^2 \left( \lambda_i(\mbA) - \log\lambda_i(\mbA) -1\right) \geq 0.
\end{equation*}
We will drop these two terms from $f_{\mbS}$ when constructing our lower bound. 
The single remaining term is the path KL term.

\textbf{Step \ref{item:mean-path}}:
Because the posterior path law is Gaussian, the map $\mbm \mapsto f_{\mbm}(\mbm)$ is a strictly convex quadratic functional. 
Its unique minimizer is the posterior mean path $(\mbm^\star(t))_{0 \leq t \leq T}$. 
Hence
\begin{equation*}
    f_{\mbm}(\mbm) \geq f_{\mbm}(\mbm^\star),
\end{equation*}
with equality if and only if $\mbm=\mbm^\star$. 
The optimum is $\geq 0$ since each term in $f_{\mbm}$ is nonnegative.
Therefore, the suboptimality of the family is controlled entirely by the covariance term $f_{\mbS}$.

\textbf{Step \ref{item:eigen}}: 
Let $\lambda_1(t) \leq \lambda_2(t)$ denote the ordered eigenvalues of $\mbS(t)$. 
Since $\mbS(\cdot)$ is continuously differentiable and symmetric positive definite, Weyl's perturbation inequality implies that $\lambda_1,\lambda_2$ are Lipschitz, hence absolutely continuous and differentiable almost everywhere.

Define the sets
\begin{equation*}
    A =\{t \in [0,T] : \lambda_1(t) < \lambda_2(t)\},
    \qquad
    E =\{t \in [0,T] : \lambda_1(t)=\lambda_2(t)\}.
\end{equation*}
We will decompose the integral in $f_{\mbS}$ over each of these sets individually.

We will start with $A$, the more involved case. 
$A$ is an open subset of $[0, T]$, and thus can be written as a countable union of disjoint open intervals. 
On each such interval, we can choose a continuously differentiable rotation matrix $\mbU(t)$ such that
\begin{equation*}
    \mbS(t)=\mbU(t)\mbLambda(t)\mbU(t)^\top, \ \mbU(t) = \begin{pmatrix} \cos \theta(t) & -\sin \theta(t)\\ \sin \theta(t) & \cos \theta(t) \end{pmatrix}, \ \mbLambda(t)=\diag\big(\lambda_1(t),\lambda_2(t)\big).
\end{equation*}
Using this eigendecomposition, we rewrite the integrand appearing in $f_{\mbS}$ as 
\begin{equation*}
    \mathrm{tr}\left(\widetilde{\mbDelta}(t)\mbLambda(t)\widetilde{\mbDelta}(t)^\top \right), \quad  \widetilde{\mbDelta}(t) = \mbU(t)^\top \big((\mbK(t) - \frac{1}{2} \mbS(t)^{-1} )-(-\alpha \bbI_2 + \omega \mbJ)\big)\mbU(t).
\end{equation*}
Since $\mbU(t)^\top (-\alpha \bbI_2 + \omega \mbJ) \mbU(t) = -\alpha \bbI_2 + \omega \mbJ$, we need to compute $\widetilde{\mbA}_q(t) = \mbU(t)^\top(\mbK(t) - \frac{1}{2} \mbS(t)^{-1})\mbU(t)$ for the symmetric and square root gauges. 
\begin{enumerate}[label=(\alph*)]
    \item \textbf{Symmetric.}
    Define the skew-symmetric matrix
    \begin{equation*}
    \mbOmega(t) = \mbU(t)^\top \frac{d}{dt}\mbU(t) = \begin{pmatrix} 0 & -w(t) \\ w(t) & 0 \end{pmatrix}, \quad w(t) = \frac{d}{dt}\theta(t).
    \end{equation*}
    By definition of the symmetric gauge, $\mbA_q(t) = \mbK(t) - \frac{1}{2} \mbS(t)^{-1}$ is the unique symmetric solution to
    \begin{equation}\label{app:eq:symm_lyap}
        \frac{d}{dt} \mbS(t) - \bbI_2  = \mbA_q(t)\mbS(t) + \mbS(t) \mbA_q(t)^\top, \quad 0 \leq t \leq T.
    \end{equation}
    A simple algebraic calculation yields
    \begin{equation}\label{app:eq:diff-lambda}
        \mbU(t)^\top  \frac{d}{dt} \mbS(t) \mbU(t) = \frac{d}{dt} \mbLambda(t) + \mbOmega(t) \mbLambda(t) + \mbLambda(t) \mbOmega(t)^\top.
    \end{equation}
    Multiplying \cref{app:eq:symm_lyap} by $\mbU(t)^\top$ on the left and $\mbU(t)$ on the right, and then substituting \cref{app:eq:diff-lambda} into \cref{app:eq:symm_lyap} yields
    \begin{equation*}
        \frac{d}{dt} \mbLambda(t) + \mbOmega(t) \mbLambda(t) + \mbLambda(t) \mbOmega(t)^\top - \bbI_2 = \widetilde{\mbA}_q \mbLambda(t) + \mbLambda(t) \widetilde{\mbA}_q^\top, \quad 0 \leq t \leq T.
    \end{equation*}

    Hence, the entries of $\widetilde{\mbA}_q(t)$ are
    \begin{equation}\label{app:eqn:eigen-symm}
        (\widetilde{\mbA}_q(t))_{ij} = \begin{cases}
            (\frac{d}{dt}\lambda_i(t) - 1)/(2 \lambda_i(t)) & \text{if $i = j$}\\
            \mbOmega_{ij}(t) (\lambda_j(t) - \lambda_i(t)) / (\lambda_i(t) + \lambda_j(t)) & \text{if $i \neq j$}
        \end{cases}.
    \end{equation}
    \item \textbf{Square root.}
    The square root gauge sets $\mbK(t) = (\frac{d}{dt}\mbS(t)^{1/2})\mbS(t)^{-1/2}$, where $\mbS(t)^{1/2} = \mbU(t) \mbLambda(t)^{1/2}\mbU(t)^\top$ is the symmetric square root. Rewriting $\mbK(t)$ in terms of $\mbU(t)$ and $\mbLambda(t)$ yields
    \begin{equation}\label{app:eq:sqrt-solve}
    \begin{aligned}
        &\left(\frac{d}{dt}\mbS(t)^{1/2} \right)\mbS(t)^{-1/2}  = \frac{d}{dt} \mbU(t) \mbU(t)^\top  + (1/2)\mbU(t) \frac{d}{dt}\mbLambda(t)\mbLambda(t)^{-1}\mbU(t)^\top +\\
        &\mbU(t)\mbLambda(t)^{1/2}\mbOmega(t)^\top\mbLambda(t)^{-1/2}\mbU(t)^\top.
    \end{aligned}
    \end{equation}
    Multiplying \cref{app:eq:sqrt-solve} by $\mbU(t)^\top$ on the left and $\mbU(t)$ on the right yields 
    \begin{equation*}
        \widetilde{\mbA}_q(t) = \mbOmega(t) + (1/2) \frac{d}{dt}\mbLambda(t)\mbLambda(t)^{-1} - \mbLambda(t)^{1/2}\mbOmega(t) \mbLambda(t)^{-1/2} - (1/2)\mbLambda(t)^{-1}.
    \end{equation*}

      Hence, the entries of $\widetilde{\mbA}_q(t)$ are
    \begin{equation}\label{app:eqn:eigen-sqrt}
        (\widetilde{\mbA}_q(t))_{ij} = \begin{cases}
            (\frac{d}{dt}\lambda_i(t) - 1) / (2\lambda_i(t)) & \text{if $i = j$}\\
            \mbOmega_{ij}(t)(1 - \sqrt{\lambda_i(t) / \lambda_j(t)}) & \text{if $i \neq j$}
        \end{cases}.
    \end{equation}
\end{enumerate}

Overall, the integrand is equal to
\begin{equation}\label{app:eqn:integrand}
    \mathrm{tr}\left(
        \widetilde{\mbDelta}(t)\mbLambda(t)\widetilde{\mbDelta}(t)^\top
    \right) = \lambda_1(t)(\widetilde{\mbDelta}(t)_{11}^2 + \widetilde{\mbDelta}(t)_{21}^2) + \lambda_2(t)(\widetilde{\mbDelta}(t)_{12}^2 + \widetilde{\mbDelta}(t)_{22}^2),
\end{equation}
where we have now computed $\widetilde{\mbDelta}(t)$ for the symmetric and square root gauges.

Next, we explicitly minimize  \cref{app:eqn:integrand} with respect to $w(t)$, for each of the symmetric and square root gauges. 
Note that only the off-diagonal entries of $\widetilde{\mbDelta}(t)$ depend on $w(t)$. 
In particular let $\widetilde{\mbDelta}(t)_{12} = a w(t) + \omega$ and $\widetilde{\mbDelta}(t)_{21} = b w(t) - \omega$ for $a, b \in \bbR$. Then the contribution of the off-diagonal terms to the integrand can be written as 
\begin{equation*}
    \lambda_1(t)\widetilde{\mbDelta}(t)_{21}^2 + \lambda_2(t)\widetilde{\mbDelta}(t)_{12}^2  = (\lambda_1(t)b^2 + \lambda_2(t)a^2) w(t)^2 + 2 \omega (\lambda_2(t) a - \lambda_1(t) b) w(t) + (\lambda_1(t) + \lambda_2(t))\omega^2.
\end{equation*}
Completing the square yields the optimal value of $w(t)$
\begin{equation*}
    w(t)^\star = - \omega (\lambda_2(t) a - \lambda_1(t) b) / (\lambda_1(t) b^2 + \lambda_2(t) a^2).
\end{equation*}
The contribution of the off-diagonal terms at this optimum is
\begin{equation}\label{app:eqn:optimal-angle}
\begin{aligned}
    &(\lambda_1(t) + \lambda_2(t))\omega^2 - \omega^2 (\lambda_2(t) a - \lambda_1(t) b)^2 / (\lambda_1(t) b^2 + \lambda_2(t) a^2)\\
    = &\omega^2 \lambda_1(t) \lambda_2(t) (a+b)^2 / (\lambda_1(t) b^2 + \lambda_2(t) a^2).
\end{aligned}
\end{equation}
For the symmetric gauge, $a = b = - (\lambda_2(t) - \lambda_1(t)) / (\lambda_1(t) + \lambda_2(t))$, so 
\begin{equation*}
    \lambda_1(t)\widetilde{\mbDelta}(t)_{21}^2 + \lambda_2(t)\widetilde{\mbDelta}(t)_{12}^2 \geq 4 \omega^2 \lambda_1(t) \lambda_2(t) / (\lambda_1(t) +  \lambda_2(t)) \geq 2 \omega^2 \lambda_1(t). 
\end{equation*}
In the last inequality, we used the ordering of the eigenvalues $\lambda_2(t) \geq \lambda_1(t)$.

And for the square root gauge, we have $a = r(t) - 1$, $b = 1 - r(t)^{-1}$ where we have defined $r(t) = 
\sqrt{\lambda_1(t) / \lambda_2(t)}$. This implies 
\begin{equation*}
    \begin{aligned}
        &(a+b)^2 = (r(t)^2 - 1)^2 / r(t)^2, \\ &\lambda_1(t) b^2 + \lambda_2(t) a^2  = (r(t) - 1)^2 (\lambda_1(t) / r(t)^2 + \lambda_2(t)) = 2 \lambda_2(t) (r(t) - 1)^2\\
        &\lambda_1(t) \lambda_2(t) (a+b)^2 / (\lambda_1(t) b^2 + \lambda_2(t) a^2) = \frac{1}{2}\lambda_2(t) (r(t) + 1)^2 = \frac{1}{2} (\sqrt{\lambda_1(t)} + \sqrt{\lambda_2(t)})^2.
    \end{aligned}
\end{equation*}
Hence, 
\begin{equation*}
    \lambda_1(t)\widetilde{\mbDelta}(t)_{21}^2 + \lambda_2(t)\widetilde{\mbDelta}(t)_{12}^2 \geq \frac{\omega^2}{2} (\sqrt{\lambda_1(t)} + \sqrt{\lambda_2(t)})^2 \geq 2 \omega^2 \sqrt{\lambda_1(t) \lambda_2(t)} \geq 2 \omega^2 \lambda_1(t).
\end{equation*}

In both cases, we have proven that for each open interval that makes up $A$,
\begin{equation*}
    \min_{w(t)}
    \left\{
        \lambda_1(t)\widetilde{\Delta}_{21}(t)^2
        +
        \lambda_2(t)\widetilde{\Delta}_{12}(t)^2
    \right\}
    \geq 
    2 \omega^2 \lambda_1(t), \quad \text{for a.e.~$t$}.
\end{equation*}
Therefore, for almost every $t\in A$,
\begin{equation*}
    \mathrm{tr}\left(
        \widetilde{\mbDelta}(t)\mbLambda(t)\widetilde{\mbDelta}(t)^\top
    \right)\geq
    \frac{\big(\frac{d}{dt}\lambda_1(t)-1+2\alpha \lambda_1(t)\big)^2}{4\lambda_1(t)}
    +
    2\omega^2\lambda_1(t).
\end{equation*}

We now proceed to the second set $E$. On this set, the covariance is isotropic: $\mbS(t)=\lambda(t)\bbI_2, \lambda(t) =\lambda_1(t)=\lambda_2(t)$.
No eigenbasis is needed here. Since $\mbS-\lambda \bbI_2=0$ on $E$ and both terms are absolutely continuous, we have
\begin{equation*}
    \frac{d}{dt}\mbS(t)= \frac{d}{dt}\lambda(t)\bbI_2
\end{equation*}
for almost every $t\in E$. For the symmetric gauge, the linear drift coefficient $\mbA_q(t)$ solves
\begin{equation*}
    \frac{d}{dt}\lambda(t)\bbI_2 = 2\lambda(t)\mbA_q(t)+\bbI_2,
\end{equation*}
and in the symmetric square root gauge,
\begin{equation*}
    \mbA_q(t)
    =
    \left(\frac{d}{dt}\mbS(t)^{1/2} \right)\mbS(t)^{-1/2} - \frac{1}{2}\mbS(t)^{-1}
    =
    \frac{\frac{d}{dt} \lambda(t) - 1}{2\lambda(t)}\bbI_2.
\end{equation*}
Thus, in either gauge,
\begin{equation*}
   \mbA_q(t) = \frac{\frac{d}{dt} \lambda(t) - 1}{2\lambda(t)}\bbI_2
    \qquad\text{for a.e. }t\in E.
\end{equation*}
A direct calculation then yields
\begin{equation*}
\begin{aligned}
    \mathrm{tr} \left(
        \big(\mbA_q(t) - (-\alpha \bbI_2+\omega \mbJ)\big)
        \mbS(t)
        \big(\mbA_q(t) - (-\alpha \bbI_2+\omega \mbJ)\big)^\top
    \right)
    =
    \frac{\big(\frac{d}{dt} \lambda(t) -1+2\alpha \lambda(t)\big)^2}{2\lambda(t)}
    +
    2\omega^2\lambda(t),
\end{aligned}
\end{equation*}
which is stronger than the preceding bound.

Combining both cases, we conclude that for almost every $t\in[0,T]$,
\begin{equation*}
\begin{aligned}
    &\mathrm{tr}\left(
        \big(\mbA_q(t)-(-\alpha \bbI_2+\omega \mbJ)\big)
        \mbS(t)
        \big(\mbA_q(t)-(-\alpha \bbI_2+\omega \mbJ)\big)^\top
    \right)
    \geq
    \frac{u(t)^2}{4\lambda_1(t)} + 2\omega^2\lambda_1(t)\\    &u(t) = \frac{d}{dt}\lambda_1(t)-1+2\alpha \lambda_1(t).
\end{aligned}
\end{equation*}
Therefore,
\begin{equation*}
    f_{\mbS}(\mbS)
    \geq
    I
    =
    \int_0^{T}
    \left[
        \frac{u(t)^2}{8\lambda_1(t)}
        +
        \omega^2\lambda_1(t)
    \right]\dif t.
\end{equation*}

\textbf{Step \ref{item:occupation-time}}:
Fix $c_0>1$ and assume $\omega \geq c_0\alpha$. 
Define the sets $E_{\delta}$ and $F_{\delta}$ that track the time $\lambda_1(t)$ spends above and below $\delta > 0$, respectively:
\begin{equation*}
    \delta = \frac{1}{4\omega},
    \qquad
    E_\delta = \{t \in [0, T] : \lambda_1(t)\geq \delta\},
    \qquad
    F_\delta = [0,T]\setminus E_\delta.
\end{equation*}
We will argue that if either $E_{\delta}$ or $F_{\delta}$ is large, then $I$ must be large.

If $|E_\delta|\geq T/2$, then
\begin{equation*}
    I
    \geq
    \int_{E_\delta}\omega^2\lambda_1(t)\,\dif t
    \geq
    |E_\delta|\,\omega^2\delta
    \geq
    \frac{T}{8}\,\omega.
\end{equation*}

Suppose instead that $|F_\delta|> T/2$. 
Since $\lambda_1(t)<\delta$ on $F_{\delta}$, 
\begin{equation*}
    2\alpha\lambda_1(t)
    \leq
    2\alpha\delta
    =
    \frac{\alpha}{2\omega}
    \leq
    \frac{1}{2c_0}.
\end{equation*}
Define
\begin{equation*}
    a_0 = 1-\frac{1}{2c_0} > 0.
\end{equation*}
Then, for all $t\in F_\delta$,
\begin{equation*}
    \frac{d\lambda_1(t)}{dt}=1-2\alpha\lambda_1(t)+u(t)\geq a_0 + u(t).
\end{equation*}

Since $F_\delta$ is open in $[0, T]$, we can write it as a countable disjoint union of open intervals $(a_j, b_j)$.
Integrating over a single interval,
\begin{equation*}
    \lambda_1(b_j) -\lambda_1(a_j)
    =
    \int_{a_j}^{b_j}\frac{d \lambda_1(t)}{dt} dt
    \geq
    a_0(b_j-a_j)
    +
    \int_{a_j}^{b_j}u(t)\,\dif t.
\end{equation*}
Rearranging yields
\begin{equation*}
    -\int_{a_j}^{b_j}u(t)\,\dif t
    \geq
    a_0(b_j-a_j)
    -
    \big(\lambda_1(b_j)-\lambda_1(a_j)\big).
\end{equation*}
If $(a_j,b_j)$ is an interior component, maximality and continuity imply
\begin{equation*}
    \lambda_1(a_j)=\lambda_1(b_j)=\delta,
\end{equation*}
so the boundary term vanishes. If $(a_j,b_j)$ touches $0$ or $T$, then both endpoint values lie in $[0, \delta]$, and therefore
\begin{equation*}
    \big|\lambda_1(b_j)-\lambda_1(a_j)\big| \leq \delta.
\end{equation*}
Since there are at most two boundary components, summing over all components gives
\begin{equation*}
    -\int_{F_\delta}u(t)\,\dif t
    \geq
    a_0 |F_\delta| - 2\delta.
\end{equation*}

By the AM-GM inequality,
\begin{equation*}
    \frac{u^2}{8\lambda}+\omega^2\lambda
    \geq
    \frac{\omega}{\sqrt{2}}|u| \implies 
    I \geq \frac{\omega}{\sqrt{2}}\int_0^{T}|u(t)|\,\dif t.
\end{equation*}
Therefore, we deduce
\begin{equation*}
\begin{aligned}
    &\int_0^{T} |u(t)|\,\dif t
    \geq
    \int_{F_\delta}|u(t)|\,\dif t
    \geq
    -\int_{F_\delta}u(t)\,\dif t
    \geq
    a_0 |F_\delta| - 2\delta
    >
    \frac{a_0 T}{2} - 2\delta\\
    \implies&I
    \geq
    \frac{\omega}{\sqrt{2}}
    \left(
        \frac{a_0 T}{2} - 2\delta
    \right)
    =
    \frac{1-\frac{1}{2c_0}}{2\sqrt{2}} T\omega
    -
    \frac{1}{2\sqrt{2}}.
\end{aligned}
\end{equation*}

Combining the two cases and the definition of $f_{\mbS}$ in 
\cref{app:eqn:mean-cov-functions}, we conclude
\begin{equation*}
    f_{\mbS}(\mbS)
    \geq
    \min\left\{
        \frac{1}{8},\ 
        \frac{1-\frac{1}{2c_0}}{2\sqrt{2}}
    \right\}T \omega - \frac{1}{2\sqrt{2}}.
\end{equation*}

\textbf{Step \ref{item:indpendence}}:
Lastly, we show that when $p(\mbx,0)=\cN\left(\mb{0},\frac{1}{2\alpha}\bbI_2\right)$, then $C_{\mathrm{OU}}$ can be lower bounded by a constant independent of $\omega$. 

Consider the vector of stacked observations
\begin{equation*}
    \mby_{1:N} = \big(\mby(t_1)^\top, \ldots, \mby(t_N)^\top\big)^\top \in \bbR^{2N}.
\end{equation*}
Under the stationary OU spiral prior, the latent process is mean-zero Gaussian and
\begin{equation*}
    \mathrm{Cov}(\mbx(t_i),\mbx(t_j))
    =
    \frac{1}{2\alpha}
    e^{-\alpha |t_i-t_j|}
    \mbR_{\omega(t_i-t_j)},
\end{equation*}
where $\mbR_{\theta} \in \bbR^{2 \times 2}$ represents the $2 \times 2$ rotation matrix by angle $\theta$.

Therefore the stacked observation vector $\mby_{1:N}$ is Gaussian,
\begin{equation*}
    \mby_{1:N}
    \sim
    \cN\big(\bar{\mbd},\mbSigma_\omega\big),
\end{equation*}
with block covariance matrix $\mbSigma_\omega \in \bbR^{2N \times 2N}$
\begin{equation*}
    (\mbSigma_\omega)_{ij}
    =
    \frac{1}{2\alpha}
    e^{-\alpha |t_i-t_j|}
    \mbC \mbR_{\omega(t_i-t_j)} \mbC^\top
    +
    \mathbf{1}\{i=j\}\,\sigma^2 \mbC \mbC^\top,
\end{equation*}
and mean
\begin{equation*}
    \bar{\mbd}
    = \big(
        \mbd^\top,
        \ldots,
        \mbd^\top
    \big)^\top
    \in \bbR^{2N}.
\end{equation*}
We first derive an $\omega$-uniform lower spectral bound. Since $\mbSigma_\omega$ is a covariance matrix, its latent contribution is positive semidefinite. Hence,
\begin{equation*}
    \mbSigma_\omega
    \succeq
    \sigma^2 \,
    \mathrm{diag}\big(
        \mbC\mbC^\top,
        \ldots,
        \mbC\mbC^\top
    \big).
\end{equation*}
It follows that
\begin{equation*}
    \lambda_{\min}(\mbSigma_\omega)
    \geq
    \sigma^2 \lambda_{\min}(\mbC\mbC^\top)
    =
    m_* > 0,
\end{equation*}
uniformly in $\omega$.

Next we derive an $\omega$-uniform upper spectral bound. 
Let
\begin{equation*}
    \mbz
    =
    \big(
        \mbz_1^\top,
        \ldots,
        \mbz_N^\top
    \big)^\top
    \in \bbR^{2N},
    \qquad
    \mbz_i \in \bbR^2.
\end{equation*}
Then
\begin{equation*}
    \mbz^\top \mbSigma_\omega \mbz
    =
    \sigma^2
    \sum_{i=1}^N
    \mbz_i^\top \mbC\mbC^\top \mbz_i +
    \frac{1}{2\alpha}
    \sum_{i,j=1}^N
    e^{-\alpha |t_i-t_j|}
    \mbz_i^\top
    \mbC \mbR_{\omega(t_i-t_j)} \mbC^\top
    \mbz_j.
\end{equation*}
Since $\mbR_{\omega(t_i-t_j)}$ is orthogonal, $\|\mbC \mbR_{\omega(t_i-t_j)} \mbC^\top \|_{\mathrm{op}} \leq \|\mbC\|_{\mathrm{op}}^2$.
Therefore
\begin{equation*}
    \mbz^\top \mbSigma_\omega \mbz
    \leq 
    \sigma^2 \|\mbC\|_{\mathrm{op}}^2
    \sum_{i=1}^N \|\mbz_i\|_2^2  +
    \frac{\|\mbC\|_{\mathrm{op}}^2}{2\alpha}
    \sum_{i,j=1}^N
    \|\mbz_i\|_2 \|\mbz_j\|_2.
\end{equation*}
Using
\begin{equation*}
    \sum_{i,j=1}^N
    \|\mbz_i\|_2 \, \|\mbz_j\|_2
    =
    \left(
        \sum_{i=1}^N \|\mbz_i\|_2
    \right)^2
    \leq
    N \sum_{i=1}^N \|\mbz_i\|_2^2,
\end{equation*}
we obtain
\begin{equation*}
    \mbz^\top \mbSigma_\omega \mbz
    \leq
    \|\mbC\|_{\mathrm{op}}^2
    \left(
        \sigma^2 + \frac{N}{2\alpha}
    \right)
    \sum_{i=1}^N \|\mbz_i\|_2^2.
\end{equation*}
Hence
\begin{equation*}
    \lambda_{\max}(\mbSigma_\omega)
    \leq
    \|\mbC\|_{\mathrm{op}}^2
    \left(
        \sigma^2 + \frac{N}{2\alpha}
    \right)
    =
    M_*,
\end{equation*}
again uniformly in $\omega$.

Thus, for every $\omega$,
\begin{equation*}
    m_* \bbI_{2N}
    \preceq
    \mbSigma_\omega
    \preceq
    M_* \bbI_{2N}.
\end{equation*}

Since $\mby_{1:N}\sim \cN(\bar{\mbd},\mbSigma_\omega)$, its density is
\begin{equation*}
    p_\omega(\mby_{1:N})
    =
    (2\pi)^{-N}
    \det(\mbSigma_\omega)^{-1/2}
    \exp\left(
        -\frac{1}{2}
        (\mby_{1:N}-\bar{\mbd})^\top
        \mbSigma_\omega^{-1}
        (\mby_{1:N}-\bar{\mbd})
    \right).
\end{equation*}
Using the spectral bounds above,
\begin{equation*}
    (\mby_{1:N}-\bar{\mbd})^\top
    \mbSigma_\omega^{-1}
    (\mby_{1:N}-\bar{\mbd})
    \leq
    \frac{\|\mby_{1:N}-\bar{\mbd}\|_2^2}{m_*}, \quad
    \det(\mbSigma_\omega)
    \leq
    M_*^{2N}.
\end{equation*}
Therefore,
\begin{equation*}
    p_\omega(\mby_{1:N})
    \geq
    (2\pi)^{-N}
    M_*^{-N}
    \exp\left(
        -\frac{\|\mby_{1:N}-\bar{\mbd}\|_2^2}{2m_*}
    \right).
\end{equation*}
Taking logarithms yields
\begin{equation*}
    \log p_\omega(\mby_{1:N})
    \geq
    -N\log(2\pi)
    -
    N\log M_*
    -
    \frac{\|\mby_{1:N}-\bar{\mbd}\|_2^2}{2m_*}.
\end{equation*}
Thus the claimed bound holds with
\begin{equation*}
    C_{\mathrm{OU-stationary}} = N\log\sigma^2 - N\log M_* - \frac{\|\mby_{1:N}-\bar{\mbd}\|_2^2}{2m_*},
\end{equation*}
which is independent of $\omega$.
\end{proof}

\subsection{Simulation experiment details}\label{app:subsec:ou-spiral-supplement}

In this subsection, we provide details regarding our OU spiral experiments. 
These experiments support the theoretical results in \Cref{app:subsec:optimality} and \Cref{app:subsec:ou-theory}.
In summary, these results state that Helmholtz-SDE achieves zero KL gap to the true posterior for all rotation speeds $\omega$, while the symmetric \cite{course2023state, course2023amortized} and square root \cite{bartosh2025sde} gauges incur a KL gap that grows at least linearly in $\omega$. 
These results characterize the variational families themselves, not the optimization problem solved in practice: with finite-capacity parameterizations of $\mbm(t)$ and $\mbS(t)$, and with stochastic gradient-based optimization, neither family attains its theoretical optimum. 
Our experiments test whether the theoretical separation between Helmholtz-SDE and the gauges translates into practical differences.

For both our inference and learning experiments, we use the Adam optimizer \cite{kingma2014adam} with initial learning rate $10^{-3}$, final learning rate $10^{-4}$, and exponential learning rate decay. We run experiments on an NVIDIA A100 or H100 GPU.

We generate the OU dataset by sampling latents according to the initial distribution $p(\mbx, 0) = \cN(\mb{0}, 1/(2\alpha) \bbI_2)$ and prior dynamics \cref{app:eqn:ou-spiral} on $[0, 5]$. 
For all experiments we fix $\alpha = 0.2$. 
We sample $(\mbx^{(m)}(t))_{0 \leq t \leq 5}$ using Euler--Maruyama with stepsize $5 \cdot 10^{-4}$. 
To generate observations, we select random times $\{t_n^{(m)}\}_{n=1}^N$ on the interval $[0, 5]$ for each trial $m$, and then we draw $\mby^{(m)}(t_n^{(m)}) \sim \cN(\mbC_{\mathrm{true}}  \mbx^{(m)}(t_n^{(m)}) + \mbd_{\mathrm{true}}, \sigma^2 \mbC_{\mathrm{true}}\mbC_{\mathrm{true}}^\top)$. $\mbC_{\mathrm{true}} \in \bbR^{2 \times 2}$ and  $\mbd_{\mathrm{true}} \in \bbR^2$ are chosen to have standard normal entries, with $\mbC_{\mathrm{true}}$ normalized to have determinant $1$.
In our experiments, we vary the rotation speed $\omega$, the number of observations $N$, and the observation noise $\sigma > 0$.

\paragraph{Inference.}
We first investigate the gap in expressivity of the variational family when the parameters of the generative model are held fixed at their true values. 

We do so by performing inference over $16$ independent trials. 
For each trial, we can compute the true posterior by running the Kalman--Bucy smoother. 
We do so with the SING codebase, which \citet{hu2025sing} prove converges to the true posterior in a single step in the conjugate setting. 
We fix the time discretization of the true posterior to stepsize $5 \cdot 10^{-4}$.

To represent the variational posterior, we discretize $[0, 5]$ on a grid with constant separation $5 \cdot 10^{-3}$ ($1001$ total grid points). 
At each point we parameterize the posterior mean and covariance, and we interpolate linearly between the grid times. 
This yields approximately $6 \cdot 10^3$ variational parameters to learn, per trial. 
Moreover, note that for a randomly sampled time, the unbiased estimate of the path-space KL depends only on the two neighboring grid points. 
For this reason, the variational posterior is slow to converge using a single Monte Carlo sample. 
Instead, we use $N_{\mathrm{time}} = 1001$ samples to compute each gradient update (i.e., on average, one sample per grid point). 
At each time point we draw $N_{\mathrm{space}} = 1$ spatial sample. 
Increasing the sample count adds minimal computational overhead, since both sampling and gradient computation parallelize across samples.
We compute the optimal $q$-divergence-free correction using the Stein estimand \cref{app:eq:stein-estimator} and $N_{\mathrm{quad}} = 1$ Gauss--Hermite quadrature point.
We perform training for $6 \cdot 10^5$ iterations.

To compute each direction (forward and backward KL divergence) of the symmetric KL divergence between the true and approximate posterior, we use Girsanov's formula \cref{app:eqn:path-kl}. 
In particular, we discretize the interval $[0, 5]$ with stepsize $5 \cdot 10^{-4}$ and at each time evaluate the integrand of the path KL $\cL_{\mathrm{path}}$. 
Since both the true and approximate posterior drifts are affine, this can be done in closed form using Gaussian expectations.

In the top row of \Cref{fig:spiral-2pi,fig:spiral-pi,fig:spiral-pi2,fig:spiral-pi4,fig:spiral-zero}, we report the symmetric KL divergence between the approximate and true posterior for $\omega \in \{ 2\pi, \pi, \pi/2, \pi/4, 0\}$. 
Broadly, we see that the suboptimality of choosing a single posterior drift per one-time Gaussian marginals (either the symmetric or square root gauge) increases as the number of observations decreases, the observation noise increases, and the rotational speed increases. 
Even when the rotation speed is strong, choosing a single posterior drift suffices as long as posterior uncertainty is small, i.e., the observations are strongly informative of the latent trajectory.
In settings where the posterior is weakly constrained by the data, we see order-of-magnitude differences between Helmholtz-SDE and the symmetric and square root gauges.

\paragraph{Joint inference and learning.}
Next, we study the impact of the constrained variational family imposed by prior simulation-free VI algorithms on dynamics learning. 

To do so, we parameterize the prior drift with a shallow neural network\\
\texttt{(2) $\to$ Linear(2, 100) $\to$ Softplus $\to$ Linear(100, 2) $\to$ (2)}.\\
We learn prior dynamics together with the affine mapping $\{\mbC, \mbd\}$ and output noise $\mbR$ that define the Gaussian likelihood model.

We perform learning using $1024$ trials. 
We average results over $5$ replicates, each of which uses a different initialization of the  prior and output parameters.
For our experiments, we considered numerous parameterizations of the variational posterior, including a per-trial grid posterior as in our inference experiments and an encoder-based amortized posterior as in our Lorenz experiments. 
For the grid posterior, we observed very slow convergence, even with a large number of Monte Carlo samples $N_{\mathrm{time}}$.
We hypothesize that this is because the parameter space is high dimensional (1001 grid points $\times$ 1024 trials $\times$ 6 parameters per grid point per trial $\approx$ $6 \cdot 10^6$ total parameters), and that each Monte Carlo sample only provides gradient signal at two neighboring grid points; moreover, the grid posterior does not benefit from a smoothness inductive bias.
On the other hand, training the neural network posterior was unstable; it required choosing a large hidden dimension and performing aggressive gradient clipping.
We did not observe that any parameterization preferred a specific gauge, i.e., the choice of posterior parameterization was independent of the gauge. 

We settled on the kernel-GLM posterior parameterization introduced by \citet{course2023state}, which can be viewed as a grid posterior augmented with a smoothness inductive bias and which couples gradient updates across nearby times.
At a high level, this parameterization chooses a grid of temporal inducing locations $\mbtau = \{\tau_n\}_{n=1}^{N_{\mathrm{induce}}}$ that partition $[0, T]$. 
It then defines the posterior mean $\mbm(t)$ at each time point $0 \leq t \leq T$ as $\mbM \mbC k(\mbtau, t)$, where $k(\cdot, \cdot)$ is a kernel function, $k(\mbtau, t) \in \bbR^{N_{\mathrm{induce}} \times 1}$ is a Gram vector, $\mbM \in \bbR^{2 \times N_{\mathrm{induce}}}$ contains the inducing values, and $\mbC \in \bbR^{N_{\mathrm{induce}} \times N_{\mathrm{induce}}}$ is a pre-conditioning matrix. 
If $\mbC$ were equal to the inverse Gram matrix $k(\mbtau, \mbtau)^{-1}$, then $\mbm(t)$ would exactly recover the predictive mean of a multi-output Gaussian process regressor with pseudo-observations $\bm{M}$ at locations $\bm{\tau}$.
However, this requires inverting an $N_{\mathrm{induce}} \times N_{\mathrm{induce}}$ matrix at each step; therefore, the authors fix it to a deterministic value at initialization (no gradient propagation).
A similar parameterization is used for the covariance $\mbS(t)$, which is chosen to be dense.
See \cite{course2023state} for details.

We use $N_{\mathrm{induce}} = 200$ inducing values and choose the kernel  $k(\cdot, \cdot)$ to be a Matérn 5/2 kernel, following \citet{course2023state}. The learnable parameters of the kernel are the output scale, length scale, and warping function parameters (Kumaraswamy warping). 
The kernel parameters (but not the inducing values) are shared across all trials.

We perform training for $3 \cdot 10^4$ iterations. 
At each iteration, we approximate the path KL for each trial using $N_{\mathrm{time}} = 16$ samples of $[0, 5]$. 
At each sampled time point we use $N_{\mathrm{space}} = 1$ spatial sample.
To compute the degree $\ell = 1$ $q$-divergence-free correction, we use the Stein estimand \cref{app:eq:stein-estimator} and $N_{\mathrm{quad}} = 1$ Gauss--Hermite quadrature point.

To compute the symmetric KL between the learned and true prior, we again use Girsanov's theorem \cref{app:eqn:path-kl}.
However, unlike in our inference experiments, the integrand cannot be computed in closed form, since the neural network prior is nonlinear and does not admit closed-form marginal distributions.
Instead, we draw $2048$ trials from both the learned and true prior using Euler--Maruyama with stepsize $5 \cdot 10^{-4}$, and approximate the path KL using Monte Carlo.

In the bottom row of \Cref{fig:spiral-2pi,fig:spiral-pi,fig:spiral-pi2,fig:spiral-pi4,fig:spiral-zero}, we report the symmetric KL divergence between the learned and true prior for $\omega \in \{ 2\pi, \pi, \pi/2, \pi/4, 0\}$. 
Generally, we observe that improved inference yields improved learning: the regimes in which the symmetric and square root gauges poorly approximate the posterior are the same regimes in which they poorly recover the prior dynamics. 
On the other hand, Helmholtz-SDE well approximates the posterior in regimes of high posterior uncertainty, which yields substantive improvements in dynamics recovery.

\clearpage
\begin{figure}[!t]
  \centering
  \includegraphics[width=0.9\textwidth]{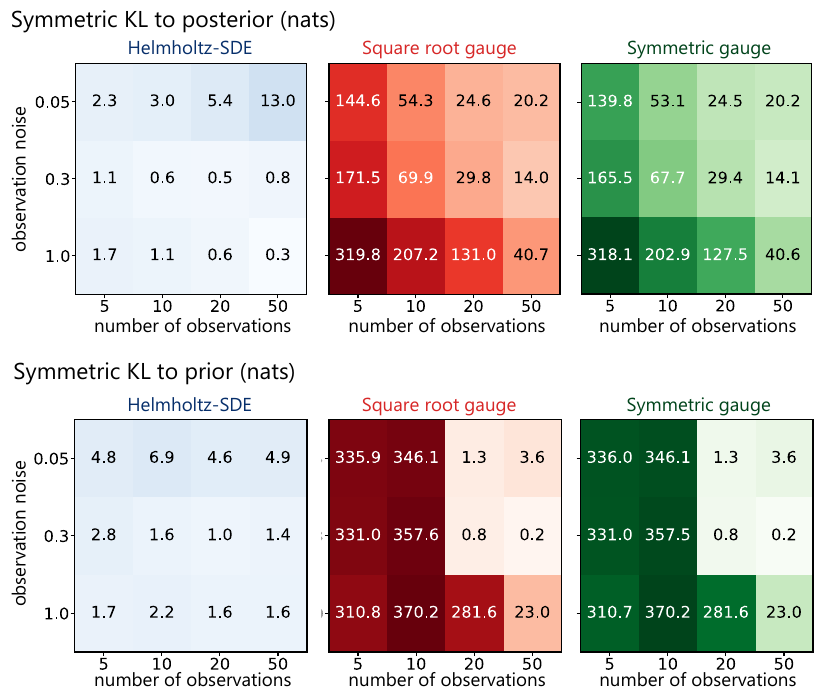}
  \caption{OU spiral with $\omega = 2\pi$; top: symmetric KL to posterior; bottom: symmetric KL to prior.}
  \label{fig:spiral-2pi}
\end{figure}
\begin{figure}[!b]
  \centering
  \includegraphics[width=0.9\textwidth]{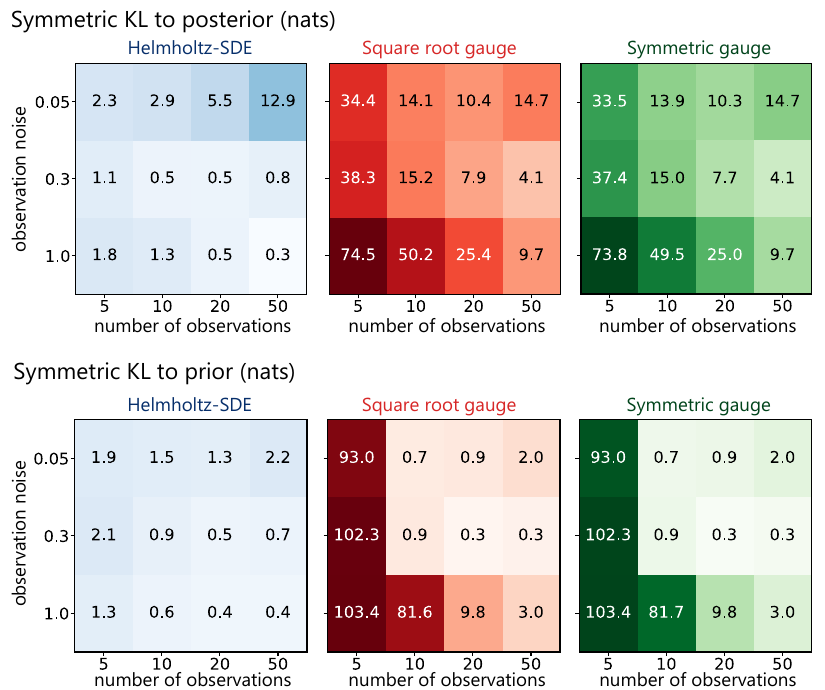}
  \caption{OU spiral  with $\omega = \pi$; top: symmetric KL to posterior; bottom: symmetric KL to prior.}
  \label{fig:spiral-pi}
\end{figure}

\clearpage
\begin{figure}[!t]
  \centering
  \includegraphics[width=0.9\textwidth]{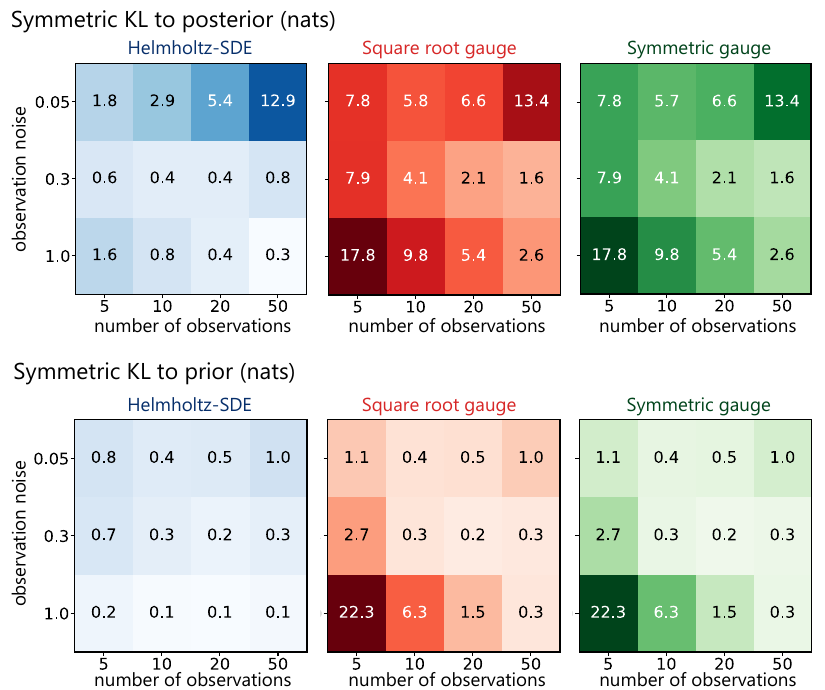}
  \caption{OU spiral  with $\omega = \pi/2$; top: symmetric KL to posterior; bottom: symmetric KL to prior.}
  \label{fig:spiral-pi2}
\end{figure}
\begin{figure}[!b]
  \centering
  \includegraphics[width=0.9\textwidth]{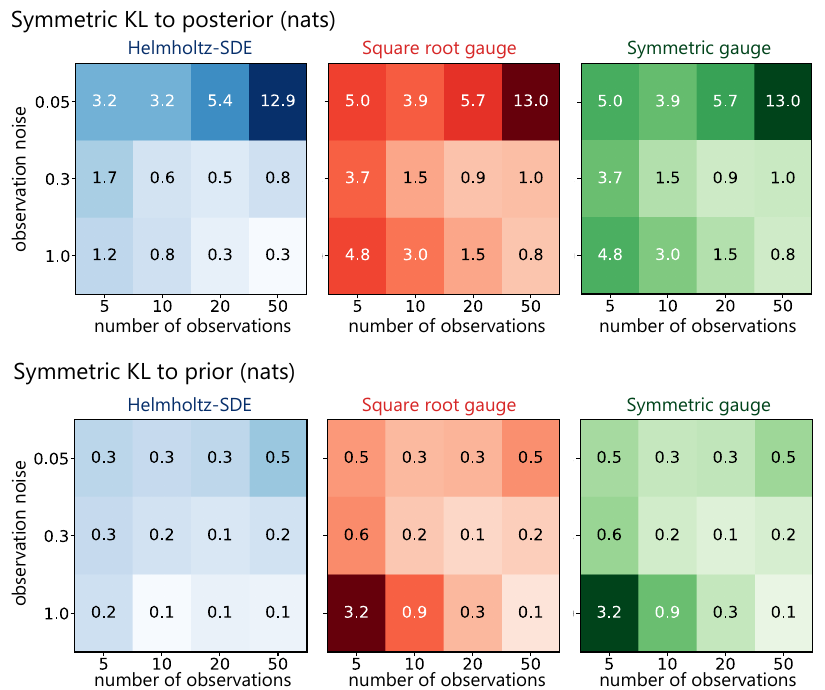}
  \caption{OU spiral with $\omega = \pi/4$; top: symmetric KL to posterior; bottom: symmetric KL to prior.}
  \label{fig:spiral-pi4}
\end{figure}

\clearpage
\begin{figure}[!t]
  \centering
  \includegraphics[width=0.9\textwidth]{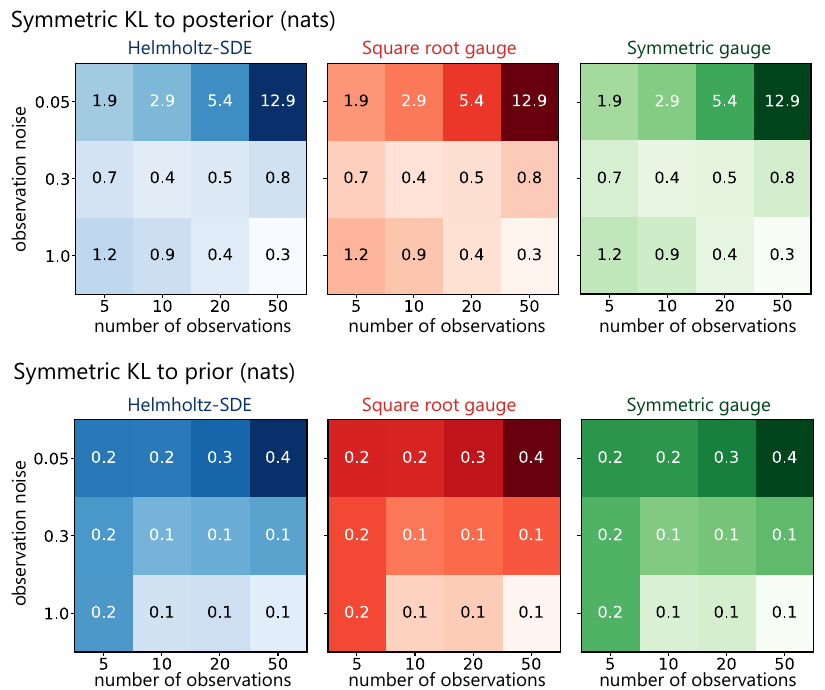}
  \caption{OU spiral  with $\omega = 0$; top: symmetric KL to posterior; bottom: symmetric KL to prior.}
  \label{fig:spiral-zero}
\end{figure}
\clearpage
\section{Nonlinear SDE experiments}\label{app:sec:nonlinear-sde}

In this section, we provide additional details on our three nonlinear SDE experiments: the stochastic Lorenz attractor (\Cref{app:subsec:lorenz}), rotifer-algae predator-prey dynamics (\Cref{app:subsec:predator-prey}), and reduced order models for fluid dynamics (\Cref{app:subsec:rom}).

For all experiments with the exception of the Latent SDE baseline \citep{li2020scalable}, we use a Gaussian posterior approximation $q(\mbx, t) = \cN(\mbm(t), \mbS(t))$ with dense $\mbS \in \bbR^{K \times K}$.
We also use the Adam optimizer \cite{kingma2014adam} with exponential learning rate decay. 
We run experiments on an NVIDIA A100 or H100 GPU.

\subsection{Noisy Lorenz attractor}\label{app:subsec:lorenz}

\begin{figure*}[!t]
  \centering
  \includegraphics[width=\textwidth]{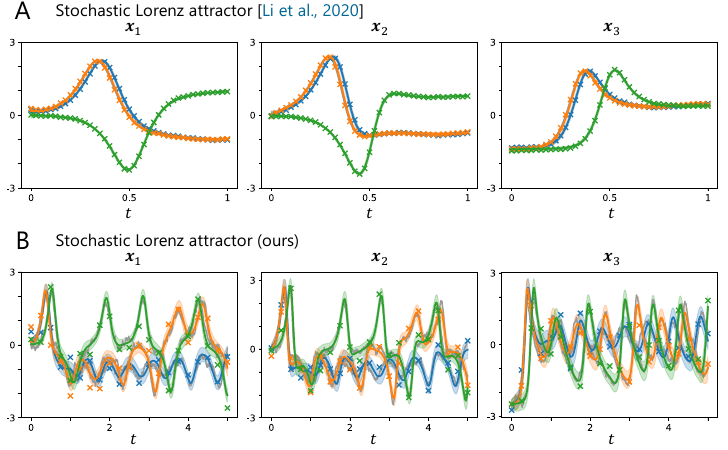}
  \caption{
    Visualization of three samples of the stochastic Lorenz attractor dataset from \citet{li2020scalable} (\textbf{A}) and this paper (\textbf{B}).
    Grey curves represent the true latent trajectory and colored $\times$ represent regularly spaced observations.
    Colored curves represent the posterior mean and shading represents $\pm2$ posterior standard deviations obtained from the iterated extended Rauch--Tung--Striebel smoother (\citet[][Chapter 13.3]{sarkka2023bayesian}).
    The Bayesian posterior in \citet{li2020scalable} collapses to the true latent trajectory, whereas our dataset displays nontrivial posterior variance.
  }
  \label{fig:lorenz-compare}
\end{figure*}

\paragraph{Data generation.}
We generate a dataset consisting of $1024$ training samples and $2048$ test samples according to the $3$-dimensional Lorenz attractor dynamics
\begin{equation}\label{app:eqn:lorenz}
d \mbx(t) = \begin{pmatrix} \alpha(\mbx_2(t) - \mbx_1(t))\\
\mbx_1(t)(\rho - \mbx_3(t)) - \mbx_2(t)\\
\mbx_1(t)\mbx_2(t) - \beta \mbx_3(t)
\end{pmatrix} dt + \sigma d\mbw(t), \quad \mbx(0) \sim \cN(\mb{0}, \bbI_3), \quad 0 \leq t \leq T
\end{equation}
where $(\alpha, \rho, \beta) = (10, 28, 8/3)$ and $\sigma > 0$ is the diffusion scaling.
We do so using the Euler--Maruyama sampler with discretization $10^{-3}$.

Following \citet{li2020scalable}, we make observations of the latent state every $\Delta t$ according to $\mbC_{\mathrm{true}} \mbx(t) + \mbd_{\mathrm{true}} + \mbepsilon(t)$, $\mbepsilon(t) \sim \cN(\mb{0}, \eta^2 \bbI_3)$, where $\mbC_{\mathrm{true}} = \mathrm{diag}(\mbs^{-1})$ and $\mbd_{\mathrm{true}} = - \mbm \odot \mbs^{-1}$; $\mbm, \mbs \in \bbR^3$ represent the empirical mean and standard deviation of the training trajectories, respectively, along each dimension and pooled across time.
$\eta > 0$ is the observation standard deviation.

In the standard stochastic Lorenz attractor, the authors choose $T = 1$ and parameters $(\Delta t, \sigma, \eta) = (2.5 \cdot 10^{-2}, 1.5 \cdot 10^{-1}, 10^{-2})$.
In the observation coordinates, the effective diffusion coefficient is $\sigma \mbs^{-1} \approx (1.7 \cdot 10^{-2}, 1.3 \cdot 10^{-2}, 9.8 \cdot 10^{-3})$, which is exceedingly small relative to the scale on which the latents evolve.
As a result of the small process noise, frequent observations, and small measurement noise, we observe that the posterior collapses to the true latent trajectory (\Cref{fig:lorenz-compare}).
Methods that differ substantially in how they represent or sample posterior uncertainty are therefore indistinguishable on this benchmark.

We instead simulate the latents over a long interval $T = 5$ with larger process noise $\sigma = 5$, less frequent observations $\Delta t = 2.5 \cdot 10^{-1}$, and larger measurement noise $\eta = 0.3$.
In this case, the effective diffusion coefficient in the observation space is $\sigma \mbs^{-1} \approx (6.0 \cdot 10^{-1}, 5.2 \cdot 10^{-1}, 5.2 \cdot 10^{-1})$.
We find that this setting leads both to more interesting latent trajectories (frequent switching between lobes of the attractor, which is not demonstrated in the original dataset; \Cref{fig:lorenz-samples}A) and nontrivial posterior uncertainty (\Cref{fig:lorenz-compare}).

\paragraph{Generative model.}
Consistent with prior work, we learn a latent dynamical system with $K = 4$ dimensions representing the Lorenz attractor.
The prior initial state is modeled as Gaussian $\cN(\mbmu_0, \mbV_0)$, $\mbmu_0 \in \bbR^4$, $\mbV_0 \in \bbR^{4 \times 4}$; the prior drift function \cref{eqn:prior_sde} is a neural network; and the observation model is linear and Gaussian $\cN(\mby(t) | \mbC \mbx(t) + \mbd, \sigma_{\mathrm{obs}}^2 \bbI)$, $\mbC \in \bbR^{3 \times 4}$, $\mbd \in \bbR^3$, $\sigma_{\mathrm{obs}} > 0$.
We learn the parameters of the prior SDE and observation model jointly with those of the posterior.

\paragraph{Baselines.}
We compare Helmholtz-SDE to two simulation-free VI algorithms, ARCTA \cite{course2023amortized} and SDE Matching \cite{bartosh2025sde}, as well as two simulation-based VI algorithms, SING \cite{hu2025sing} and Latent SDE \cite{li2020scalable}.
Note that `ARCTA' simply describes an amortization strategy for SVISE \cite{course2023state}.
In particular, the variational drift, the symmetric gauge, remains the same.

We aim to reproduce the Lorenz attractor experimental setup as faithfully as possible from the original papers, with our modified data generation procedure.
Among these algorithms, there are two noteworthy differences:
(1) SDE Matching \cite{bartosh2025sde} and Latent SDE \cite{li2020scalable} support state-dependent diffusion coefficients;
and (2) SING does not support an amortized posterior approximation.
For SING, we discretize the posterior on a grid of size $10^{-3}$ and use mini-batches of $500$ trials; mini-batching is necessary since the posteriors for all $1024$ trials cannot fit simultaneously into memory (on an 80 GB GPU).
We performed a small sweep over the SING mini-batch size and grid size.

For Helmholtz-SDE, we adopt the same amortization strategy employed by SDE Matching \cite{bartosh2025sde} and Latent SDE \cite{li2020scalable} wherein for each trial, the sequence of observations is encoded by a GRU to produce an embedding.
This embedding is then passed through an MLP to represent the posterior at each time.
This strategy shares posterior parameters both across trials and across time.

We used the ARCTA (\url{https://github.com/coursekevin/arlatentsde}), SDE Matching (\url{https://github.com/GrigoryBartosh/sde_matching}), SING (\url{https://github.com/lindermanlab/sing}), and Latent SDE (\url{https://github.com/google-research/torchsde}) codebases, with modifications as described below.

\begin{itemize}
    \item \textbf{Helmholtz-SDE}:
    \begin{itemize}
        \item \texttt{latent\_dim = 4}, \texttt{obs\_dim = 3},  \texttt{time\_dim = 1}, \texttt{hidden\_dim = 100}
        \item
        Prior drift:\\
        \texttt{(latent\_dim) $\to$ Linear(latent\_dim, hidden\_dim) $\to$ Softplus $\to$ Linear(hidden\_dim, latent\_dim) $\to$ (latent\_dim)}
        \item Posterior encoder:\\ \texttt{(n\_obs, obs\_dim) $\to$ GRU $\to$ (n\_obs + 1, hidden\_dim)}
        \item Posterior MLP:\\
        for a specified time, aggregate embeddings from encoder using softmax, \\
        \texttt{(n\_obs + 1, hidden\_dim), (time\_dim) $\to$ softmax $\to$  (hidden\_dim)}\\
        then map to posterior at that time\\
        \texttt{(hidden\_dim + time\_dim) $\to$ Linear(hidden\_dim + time\_dim, hidden\_dim) $\to$ SiLU  $\to$ Linear(hidden\_dim, hidden\_dim) $\to$ SiLU $\to$ Linear(hidden\_dim, output\_dim) $\to$ (output\_dim)}\\
        where \texttt{output\_dim} is \texttt{latent\_dim} (mean) + \texttt{latent\_dim $\cdot$ (latent\_dim + 1) / 2} (covariance)
        \item Diffusion network: fixed $\bbI_{\texttt{latent\_dim}}$
        \item Helmholtz correction: linear approximation $\ell = 1$, computed using the Stein estimand with $N_{\mathrm{mc}} = 1$
        \item Observation model: fixed covariance $\sigma_{\mathrm{obs}}^2 = \eta^2$, learned $\mbC$, $\mbd$
        \item LR: $10^{-3}$ (constant)
        \item Gradient clipping: $5.0$
        \item Iterations: $3 \cdot 10^4$
        \item Number of MC samples (per iteration): $N_{\mathrm{time}} = 1$ time sample over $[0, T]$ to estimate the KL term and $1$ sampled observation time to estimate the reconstruction loss, per trial
    \end{itemize}
    \item \textbf{SDE Matching}:
    \begin{itemize}
        \item \texttt{latent\_dim = 4}, \texttt{obs\_dim = 3},  \texttt{time\_dim = 1}, \texttt{hidden\_dim = 100}
        \item
        Prior drift:\\
        \texttt{(latent\_dim) $\to$ Linear(latent\_dim, hidden\_dim) $\to$ Softplus $\to$ Linear(hidden\_dim, latent\_dim) $\to$ (latent\_dim)}
        \item Posterior encoder:\\ \texttt{(n\_obs, obs\_dim) $\to$ GRU $\to$ (n\_obs + 1, hidden\_dim)}
        \item Posterior MLP:\\
        for a specified time, aggregate embeddings from encoder using softmax, \\
        \texttt{(n\_obs + 1, hidden\_dim), (time\_dim) $\to$ softmax $\to$  (hidden\_dim)}\\
        then map to posterior at that time\\
        \texttt{(hidden\_dim + time\_dim) $\to$ Linear(hidden\_dim + time\_dim, hidden\_dim) $\to$ SiLU  $\to$ Linear(hidden\_dim, hidden\_dim) $\to$ SiLU $\to$ Linear(hidden\_dim, output\_dim) $\to$ (output\_dim)}\\
        where \texttt{output\_dim} is \texttt{latent\_dim} (mean) + \texttt{latent\_dim $\cdot$ (latent\_dim + 1) / 2} (covariance)
        \item Diffusion network:\\
        state-dependent, diagonal diffusion coefficient; the input to each diagonal entry is a single latent coordinate\\
        \texttt{latent\_dim $\times$ [(1) $\to$ Linear(1, hidden\_dim) $\to$ Softplus $\to$ Linear(hidden\_dim, 1) $\to$ Sigmoid $\to$ (1)]}
        \item Observation model: fixed covariance $\sigma_{\mathrm{obs}}^2 = \eta^2$, learned $\mbC$, $\mbd$
        \item LR: $10^{-3}$ (constant)
        \item Gradient clipping: $5.0$
        \item Iterations: $3 \cdot 10^4$
        \item Number of MC samples (per iteration): $N_{\mathrm{time}} = 1$ time sample over $[0, T]$ to estimate the KL term and $1$ sampled observation time to estimate the reconstruction loss, per trial
    \end{itemize}
    \item \textbf{ARCTA/SVISE}:
    We found that the Gaussian process parameterization of the posterior described by \citet{course2023amortized} (see Appendix D therein) led to poor empirical performance.
    This is because, at a high level, the posterior defines the mean and covariance at each time as a linear combination of means and covariances at observation times.
    Since in our setting the observations occur sparsely in time, interpolating in the space of means and covariances directly is inappropriate.
    Instead, we adopt the posterior parameterization from SDE Matching, wherein interpolation occurs in the embedding space \cite{bartosh2025sde}.

    Moreover, \citet{course2023amortized} propose chunking the full time series of observations into non-overlapping windows.
    For each observation within a given window, they produce an embedding using a specified number of future observations, and then they aggregate embeddings within the window.
    We find that choosing a window size and a number of observations smaller than the full time series degrades performance.
    For this reason, we consider a window equal to the full time series and allow each embedding to depend on all future observations.

    \begin{itemize}
        \item \texttt{latent\_dim = 4}, \texttt{obs\_dim = 3},  \texttt{time\_dim = 1}, \texttt{hidden\_dim = 100}
        \item
        Prior drift:\\
        \texttt{(latent\_dim) $\to$ Linear(latent\_dim, hidden\_dim) $\to$ Softplus $\to$ Linear(hidden\_dim, latent\_dim) $\to$ (latent\_dim)}
        \item Posterior encoder:\\ \texttt{(n\_obs, obs\_dim) $\to$ GRU $\to$ (n\_obs, hidden\_dim)}
        \item Posterior MLP:\\
        for a specified time, aggregate embeddings from encoder using softmax, \\
        \texttt{(n\_obs, hidden\_dim), (time\_dim) $\to$ softmax $\to$  (hidden\_dim)}\\
        then map to posterior at that time\\
        \texttt{(hidden\_dim + time\_dim) $\to$ Linear(hidden\_dim + time\_dim, hidden\_dim) $\to$ SiLU  $\to$ Linear(hidden\_dim, hidden\_dim) $\to$ SiLU $\to$ Linear(hidden\_dim, output\_dim) $\to$ (output\_dim)}\\
        where \texttt{output\_dim} is \texttt{latent\_dim} (mean) + \texttt{latent\_dim $\cdot$ (latent\_dim + 1) / 2} (covariance)
        \item Diffusion network: fixed $\bbI_{\texttt{latent\_dim}}$
        \item Observation model: fixed covariance $\sigma_{\mathrm{obs}}^2 = \eta^2$, learned $\mbC$, $\mbd$
        \item LR: $10^{-3}$ (constant)
        \item Gradient clipping: $5.0$
        \item Iterations: $3 \cdot 10^4$
        \item Number of MC samples (per iteration): $1$ sample over $[0, T]$ to estimate the KL, per trial; use full set of per-trial observations at each step to estimate reconstruction loss (for consistency with ARCTA window algorithm)
    \end{itemize}
    \item \textbf{SING}:
    \begin{itemize}
        \item \texttt{latent\_dim = 4}, \texttt{obs\_dim = 3},  \texttt{hidden\_dim = 100},
        \texttt{grid\_size = $5 \cdot 10^3$ + 1}
        \item
        Prior drift:\\
        \texttt{(latent\_dim) $\to$ Linear(latent\_dim, hidden\_dim) $\to$ Softplus $\to$ Linear(hidden\_dim, latent\_dim) $\to$ (latent\_dim)}
        \item Posterior (non-amortized): tensors of shape \texttt{(grid\_size, latent\_dim)} (mean), \texttt{(grid\_size, latent\_dim, latent\_dim)} (covariance), \texttt{(grid\_size - 1, latent\_dim, latent\_dim)} (cross-covariance)
        \item Diffusion network: fixed $\bbI_{\texttt{latent\_dim}}$
        \item Observation model: fixed covariance $\sigma_{\mathrm{obs}}^2 = \eta^2$, learned $\mbC$, $\mbd$
        \item Mini-batch size: $500$
        \item LR:
        \begin{itemize}
            \item SING stepsize (i.e., natural gradient update): $10^{-3}$ to $10^{-2}$ on log-linear scale for first $10$ iterations, constant after
            \item M-step stepsize: $5 \cdot 10^{-3}$ (Adam)
        \end{itemize}
        \item Gradient clipping: $5.0$
        \item Iterations: for each iteration perform 1 SING step, 1 M-step gradient update; $10^4$ total iterations
        \item Number of MC samples (per iteration): $1$ sample to compute expectations under the variational posterior
    \end{itemize}
    \item \textbf{Latent SDE}:
    \begin{itemize}
        \item \texttt{latent\_dim = 4}, \texttt{obs\_dim = 3},  \texttt{time\_dim = 1}, \texttt{hidden\_dim = 128}, \texttt{embedding\_dim = 64}
        \item
        Prior drift:\\
        \texttt{(latent\_dim) $\to$ Linear(latent\_dim, hidden\_dim) $\to$ Softplus $\to$ Linear(hidden\_dim, hidden\_dim) $\to$ Softplus $\to$ Linear(hidden\_dim, latent\_dim) $\to$ (latent\_dim)}
        \item Posterior encoder:\\ \texttt{(n\_obs, obs\_dim) $\to$ reverse $\to$ GRU $\to$ reverse $\to$ (n\_obs, hidden\_dim) $\to$ Linear(hidden\_dim, embedding\_dim) $\to$ (embedding\_dim)}
        \item Posterior drift network:\\
        for a specified latent state, use embedding to produce drift\\
        \texttt{(latent\_dim + embedding\_dim) $\to$ Linear(latent\_dim + embedding\_dim, hidden\_dim) $\to$ Softplus $\to$ Linear(hidden\_dim, hidden\_dim) $\to$ Softplus $\to$ Linear(hidden\_dim, latent\_dim) $\to$ (latent\_dim)}
        \item Posterior initial distribution:\\
        use embedding to produce distribution (Gaussian) over initial latent state\\
        \texttt{(embedding\_dim) $\to$ Linear(embedding\_dim, latent\_dim + latent\_dim) $\to$ (latent\_dim + latent\_dim)}\\
        first \texttt{latent\_dim} outputs represent mean, second \texttt{latent\_dim} outputs represent log standard deviation
        \item Diffusion network:\\
        state-dependent, diagonal diffusion coefficient; the input to each diagonal entry is a single latent coordinate\\
        \texttt{latent\_dim $\times$ [(1) $\to$ Linear(1, hidden\_dim) $\to$ Softplus $\to$ Linear(hidden\_dim, 1) $\to$ Sigmoid $\to$ (1)]}
        \item Observation model: fixed covariance $\sigma_{\mathrm{obs}}^2 = \eta^2$, learned $\mbC$, $\mbd$
        \item Training discretization: $10^{-2}$ (for posterior sampling)
        \item LR: initial $10^{-3}$, final $10^{-4}$ (exponential LR decay)
        \item Gradient clipping: None
        \item Iterations: $3 \cdot 10^4$
        \item KL annealing iterations: $10^3$
        \item Number of MC samples (per iteration): $1$ path simulated from the posterior per trial

    \end{itemize}
\end{itemize}

\paragraph{Metrics.}\leavevmode\\
We evaluate the fidelity of the posterior and prior dynamics according to four metrics, which we explain below.

\textbf{nELBO.} Average per-trial negative ELBO (the variational objective).
Smaller values are better but do not imply better posterior recovery or learning.

\textbf{Posterior predictive LL.}
To evaluate posterior quality, we use the posterior predictive log density at the noise-free observations, averaged over trials and evaluation times.
Let $\cI$ denote the set of trial-time pairs $(m, t)$ used for evaluation, excluding the observation times used for training in each trial.
The metric is
\begin{equation}\label{app:eqn:predictive-prob}
    \frac{1}{|\cI|} \sum_{(m, t) \in \cI} \log \bar{q}(\mby^{(m)}(t), t).
\end{equation}
Here, $\mby^{(m)}(t) = \mbC_{\mathrm{true}} \mbx_{\mathrm{true}}^{(m)}(t) + \mbd_{\mathrm{true}}$ is the noise-free observation obtained from the true $3$-dimensional latent state generated according to \cref{app:eqn:lorenz}.
$\bar{q}(\mby, t)$ represents the posterior predictive density of $\bar{\mby}(t) = \mbC \mbx(t) + \mbd + \mbepsilon(t)$, where $\mbepsilon(t) \sim \cN(\mb{0}, \mbR)$ and $\mbx(t)$ is distributed according to the learned variational posterior for trial $m$.
We suppress the trial index on $q$, $\bar{q}$, and the posterior samples, which are conditional on the observations from the corresponding trial.

For all cases except Latent SDE \cite{li2020scalable}, $q(\mbx, t) = \cN(\mbm(t), \mbS(t))$ is Gaussian, in which case the posterior predictive density admits a closed form $\bar{q}(\mby, t) = \cN(\mby \mid \mbC \mbm(t) + \mbd, \mbC \mbS(t) \mbC^\top + \mbR)$.
For Latent SDE, we estimate $\bar{q}(\mby^{(m)}(t), t)$ by drawing $N_{\mathrm{density}} = 64$ samples $\mbx^{(k)}(t) \sim q(\mbx, t)$, computing $\bar{\mby}^{(k)}(t) = \mbC \mbx^{(k)}(t) + \mbd$, and then evaluating
\begin{equation}\label{app:eqn:latent-sde-estimate}
    \frac{1}{|\cI|} \sum_{(m, t) \in \cI} \log \left(\frac{1}{N_{\mathrm{density}}} \sum_{k=1}^{N_{\mathrm{density}}} \cN(\mby^{(m)}(t) \mid \bar{\mby}^{(k)}(t), \mbR) \right).
\end{equation}
Note that \cref{app:eqn:latent-sde-estimate} is a sensible approximation to \cref{app:eqn:predictive-prob} since
{\small
\begin{equation*}
\begin{aligned}
    \bar{q}(\mby, t) &= \int q(\mbx, t) \cN(\mby \mid \mbC \mbx + \mbd, \mbR) d \mbx = \bbE_{q(\mbx, t)}[\cN(\mby \mid \mbC \mbx + \mbd, \mbR)]\\
    &\approx \frac{1}{N_{\mathrm{density}}} \sum_{k=1}^{N_{\mathrm{density}}} \cN(\mby \mid \bar{\mby}^{(k)}(t), \mbR).
\end{aligned}
\end{equation*}
}

\begin{figure*}[h]
  \centering
  \includegraphics[width=0.9\textwidth]{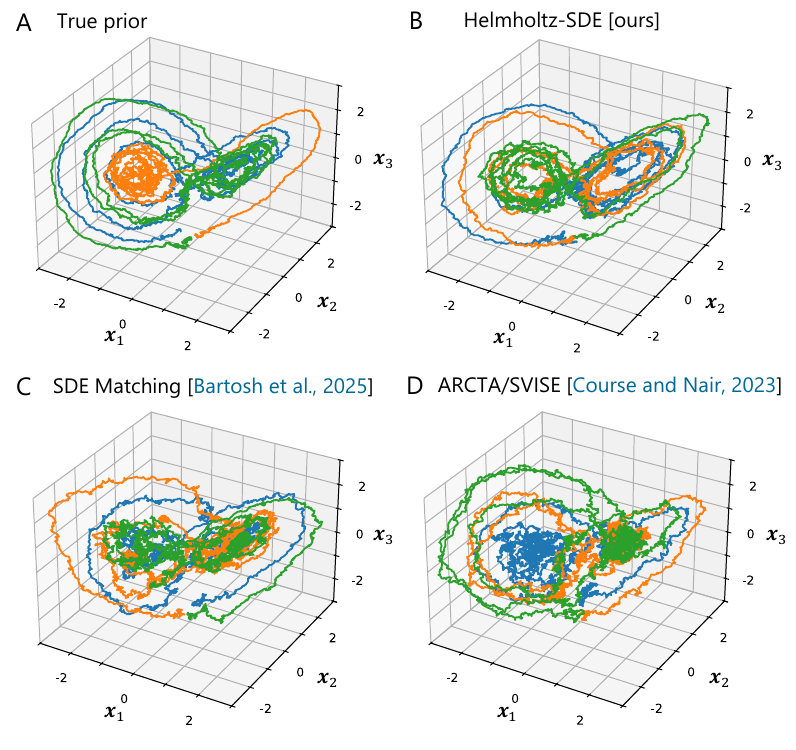}
  \caption{
   Three random samples from the 
   true prior (\textbf{A}),
   the prior learned by Helmholtz-SDE (\textbf{B}),
   SDE Matching (\textbf{C}),
   and ARCTA/SVISE (\textbf{D}).
   Helmholtz-SDE more faithfully learns the dynamics, particularly the outward spiral dynamics within each lobe of the attractor. 
  }\label{fig:lorenz-samples}
\end{figure*}

\textbf{Prior learning.}
For evaluating the learned prior dynamics, we draw $2048$ samples from the learned prior, starting from $p(\mbx, 0)$ and simulating using Euler--Maruyama with discretization $10^{-3}$.
We evaluate the learned output mapping at these samples, from which we compute the following metrics:
\begin{enumerate}[label=(\roman*)]
    \item \textbf{One-time symmetric KL divergence}: symmetric KL divergence at each time in $[0, T]$, averaged over times.
    For samples from two distributions $p$ and $q$, we estimate the KL divergence $\KL{p}{q}$ by fitting two Gaussian kernel density estimates and computing a Monte Carlo average.

    We do so using the following procedure, which mirrors that of \citet{kiyohara2025neural}:
    \begin{enumerate}
        \item Using 3-fold CV, select the KDE bandwidth for $q$; fit the density estimate (on all samples) using the chosen bandwidth.
        \item Split the samples from $p$ into $10$ folds.
        Holding out a single fold at a time, repeat the following subprocedure:
        \begin{enumerate}
            \item Choose the kernel bandwidth for $p$ using 3-fold CV; fit the density estimate (on all training folds) using the chosen bandwidth.
            \item Compute a Monte Carlo estimate of the KL divergence using the held-out fold.
        \end{enumerate}
    \end{enumerate}
    The outer CV step evaluates each density estimate for $p$ on held-out samples, avoiding reuse of its training samples for the Monte Carlo log-density average.
    We repeat this procedure with $p$ and $q$ exchanged to estimate the reverse KL divergence.

    \item \textbf{Global autocorrelation error}: Using samples from the learned prior and true prior, we estimate lagged correlation matrices in the noise-free observation coordinates $\mbu(t) = \mbC \mbx(t) + \mbd$ and $\mbu_{\mathrm{true}}(t) = \mbC_{\mathrm{true}} \mbx_{\mathrm{true}}(t) + \mbd_{\mathrm{true}}$, respectively.
    For each lag $\tau$, we estimate $\mathrm{Corr}(\mbu(t), \mbu(t + \tau))$ across samples at each time $t$ such that $t, t + \tau$ lie on the sampling grid, and then average these matrices over all valid times.
    We do the same for the true prior.
    We compute the squared Frobenius norm difference between the two time-averaged correlation matrices and average over lags $\tau$ on a linear grid of $[0, 2]$ with spacing $5 \cdot 10^{-2}$.
    Finally, we normalize by the total number of entries in the correlation matrix ($3 \times 3 = 9$) and take the square root of the result, yielding the entrywise RMSE.
\end{enumerate}

\paragraph{Ablation of $q$-divergence-free correction.}
We report ablations over the $q$-divergence-free correction used by Helmholtz-SDE. 
In \Cref{tab:lorenz-nelbo}, we report the negative ELBO for polynomial degrees $\ell \in \{1, 2, 3\}$ and for different estimators of the polynomial divergence-free correction.
For all estimators, we evaluate the final negative ELBO using $N_{\mathrm{quad}} = 6$ quadrature points per dimension (i.e., the estimators differ during training but not evaluation).
$\kappa$ denotes the degree of the Stein estimand \cref{app:eq:stein-estimator}; when $\kappa$ is smaller than the polynomial degree $\ell$, we take derivatives of the largest entries of $\mbalpha$.

In general, we observe that the Stein estimand \cref{app:eq:stein-estimator} performs better than the original expectation, especially when it is estimated using Monte Carlo with few samples $N_{\mathrm{mc}}$.
Gauss--Hermite quadrature is less sensitive to the degree of the Stein estimand.
Moreover, we see that increasing the polynomial degree from $\ell = 1$ to $\ell = 2$ reduces the negative ELBO by approximately $0.7$ nats.
We do not observe a significant improvement when increasing the degree from $\ell = 2$ to $\ell = 3$. 

\begin{table*}[h]
\centering
\footnotesize
\setlength{\tabcolsep}{3.5pt}

% --- (a) ell = 1 ---
\subcaptionbox{$\ell = 1$\label{tab:lorenz-ell1}}{%
\begin{tabular}{lccccccc}
\toprule
& \multicolumn{5}{c}{Monte Carlo ($N_{\mathrm{mc}}$)} & \multicolumn{2}{c}{Quadrature ($N_{\mathrm{quad}}$)} \\
\cmidrule(lr){2-6}\cmidrule(lr){7-8}
$\kappa$ & 1 & 4 & 16 & 64 & 1024 & 2 & 4 \\
\midrule
0 & $76.56 \pm 2.80$ & $55.74 \pm 0.41$ & $49.89 \pm 0.18$ & $48.05 \pm 0.13$ & $47.61 \pm 0.30$ & $47.44 \pm 0.17$ & $47.42 \pm 0.14$ \\
1 & $47.96 \pm 0.45$ & $47.48 \pm 0.17$ & $47.43 \pm 0.17$ & $47.53 \pm 0.23$ & $47.60 \pm 0.22$ & $47.43 \pm 0.15$ & $47.50 \pm 0.20$ \\
\bottomrule
\end{tabular}}

\vspace{0.75em}

% --- (b) ell = 2 ---
\subcaptionbox{$\ell = 2$\label{tab:lorenz-ell2}}{%
\begin{tabular}{lccccccc}
\toprule
& \multicolumn{5}{c}{Monte Carlo ($N_{\mathrm{mc}}$)} & \multicolumn{2}{c}{Quadrature ($N_{\mathrm{quad}}$)} \\
\cmidrule(lr){2-6}\cmidrule(lr){7-8}
$\kappa$ & 1 & 4 & 16 & 64 & 1024 & 2 & 4 \\
\midrule
0 & $111.66 \pm 2.08$ & $93.88 \pm 0.47$ & $60.32 \pm 0.53$ & $51.71 \pm 0.54$ & $46.94 \pm 0.12$ & $46.82 \pm 0.16$ & $46.69 \pm 0.16$ \\
1 & $68.32 \pm 0.90$ & $52.97 \pm 1.12$ & $47.89 \pm 0.17$ & $47.11 \pm 0.20$ & $46.95 \pm 0.18$ & $46.71 \pm 0.13$ & $47.03 \pm 0.25$ \\
2 & $47.23 \pm 0.41$ & $46.85 \pm 0.19$ & $46.93 \pm 0.16$ & $46.80 \pm 0.20$ & $46.71 \pm 0.12$ & $46.86 \pm 0.13$ & $46.86 \pm 0.17$ \\
\bottomrule
\end{tabular}}

\vspace{0.75em}

% --- (c) ell = 3 ---
\subcaptionbox{$\ell = 3$\label{tab:lorenz-ell3}}{%
\begin{tabular}{lccccccc}
\toprule
& \multicolumn{5}{c}{Monte Carlo ($N_{\mathrm{mc}}$)} & \multicolumn{2}{c}{Quadrature ($N_{\mathrm{quad}}$)} \\
\cmidrule(lr){2-6}\cmidrule(lr){7-8}
$\kappa$ & 1 & 4 & 16 & 64 & 1024 & 2 & 4 \\
\midrule
0 & $172.69 \pm 0.90$ & $109.99 \pm 1.94$ & $94.06 \pm 0.37$ & $60.25 \pm 0.39$ & $47.84 \pm 0.18$ & $55.33 \pm 0.40$ & $46.84 \pm 0.28$ \\
1 & $92.99 \pm 1.28$ & $59.86 \pm 0.44$ & $50.16 \pm 0.68$ & $49.25 \pm 3.36$ & $46.82 \pm 0.23$ & $46.79 \pm 0.18$ & $46.85 \pm 0.19$ \\
2 & $47.19 \pm 0.30$ & $46.94 \pm 0.22$ & $46.83 \pm 0.18$ & $46.81 \pm 0.16$ & $46.79 \pm 0.19$ & $46.82 \pm 0.21$ & $46.82 \pm 0.17$ \\
3 & $46.98 \pm 0.23$ & $47.04 \pm 0.35$ & $46.96 \pm 0.28$ & $46.65 \pm 0.11$ & \textemdash & $46.77 \pm 0.16$ & \textemdash \\
\bottomrule
\end{tabular}}
\caption{Ablation of the $q$-divergence-free correction used by Helmholtz-SDE. 
Dashes represent configurations that could not fit in memory on an 80 GB GPU.
Errors are $\pm2$ SE over $10$ replicates.
}
\label{tab:lorenz-nelbo}
\end{table*}

Next, we evaluate a strategy for estimating the improvement from higher order polynomial corrections that does \emph{not} require retraining Helmholtz-SDE. 
In \Cref{app:subsec:optimality-nonconj}, we derived a lower bound on the improvement in the negative ELBO from degree $\ell$ to $\ell + k$ that can be approximated at the marginals of the degree $\ell$ model.
We train models with degree $\ell = 0$ (reference drift) and $\ell = 1$ polynomial corrections using the Stein estimand with $N_{\mathrm{mc}} = 1$.
After training, we approximate the lower bound by computing the Hermite coefficients of the residual with $N_{\mathrm{quad}} = 6$ quadrature points.

In \Cref{tab:lorenz-fixed-gauge-improvement}, we see that our estimated bound on the improvement from $\ell = 0$ to $\ell = 1$ is nearly $12$ nats, which is smaller than the observed $20$ nat improvement in \Cref{fig:lorenz}. 
The estimated bound on the improvement from $\ell = 1$ to $\ell = 2$ is $0.4$ nats, smaller than the $0.7$ nat improvement in \Cref{tab:lorenz-nelbo}.

\begin{table}[!h]
\centering
\setlength{\tabcolsep}{3pt}
\begin{tabular}{lccc}
\toprule
& $\ell = 1$ & $\ell = 2$ & $\ell = 3$ \\
\midrule
$0 \to \ell$ & $11.72 \pm 0.80$ & $12.16 \pm 0.77$ & $12.18 \pm 0.77$ \\
$1 \to \ell$ & \textemdash & $0.374 \pm 0.006$ & $0.378 \pm 0.006$ \\
\bottomrule
\end{tabular}
\vspace{0.5em}
\caption{
$k \to \ell $ estimates a lower bound on the reduction in negative ELBO from using a degree $\ell$ polynomial over a degree $k$ polynomial.
See \Cref{app:subsec:optimality-nonconj} for details.
Errors are $\pm2$ SE over $10$ replicates.}
\label{tab:lorenz-fixed-gauge-improvement}
\vspace{-1em}
\end{table}

\subsection{Experimental predator-prey dynamics}\label{app:subsec:predator-prey}

In this subsection, we provide details regarding our experiments with the algae-rotifer predator-prey system from \citet{blasius2020long}.

\paragraph{Preprocessing.}
We obtain the dataset made public by the authors at \url{https://figshare.com/articles/dataset/Time_series_of_long-term_experimental_predator-prey_cycles/10045976/1}.

In our analysis, we consider only the first trial, which is the longest and spans $374$ days.
We choose not to consider multiple trials ($4$ total in the same experimental condition) simultaneously since there are apparent differences between them, for example, the mean and standard deviation of population sizes. Modeling the dynamics across multiple trials using, for example, a hierarchical Bayesian model is interesting future work. 

Since we learn autonomous (time-independent) dynamics, we restrict each trial to periods of coherent oscillation as determined by \citet{blasius2020long}. 
Otherwise, the latent dynamics must explain both coherent and non-coherent (no systematic phase lag) oscillations.
For the trial of interest, this yields four subtrials $\{ [19.5, 130.5], [159.5, 240.5], [274.5, 295.5], [314.5, 370.5] \}$.

We next log transform the data. The reason for this decision is that, in stochastic models of population dynamics,  modeling the noise as multiplicative in the state is a popular choice, particularly due to the stability properties it affords \cite{mao2002environmental}.
Log transforming the data then allows us to model the noise as state-independent, which is the setting in which Helmholtz-SDE operates.
For example, consider the Lotka--Volterra equations with state-dependent noise for predator $y$ and prey $x$
\begin{equation}\label{eqn:lotka-volterra}
    \begin{aligned}
        &dx(t) = (\alpha x(t) - \beta x(t)y(t))dt + \sigma_x x d \mbw_1(t)\\
        &dy(t) = (\delta x(t)y(t) - \gamma y(t) )dt + \sigma_y y d \mbw_2(t)
    \end{aligned}
\end{equation}
with fixed initial distribution and $\alpha, \beta, \delta, \gamma > 0$. Applying Ito's lemma to the map $(x(t), y(t)) \mapsto (\log x(t), \log  y(t))$ yields the dynamics in log coordinates
\begin{equation}\label{eqn:log-lotka-volterra}
    \begin{aligned}
        &d(\log x(t)) = \left( \alpha - \beta y(t) - \frac{1}{2}\sigma_x^2  \right)dt + \sigma_x d \mbw_1(t)\\
        &d(\log y(t)) = \left( \delta x(t) - \gamma - \frac{1}{2}\sigma_y^2 \right) dt + \sigma_y d \mbw_2(t)
    \end{aligned}.
\end{equation}
Importantly, the diffusion coefficient appearing in \cref{eqn:log-lotka-volterra} does not depend on the state. Observe that if \cref{eqn:log-lotka-volterra} describes the prior of a latent SDE model and we choose a posterior approximation with Gaussian one-time marginals, then the one-time marginal distributions of $x(t)$ become log-normal.

We hold out the last $8$ observations of each subtrial to compute the forecast MSE.

\paragraph{Generative model.}
We learn a $K = 2$ dimensional dynamical systems model of the predator-prey dynamics.
The prior initial state is modeled as Gaussian $\cN(\mbmu_0, \mbV_0)$, $\mbmu_0 \in \bbR^2$, $\mbV_0 \in \bbR^{2 \times 2}$; the prior drift function \cref{eqn:prior_sde} is a neural network; and the observation model is linear and Gaussian $\cN(\mby(t) | \mbC \mbx(t) + \mbd, \sigma_{\mathrm{obs}}^2 \bbI)$, $\mbC \in \bbR^{2 \times 2}$, $\mbd \in \bbR^2$, $\sigma_{\mathrm{obs}} > 0$.
We perform joint inference and parameter learning.

\paragraph{Baselines.}
We compare Helmholtz-SDE to SDE Matching \cite{bartosh2025sde} and SVISE \cite{course2023state}. 
For this case study, we find that defining the posterior on a grid of size $0.1$ days, as done in \Cref{app:subsec:ou-spiral-supplement}, yields good performance. 
Other parameterizations, such as the kernel-GLM in SVISE \cite{course2023state}, perform comparably. 
For SDE Matching, we consider modeling the diffusion coefficient as both state-dependent and state-independent, but we find no substantive difference between the results. Results reported are for a state-dependent diffusion coefficient.

\begin{itemize}
    \item \textbf{Helmholtz-SDE}:
    \begin{itemize}
        \item \texttt{latent\_dim = 2}, \texttt{obs\_dim = 2},  \texttt{time\_dim = 1}, \texttt{hidden\_dim = 64}
        \item 
        Prior drift:\\
        \texttt{(latent\_dim) $\to$ Linear(latent\_dim, hidden\_dim) $\to$ Softplus $\to$ Linear(hidden\_dim, latent\_dim) $\to$ (latent\_dim)}
        \item Posterior: grid of means and covariances, with spacing $0.1$ days; linear interpolation between grid points
        \item Diffusion network: fixed $\bbI_{\texttt{latent\_dim}}$
        \item Helmholtz correction: linear approximation $\ell = 1$, computed using the Stein estimand with $N_{\mathrm{mc}} = 1$
        \item Observation model: learned $\mbC$, $\mbd$ and $\sigma_{\mathrm{obs}}^2$
        \item LR: $10^{-2}$ initial, $10^{-3}$ final (exponential LR decay)
        \item Gradient clipping: $5.0$
        \item Iterations: $3 \cdot 10^4$
        \item Number of MC samples (per iteration): $N_{\mathrm{time}} = 10^4$ samples over $[0, T]$ to estimate each of the KL and reconstruction loss, per trial
    \end{itemize}
    \item \textbf{SDE Matching}:
    \begin{itemize}
        \item \texttt{latent\_dim = 2}, \texttt{obs\_dim = 2},  \texttt{time\_dim = 1}, \texttt{hidden\_dim = 64}
        \item 
        Prior drift:\\
        \texttt{(latent\_dim) $\to$ Linear(latent\_dim, hidden\_dim) $\to$ Softplus $\to$ Linear(hidden\_dim, latent\_dim) $\to$ (latent\_dim)}
         \item Posterior: grid of means and covariances, with spacing $0.1$ days; linear interpolation  between grid points
        \item Diffusion network:\\ 
        state-dependent, diagonal diffusion coefficient; the input to each diagonal entry is a single latent coordinate\\
        \texttt{latent\_dim $\times$ [(1) $\to$ Linear(1, hidden\_dim) $\to$ Softplus $\to$ Linear(hidden\_dim, 1) $\to$ Sigmoid $\to$ (1)]}
        \item Observation model: learned $\mbC$, $\mbd$ and $\sigma_{\mathrm{obs}}^2$
        \item LR: $10^{-2}$ initial, $10^{-3}$ final (exponential LR decay)
        \item Gradient clipping: $5.0$
        \item Iterations: $3 \cdot 10^4$
        \item Number of MC samples (per iteration): $N_{\mathrm{time}} = 10^4$ samples over $[0, T]$ to estimate each of the KL and reconstruction loss, per trial
    \end{itemize}
    \item \textbf{SVISE}:
    \begin{itemize}
        \item \texttt{latent\_dim = 2}, \texttt{obs\_dim = 2},  \texttt{time\_dim = 1}, \texttt{hidden\_dim = 64}
        \item 
        Prior drift:\\
        \texttt{(latent\_dim) $\to$ Linear(latent\_dim, hidden\_dim) $\to$ Softplus $\to$ Linear(hidden\_dim, latent\_dim) $\to$ (latent\_dim)}
         \item Posterior: grid of means and covariances, with spacing $0.1$ days; linear interpolation  between grid points
        \item Diffusion network: fixed $\bbI_{\texttt{latent\_dim}}$
        \item Observation model: learned $\mbC$, $\mbd$ and $\sigma_{\mathrm{obs}}^2$
        \item LR: $10^{-2}$ initial, $10^{-3}$ final (exponential LR decay)
        \item Gradient clipping: $5.0$
        \item Iterations: $3 \cdot 10^4$
        \item Number of MC samples (per iteration): $N_{\mathrm{time}} = 10^4$ samples over $[0, T]$ to estimate each of the KL and reconstruction loss, per trial
    \end{itemize}
\end{itemize}

\paragraph{Wavelet analysis.}
We assess the quality of the learned prior dynamics using the wavelet analysis from \citet{blasius2020long}.
In particular, we draw $1024$ samples from the learned prior with stepsize $2 \cdot 10^{-3}$ days, each of length $100$ days.
From each sample, we downsample to a grid of $1$ day and extract the peak cyclic period (\Cref{fig:rotifer-algae}C), the lag of predator behind prey (\Cref{fig:rotifer-algae}D), and the rotational coherence (\Cref{fig:rotifer-algae}D).
We then report the median and IQR of these statistics across $10$ replicates of fitting the latent SDE model.

We use the authors' public codebase \url{https://github.com/berndblasius/WaveletAnalysis} to perform our analyses.
Next, we briefly discuss how we compute each of these statistics.
We refer the interested reader to \citet[][supplementary methods]{blasius2020long} as well as \citet{torrence1998practical} for additional background on wavelet analysis.

\textbf{1.~The continuous wavelet transform.}\\
The \emph{continuous wavelet transform} of a one-dimensional signal $(x(t))_{0 \leq t \leq T}$ is defined as the convolution of the signal with the complex-valued wavelet function, centered at time $t$ and dilated by scale $s > 0$
\begin{equation}\label{app:eqn:wavelet}
    W_x(s, t) = \frac{1}{\sqrt{s}}\int_0^{T} x(t') \psi^*\left( \frac{t' - t}{s}\right)dt'.
\end{equation}
Here, $\psi^*$ represents the complex conjugate of $\psi$.
While there are many possible choices of wavelet function, we consider the Morlet wavelet with $\omega_0 = 6$:
\begin{equation*}
    \psi(\tau) = \pi^{-1/4} \exp\{i \omega_0 \tau\} \exp\{-\tau^2/2\}.
\end{equation*}

To build intuition, consider what \cref{app:eqn:wavelet} would compute if $\psi$ were a pure complex exponential $\psi(\tau) = \exp\{2 \pi i \tau\}$.
Then $|W_x(s, t)|$ equals $s^{-1/2}$ times the magnitude of the Fourier transform of $x$ at frequency $1/s$: it measures how strongly the signal oscillates at period $s$, but the result does not depend on $t$ and so carries no information about when in the time series this oscillation occurs.
The Morlet wavelet recovers time localization by multiplying the complex exponential by a Gaussian envelope $\exp\{-\tau^2/2\}$, which restricts the integral in \cref{app:eqn:wavelet} to a neighborhood of $t$.
The price of this localization is that $\psi$ no longer corresponds to a single frequency; instead, each scale $s$ is associated with a \emph{Fourier period} by matching $s$ to the scale at which the wavelet power of a sinusoid of that period is maximized.
For the Morlet wavelet, this correspondence is
\begin{equation}\label{app:eqn:fourier-period}
    T(s) = \frac{4 \pi s}{\omega_0 + \sqrt{2 + \omega_0^2}},
\end{equation}
so that $T(s) \approx 1.03 s$ when $\omega_0 = 6$ \citep{torrence1998practical}.
With this identification, $|W_x(s, t)|$ measures the local oscillatory content of the signal at time $t$ and period $T(s)$.

Following \citet{blasius2020long}, we evaluate the wavelet transform on a logarithmic grid of scales.
The smallest scale is $s_0 = 2$ days, and the largest scale is $s_0 2^J = 16$ days, where $J = 3$ is the number of octaves, and we use $100$ grid points per octave.

\textbf{2.~Wavelet power spectrum and cross spectrum.}\\
The local oscillatory content of a single signal is summarized by the \emph{wavelet power spectrum} (WPS), $|W_x(s, t)|^2$.
To compare two signals, $x_1$ (algae) and $x_2$ (rotifers), we form the \emph{wavelet cross-spectrum} (WCS) $W_{x_1 x_2}(s, t) = W_{x_1}(s, t) W_{x_2}^*(s, t)$, whose magnitude $|W_{x_1 x_2}(s, t)|$ quantifies the joint oscillatory power at scale $s$ and time $t$.
Following \citet{blasius2020long}, we smooth both the WPS and WCS using a Gaussian temporal filter and a moving-average filter over scales; we denote this operator by $\cS(\cdot)$.
The argument of the smoothed WCS,
\begin{equation}\label{app:eqn:phase}
    \phi(s, t) = \arg \cS(W_{x_1 x_2}(s, t)),
\end{equation}
gives the local phase difference between the two signals at scale $s$ and time $t$.
A nonzero phase difference $\phi(s, t)$ indicates that one signal leads the other; converting to a time lag at the corresponding Fourier period yields $\Delta t (s, t) = \phi(s, t) T(s) / (2 \pi)$.

A high value of $|W_{x_1 x_2}(s, t)|$ does not by itself indicate that the two signals are oscillating coherently: it may simply reflect that both signals have large individual power at $(s, t)$, even if their phases drift independently.
To distinguish coherent oscillation from coincident power, we use the \emph{wavelet coherence} (WCO), defined as the normalized, locally smoothed cross-spectrum
\begin{equation}\label{app:eqn:coherence}
    R(s, t) = \frac{\left| \cS\left( W_{x_1 x_2}(s, t) \right) \right|}{\sqrt{\cS\left(|W_{x_1}(s, t)|^2 \right) \cdot \cS\left( |W_{x_2}(s, t)|^2 \right)}}.
\end{equation}
Coherence takes values in $[0, 1]$ and is invariant to the marginal amplitudes of $x_1$ and $x_2$: $R^2(s, t) \approx 1$ at points $(s, t)$ where the two signals oscillate with a stable phase relationship over a local neighborhood.
As pointed out by \citet{torrence1998practical}, smoothing is necessary since otherwise the wavelet coherence is identically $1$ for all times and scales.

\textbf{3.~Cone of influence.}\\
Because the wavelet transform \cref{app:eqn:wavelet} is computed on a finite signal, the integral is implicitly zero-padded outside $[0, T]$.
This padding contaminates $W_x(s, t)$ near the boundaries $t = 0$ and $t = T$, with the contamination extending farther into the interior at larger scales $s$.
The \emph{cone of influence} (COI) is the region of the time-scale plane where this edge contamination is important; following \citet{torrence1998practical}, points $(s, t)$ inside the COI are discarded.
For the regular time grid $\cT$, let $\cT_{\mathrm{COI}}(s)$ denote the times inside the COI at scale $s$, and write $\cT_{\mathrm{COI}}(s)^c = \cT \setminus \cT_{\mathrm{COI}}(s)$ for the retained times.

\textbf{4.~Extracted statistics.}
For each prior sample, we use the WCS, WCO, and per-signal WPS to extract three summary statistics matching those reported by \citet{blasius2020long}.

\begin{enumerate}
    \item \textbf{Peak cyclic period}: The peak period of each individual population (algae, rotifers) is obtained by averaging the square root of the smoothed WPS outside the COI and selecting the scale at which this time-averaged amplitude is maximized:
    \begin{equation}\label{app:eqn:max-period}
        s_i^\star = \arg\max_s \frac{1}{|\cT_{\mathrm{COI}}(s)^c|} \sum_{t \in \cT_{\mathrm{COI}}(s)^c} \cS(|W_{x_i}(s, t)|^2)^{1/2}, \qquad T_i^\star = T(s_i^\star),
    \end{equation}
    where the maximization is over scales with at least one retained time, and the maximizing scale is mapped to a period via \cref{app:eqn:fourier-period}.
    The peak period for a given sample is the average of the rotifer and algae peak periods.
    We report the median peak period across all samples, $T^\star_{\mathrm{ensemble}}$.

    \item \textbf{Predator-prey phase lag}:
    Let $s_{\mathrm{joint}}$ be the scale of the dominant joint period, the scale at which the amplitude of the smoothed WCS, $|\cS(W_{x_1 x_2}(s, t))|$, averaged over the full series is maximal.
    Let $\cB$ be the band of scales within $0.6$ octaves of $s_{\mathrm{joint}}$.
    At each time $t$, we locate the scale of maximal coherence within this band,
    \begin{equation*}
        s^\star(t) = \arg\max_{s \in \cB} R(s, t),
    \end{equation*}
    and retain the time $t$ if $(s^\star(t), t)$ lies outside the COI, the coherence $R(s^\star(t), t)$ is at least the significance threshold $0.83$ used by \citet{blasius2020long}, and the maximum is not attained at an edge of the band (an edge maximum suggests that the coherence peak may lie outside the band).
    We denote this set of retained times by $\cT^{(n)}_{\mathrm{coh}}$ for prior sample $n$.
    At each retained time we record the phase $\phi(s^\star(t), t)$ and summarize these phases by their circular mean
    \begin{equation*}
        \bar{\phi}^{(n)} = \arg\left( \sum_{t \in \cT^{(n)}_{\mathrm{coh}}} e^{i \phi(s^\star(t), t)} \right).
    \end{equation*}
    The ensemble phase lag is the arithmetic mean of $\bar{\phi}^{(n)}$ across prior samples with at least one retained time.
    We convert this angle to a lag in time units by multiplying by $T^\star_{\mathrm{ensemble}} / (2\pi)$, where $T^\star_{\mathrm{ensemble}}$ is the median peak period from the previous statistic.

    \item \textbf{Rotational coherence}:
    The rotational coherence of prior sample $n$ is the fraction of time points that are retained in the phase analysis above,
    \begin{equation*}
        \rho^{(n)}_{\mathrm{coh}} = \frac{|\cT^{(n)}_{\mathrm{coh}}|}{|\cT|},
    \end{equation*}
    where $|\cT|$ is the total number of time points on the regular grid, including those excluded from the phase analysis by the cone of influence.
    This statistic measures the proportion of the trajectory over which algae and rotifers oscillate with a stable phase relationship near the dominant period, and is a sample-level analogue of the ``rotational fraction'' reported by \citet{blasius2020long}.
    We report the median across prior samples.
\end{enumerate}

In addition to the wavelet analysis, we compute the forecast MSE on held-out observations.
We do so by simulating $128$ trajectories from the learned prior SDE, initialized from the posterior at the last training time, using Euler--Maruyama with stepsize $10^{-3}$; squared errors are averaged over held-out observations, subtrials, and trajectories.

\subsection{Reduced order models for fluid dynamics}\label{app:subsec:rom}

In this subsection, we discuss our experiments modeling a two-dimensional incompressible flow past a circular cylinder.
We obtain the public dataset from \citet{course2023state} at \url{https://github.com/coursekevin/svise}.

The data is simulated on the domain $[-0.98, 28.98] \times [-1.98, 1.98]$ with a circular cylinder of radius $0.5$ placed at $(0, 0)$ (\Cref{fig:fluid-dynamics}). The velocity field is obtained by numerically solving the incompressible Navier--Stokes equation
\begin{equation}
\begin{aligned}\label{app:eqn:navier-stokes}
   &\frac{\partial \mbu_i(\mbz, t)}{\partial t} + \sum_j \mbu_j(\mbz, t) \frac{\partial \mbu_i(\mbz, t)}{\partial \mbz_j} = -\frac{\partial \bar{p}(\mbz, t)}{\partial \mbz_i} + \frac{1}{\mathrm{Re}} \sum_j \frac{\partial^2 \mbu_i(\mbz, t)}{\partial \mbz_j^2}\\
   &\sum_i \frac{\partial \mbu_i(\mbz, t)}{\partial \mbz_i} = 0, \quad i \in \{1, 2\}, \quad (\mbz, t) \in \bbR^2 \times \bbR
\end{aligned}
\end{equation}
where $\mbu: \bbR^2 \times \bbR \to \bbR^2$ and $\bar{p}: \bbR^2 \times \bbR \to \bbR$ are the non-dimensionalized velocity and pressure, respectively, and $\mathrm{Re}$ is the Reynolds number. 
$\mathrm{Re}$ represents the ratio of inertial to viscous forces. 
For our experiments we use $\mathrm{Re} = 2 \cdot 10^3$, which yields strong vortices in the wake of the cylinder.
\Cref{app:eqn:navier-stokes} is solved by discretizing the spatial domain on a grid of size $596602$, and solving the Navier--Stokes equation on this grid (in time), subject to the boundary conditions defined by the outer domain and 2D cylinder (\Cref{fig:fluid-dynamics}A). 
The dataset simulates the dynamics for $305.1$ time units with snapshots every $0.1$ time units.

\paragraph{Preprocessing.}
We follow the same preprocessing steps as in \citet{course2023state}. 
Namely, we drop the first $20\%$ of the snapshots representing transients in the system.
We then use the interval $[61, 256.2]$ for training and the subsequent $10$ time units $[256.2, 266.2]$ for forecasting.
From the training data, we perform the proper orthogonal decomposition (POD) \cite{berkooz1993proper} and retain the leading spatial modes that together capture $90\%$ of the time-averaged fluctuation kinetic energy; this amounts to $38$ modes. 
We then project the snapshots onto these $38$ POD modes to obtain a low-dimensional representation of the system.
Lastly, we normalize the POD coefficients by the maximum standard deviation, computed along each dimension (i.e., normalization is by a scalar $a$).

Modeling the dynamics in the space of the top POD modes will yield a reduced order model of the fluid dynamics. 
Given an initial condition, one can project the velocity field on the grid into the space of POD modes, roll forward the (stochastic) dynamics, and reconstruct the trajectory on the original grid. 
Mathematically, this is equivalent to taking the dynamics $\mbf$ defined in the space of POD modes and defining the dynamics $\bar{\mbf} = a\mbPhi \mbf(a^{-1}\mbPhi^\top (\mbz_{\mathrm{grid}} - \bar{\mbz}_{\mathrm{grid}}))$ in the original data space, where $\bar{\mbz}_{\mathrm{grid}}$ is the mean of the velocity measurements across time and $\mbPhi$ is the matrix whose columns contain the top POD modes.
Note that the dynamics $\bar{\mbf}$ live on a low-dimensional affine subspace.

To generate observations, we subsample the POD coefficients to a grid of size $\Delta t = 0.3$.
We observe that many of the higher-order modes oscillate at a smaller scale than the lower-order modes. 
In order to add Gaussian noise in such a way that the higher order modes do not have low signal-to-noise ratio, we add noise proportional to the standard deviation along each dimension. 
And to demonstrate a setting with high posterior uncertainty, we choose noise standard deviation equal to $0.3$ times the standard deviation along each dimension (signal-to-noise ratio $3.33$).
We observe that with more frequent observations (e.g., $\Delta t = 0.1$) and smaller noise scale (e.g., $0.1$ times the standard deviation along each dimension), the discrepancy between the learned dynamics is less apparent. 

\paragraph{Generative model.}
We learn a $K = 38$ dimensional ROM of fluid flow past the cylinder.
The prior initial state is modeled as Gaussian $\cN(\mbmu_0, \mbV_0)$, $\mbmu_0 \in \bbR^{38}$, $\mbV_0 \in \bbR^{38 \times 38}$; the prior drift function \cref{eqn:prior_sde} is a neural network; and the observation model is Gaussian $\cN(\mby(t) | \mbx(t), \mbR)$.
$\mbR$ is the diagonal observation-noise covariance, which we fix to its true value.

\paragraph{Baselines.}
To match the original experimental setup in \citet{course2023state}, we use the kernel-GLM posterior parameterization proposed by the authors. 
See \Cref{app:subsec:ou-spiral-supplement} for a discussion.

For SDE Matching \cite{bartosh2025sde}, we saw in our preliminary experiments that learning a state-dependent diffusion coefficient degraded learning. 
In order to separate the prior parameterization from posterior expressivity, we drop the state dependence from the diffusion parameterization. 
\begin{itemize}
    \item \textbf{Helmholtz-SDE}:
    \begin{itemize}
        \item \texttt{latent\_dim = 38}, \texttt{obs\_dim = 38},  \texttt{time\_dim = 1}, \texttt{hidden\_dim = 100}
        \item 
        Prior drift:\\
        \texttt{(latent\_dim) $\to$ Linear(latent\_dim, hidden\_dim) $\to$ Tanh $\to$ Linear(hidden\_dim, hidden\_dim) $\to$ Tanh $\to$ Linear(hidden\_dim, latent\_dim) $\to$ (latent\_dim)}
        \item Kernel-GLM, with $500$ evenly-spaced temporal inducing points; defined for both the posterior mean and covariance
        \item Diffusion network: learned diagonal $\mbSigma$, state and time independent 
        \item Helmholtz correction: linear approximation $\ell = 1$, computed using the Stein estimand with $N_{\mathrm{mc}} = 1$
        \item Observation model: fixed
        \item LR: $10^{-3}$ initial, $10^{-4}$ final (exponential LR decay)
        \item Gradient clipping: None
        \item Iterations: $3 \cdot 10^4$
        \item Number of MC samples (per iteration): $N_{\mathrm{time}} = 512$ samples over $[0, T]$ to estimate each of the KL and reconstruction loss, per trial
    \end{itemize}
    \item \textbf{SDE Matching}:
    \begin{itemize}
        \item \texttt{latent\_dim = 38}, \texttt{obs\_dim = 38},  \texttt{time\_dim = 1}, \texttt{hidden\_dim = 100}
        \item 
        Prior drift:\\
        \texttt{(latent\_dim) $\to$ Linear(latent\_dim, hidden\_dim) $\to$ Tanh $\to$ Linear(hidden\_dim, hidden\_dim) $\to$ Tanh $\to$ Linear(hidden\_dim, latent\_dim) $\to$ (latent\_dim)}
        \item Kernel-GLM, with $500$ evenly-spaced temporal inducing points; defined for both the posterior mean and covariance
        \item Diffusion network: learned diagonal $\mbSigma$, state and time independent 
        \item Observation model: fixed
        \item LR: $10^{-3}$ initial, $10^{-4}$ final (exponential LR decay)
        \item Gradient clipping: None
        \item Iterations: $3 \cdot 10^4$
        \item Number of MC samples (per iteration): $N_{\mathrm{time}} = 512$ samples over $[0, T]$ to estimate each of the KL and reconstruction loss, per trial
    \end{itemize}
    \item \textbf{SVISE}:
    \begin{itemize}
        \item \texttt{latent\_dim = 38}, \texttt{obs\_dim = 38},  \texttt{time\_dim = 1}, \texttt{hidden\_dim = 100}
        \item 
        Prior drift:\\
        \texttt{(latent\_dim) $\to$ Linear(latent\_dim, hidden\_dim) $\to$ Tanh $\to$ Linear(hidden\_dim, hidden\_dim) $\to$ Tanh $\to$ Linear(hidden\_dim, latent\_dim) $\to$ (latent\_dim)}
        \item Kernel-GLM, with $500$ evenly-spaced temporal inducing points; defined for both the posterior mean and covariance
        \item Diffusion network: learned diagonal $\mbSigma$, state and time independent 
        \item Observation model: fixed
        \item LR: $10^{-3}$ initial, $10^{-4}$ final (exponential LR decay)
        \item Gradient clipping: None
        \item Iterations: $3 \cdot 10^4$
        \item Number of MC samples (per iteration): $N_{\mathrm{time}} = 512$ samples over $[0, T]$ to estimate each of the KL and reconstruction loss, per trial
    \end{itemize}
\end{itemize}

\paragraph{Metrics.}

In \Cref{fig:fluid-dynamics}B, we compute the space--time autocorrelation of streamwise velocity fluctuations along the wake centerline. 
We draw samples from the prior starting at the true initial condition $\mbx(0)$ at the beginning of the training window, and simulating forward for $10$ time units using Euler--Maruyama with stepsize $10^{-4}$. 
We observe that some trajectories begin to numerically diverge beginning at approximately $15-20$ time units, so we only analyze the first $10$ units.

At each sampled trajectory, we reconstruct the velocity field on the original grid by projecting back through the POD modes, and we extract the streamwise velocity component $\mbu_1(\mbz, t)$ along the strip $[1, 10] \times [-0.1, 0.1]$ in the near wake, averaging across the transverse direction.
For each trajectory, we compute the joint space--time autocorrelation\begin{equation*}
    C(\Delta \mbz_1, \Delta t) = \frac{\bbE\left[ \tilde{\mbu}_1(\mbz_1, t) \tilde{\mbu}_1(\mbz_1 + \Delta \mbz_1, t + \Delta t) \right]}{\bbE\left[ \tilde{\mbu}_1(\mbz_1, t)^2 \right]},
\end{equation*}
where $\tilde{\mbu}_1$ denotes the fluctuation about the temporal mean at each streamwise location within the trajectory.
The numerator is averaged over valid reference pairs at each lag, and the denominator over all reference points in the trajectory.
We average the resulting correlations across prior samples.
The reference correlation is computed by concatenating the training and test trajectories.

There are two apparent patterns in the space--time autocorrelation: (1) oscillation in streamwise lag (from the periodic vortex pattern in the wake) and (2) diagonal banding in time (vortices being carried downstream).
The failure mode of simulation-free VI is specifically a loss of temporal coherence. 
This is because temporal correlations are those that depend on the joint posterior path law which, in turn, degrades the prior path law. The spatial correlations, on the other hand, exist within a single snapshot of time.

We also report the cumulative MSE, summed over the $10$-unit forecasting window at intervals of $0.1$ time units.
We compute this quantity by simulating $128$ trajectories for $10$ time units using Euler--Maruyama with stepsize $10^{-4}$, starting from the ground truth POD coefficients at the end of the training window. 
We then transform the samples into the physical space by multiplying by the normalization constant and computing the linear combination with the top-$38$ POD vectors. 
Finally, we compute the MSE between the reconstructed velocity and the projection of the (true) velocity onto the top-$38$ POD modes, averaged over the $128$ simulated trajectories. 
We normalize by the total number of velocity components on the grid.
We report results for $10$ training replicates with $\pm 2$ SE.

\begin{table}[!h]
\centering
\setlength{\tabcolsep}{3pt}
\begin{tabular}{lccc}
\toprule
& Helmholtz-SDE & SDE Matching & SVISE \\
\midrule
forecast MSE $\downarrow$ & $\mathbf{3.46 \pm 0.75}$ & $\mathbf{3.88 \pm 0.53}$ & $\mathbf{3.86 \pm 0.43}$\\
nELBO $\downarrow$ & $\mathbf{-2.57 \pm 0.00}$ & $-2.11 \pm 0.01$ & $-2.13 \pm 0.01$\\
\bottomrule
\end{tabular}
\vspace{0.5em}
\caption{Cumulative MSE computed over the $10$-time-unit forecasting window. nELBO is reported in units of $10^4$, e.g., Helmholtz-SDE has nELBO $-2.57 \cdot 10^4 \pm 0.00 \cdot 10^4$.}
\label{tab:rom-mse}
\vspace{-1em}
\end{table}
\section{Runtime comparisons}\label{app:sec:timing}

We compare the runtime of Helmholtz-SDE to all baseline VI algorithms considered. 

We time all algorithms on a single NVIDIA H100 GPU, and we exclude JIT compilation time and ELBO evaluation time from the runtime calculation. 
For Helmholtz-SDE, SING, SDE Matching, and SVISE, we assess convergence using the negative ELBO \cref{eqn:elbo} calculated on a dense time grid. 
For Latent SDE, we approximate the negative ELBO using 10 samples per trial from the variational SDE. 

\begin{table*}[h!]
\centering
\setlength{\tabcolsep}{3.5pt}

% --- (a) Lorenz ---
\subcaptionbox{Lorenz system\label{tab:lorenz-convergence}}{%
\begin{tabular}{lccccc}
\toprule
& & \multicolumn{2}{c}{nELBO $< 50$} & \multicolumn{2}{c}{nELBO $< 70$} \\
\cmidrule(lr){3-4}\cmidrule(lr){5-6}
Method & ms/iter & time (s) & iterations & time (s) & iterations \\
\midrule
Helmholtz-SDE & $8.56 \pm 0.05$ & $123.7 \pm 3.4$ & $14400 \pm 427$ & $27.7 \pm 1.2$ & $3200 \pm 158$ \\
SDE Matching & $8.08 \pm 0.06$ & \textemdash & \textemdash & $185.9 \pm 18.7$ & $22950 \pm 2294$ \\
ARCTA/SVISE & $39.58 \pm 0.04$ & \textemdash & \textemdash & $336.2 \pm 32.6$ & $8490 \pm 823$ \\
SING & $1157.8 \pm 0.8$ & $3695.6 \pm 503.8$ & $3192 \pm 435$ & $807.5 \pm 117.8$ & $697 \pm 102$ \\
Latent SDE & $1068.2 \pm 42.7$ & \textemdash & \textemdash & $9464.9 \pm 930.4$ & $8870 \pm 893$ \\
\bottomrule
\end{tabular}}

\vspace{0.75em}

% --- (b) predator-prey ---
\subcaptionbox{Predator--prey system\label{tab:predator-prey-convergence}}{%
\begin{tabular}{lccccc}
\toprule
& & \multicolumn{2}{c}{nELBO $< 70$} & \multicolumn{2}{c}{nELBO $< 90$} \\
\cmidrule(lr){3-4}\cmidrule(lr){5-6}
Method & ms/iter & time (s) & iterations & time (s) & iterations \\
\midrule
Helmholtz-SDE & $21.55 \pm 0.01$ & $96.0 \pm 7.6$ & $4450 \pm 352$ & $16.5 \pm 1.6$ & $760 \pm 74$ \\
SDE Matching & $15.28 \pm 0.01$ & \textemdash & \textemdash & $40.2 \pm 8.1$ & $2620 \pm 528$ \\
SVISE & $12.80 \pm 0.01$ & \textemdash & \textemdash & $41.3 \pm 7.7$ & $3220 \pm 604$ \\
\bottomrule
\end{tabular}}

\vspace{0.75em}

% --- (c) cylinder ---
\subcaptionbox{Two-dimensional cylinder\label{tab:cylinder-convergence}}{%
\begin{tabular}{lccccc}
\toprule
& & \multicolumn{2}{c}{nELBO $< -2.5 \cdot 10^4$} & \multicolumn{2}{c}{nELBO $< -2.0 \cdot 10^4$} \\
\cmidrule(lr){3-4}\cmidrule(lr){5-6}
Method & ms/iter & time (s) & iterations & time (s) & iterations \\
\midrule
Helmholtz-SDE & $859.5 \pm 9.6$ & $17094.7 \pm 314.0$ & $19870 \pm 313$ & $5197.9 \pm 72.8$ & $6000 \pm 42$ \\
SDE Matching & $577.5 \pm 7.0$ & \textemdash & \textemdash & $9017.5 \pm 252.2$ & $15580 \pm 421$ \\
SVISE & $298.7 \pm 3.2$ & \textemdash & \textemdash & $4178.6 \pm 106.0$ & $13900 \pm 336$ \\
\bottomrule
\end{tabular}}

\caption{
Per-iteration runtime and the time and number of iterations required to reach the indicated nELBO thresholds. 
Errors are $\pm$ 2 standard errors over 10 replicates. 
A dash indicates that the threshold was not reached by every replicate within the training budget.
}
\label{tab:convergence}
\end{table*}

\newpage

\end{document}